\documentclass{article}

\usepackage{microtype}
\usepackage{graphicx}
\usepackage{subfigure}
\usepackage{booktabs} 
\usepackage{caption}  
\usepackage{amsmath}
\usepackage{amssymb}
\usepackage{amsthm}
\usepackage{siunitx}
\usepackage{multirow}
\usepackage{listings}
\usepackage{placeins}
\usepackage{enumitem}
\usepackage[most]{tcolorbox}
\usepackage[normalem]{ulem}
\usepackage{comment}
\usepackage{sidecap}
\sidecaptionvpos{table}{c}

\usepackage{tikz}
\usepackage{bm}
\usepackage{pgfplots}
\pgfplotsset{compat=1.18}
\usepgfplotslibrary{groupplots}
\usepackage{xcolor}
\definecolor{hyenared}{HTML}{C03E35}
\definecolor{attnblue}{HTML}{3171B2}
\definecolor{cnextgreen}{HTML}{768D21}
\usetikzlibrary{arrows.meta, decorations.pathreplacing, calc, positioning, fit, backgrounds, shapes.geometric}

\newcommand{\thewell}{The Well}
\newcommand{\acoustic}{\texttt{acoustic\_scattering}}
\newcommand{\activematter}{\texttt{active\_matter}}

\newcommand{\euler}{\texttt{euler\_multi\_quadrants}}
\newcommand{\staircase}{\texttt{helmholtz\_staircase}}
\newcommand{\MHD}{\texttt{MHD}}

\newcommand{\pattern}{\texttt{gray\_scott\_reaction\_diffusion}}

\newcommand{\shearflow}{\texttt{shear\_flow}}
\newcommand{\supernova}{\texttt{supernova\_explosion}}

\usepackage{xcolor}

\usepackage{ulem}

\usepackage{hyperref}

\theoremstyle{plain}
\newtheorem{theorem}{Theorem}[section]
\newtheorem{lemma}[theorem]{Lemma}

\theoremstyle{definition}
\newtheorem{definition}[theorem]{Definition}
\newtheorem{remark}[theorem]{Remark}

\newcommand{\norm}[1]{\left\|#1\right\|}
\newcommand{\abs}[1]{\left|#1\right|}
\newcommand{\LK}{L_{\!K}}

\PassOptionsToPackage{numbers,compress}{natbib}
\usepackage[preprint]{neurips_2026} 

\title{Native Multi-Dimensional Subquadratic Operators via Input Dependent Long Convolutions}
\author{%
\parbox{\textwidth}{\centering
David~R.~Wessels\,$^{*,\ddagger,1}$ \quad
Farhad~Ramezanghorbani\,$^{*,\ddagger,2}$ \quad
Alireza~Moradzadeh\,$^{\ddagger,2}$ \quad
David~W.~Romero\,$^{\dagger,\ddagger,3}$ \quad
Olivia~Viessmann\,$^{2}$ \quad
Maksim~Zhdanov\,$^{1}$ \quad
John~St.~John\,$^{2}$ \quad
Ken~Janik\,$^{2}$ \quad
David~M.~Knigge\,$^{4}$ \quad
Yucheng~Tang\,$^{2}$ \quad
Erik~J.~Bekkers\,$^{1}$ \quad
Saee~Gopal~Paliwal\,$^{2}$
\\[0.6em]
{\normalfont\small
$^{1}$AMLab, University of Amsterdam \quad
$^{2}$NVIDIA \quad
$^{3}$Cartesia AI \quad
$^{4}$New Theory AI}
}}

\begin{document}

\maketitle

\begingroup
\renewcommand{\thefootnote}{}%
\footnotetext{$^{*}$Equal contribution. \quad
$^{\dagger}$Technical lead. \quad
$^{\ddagger}$Core contributor.\\
Correspondence: \texttt{davidwessels15@gmail.com}, \texttt{farhadr@nvidia.com}}%
\endgroup

\begin{abstract}
Subquadratic alternatives to attention require compromises when applied to multi-dimensional data: standard convolutions lack global receptive fields and input dependency, while recurrent models require rasterizing data such as images, volumes, and partial differential equation (PDE) into an ad-hoc $1\rm D$ scan order that violates their spatial structure. We introduce \textit{HyenaND}, a subquadratic, global, input-dependent operator that acts directly on the native geometry of multidimensional data through convolutions with implicitly parametrized global, input-dependent multi-dimensional convolutional kernels. Our CUDA implementation, \texttt{nSubQ}, fuses the FFT-convolution path to turn HyenaND's $\mathcal{O}(L \log L)$ scaling into wall-clock speedups. Across long-context genomics, computer vision, medical imaging, and PDE modeling, pure HyenaND stacks match the accuracy of strong attention baselines, while hybrid configurations that interleave HyenaND and attention layers outperform both pure attention and strong recurrence-based hybrids.
\vspace{-2mm}
\end{abstract}

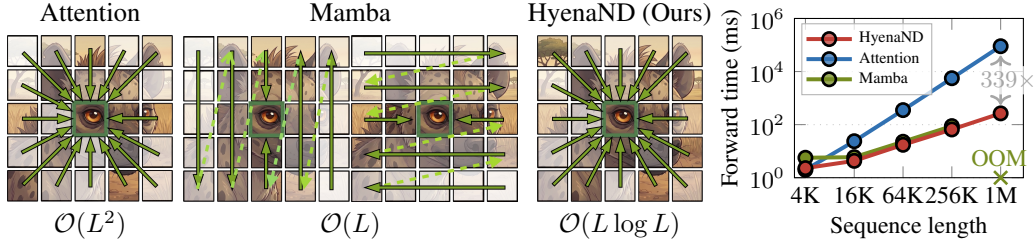
\begin{figure}[h]
    \centering
    \begin{minipage}[t]{0.64\linewidth}
        \centering
        \setlength{\tabcolsep}{2pt}
        \begin{tabular}{ccc}
            {Attention} & {Mamba} & {HyenaND (Ours)} \\[2pt]
            \includegraphics[width=0.245\linewidth, clip]{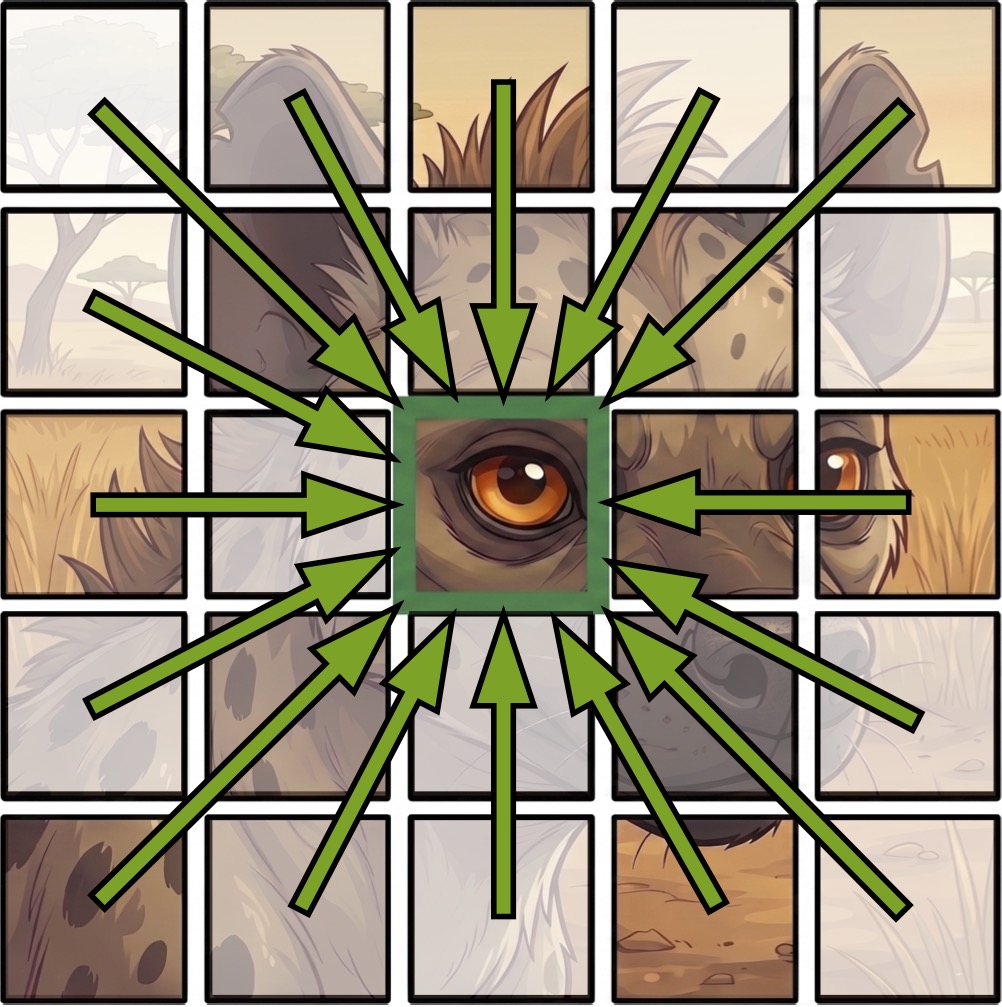} &
            \includegraphics[width=0.245\linewidth, clip]{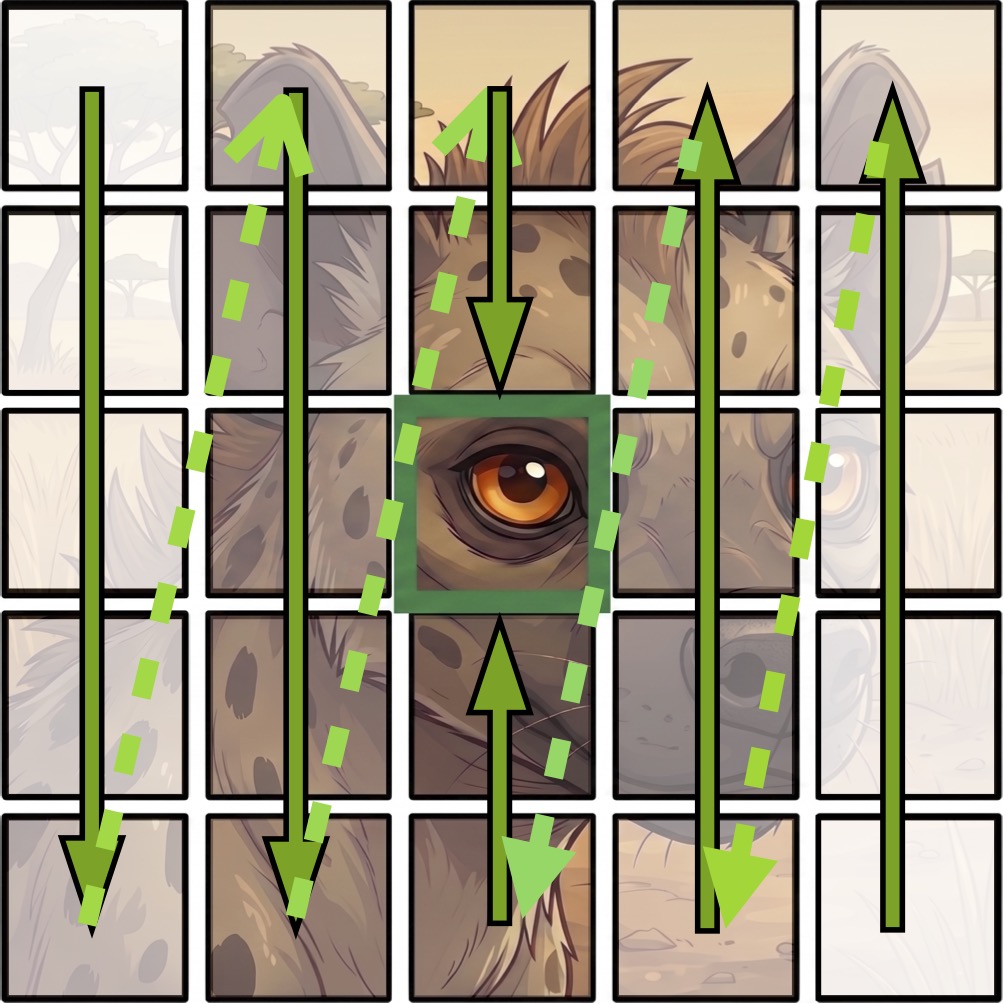}\hspace{1pt}\includegraphics[width=0.245\linewidth, clip]{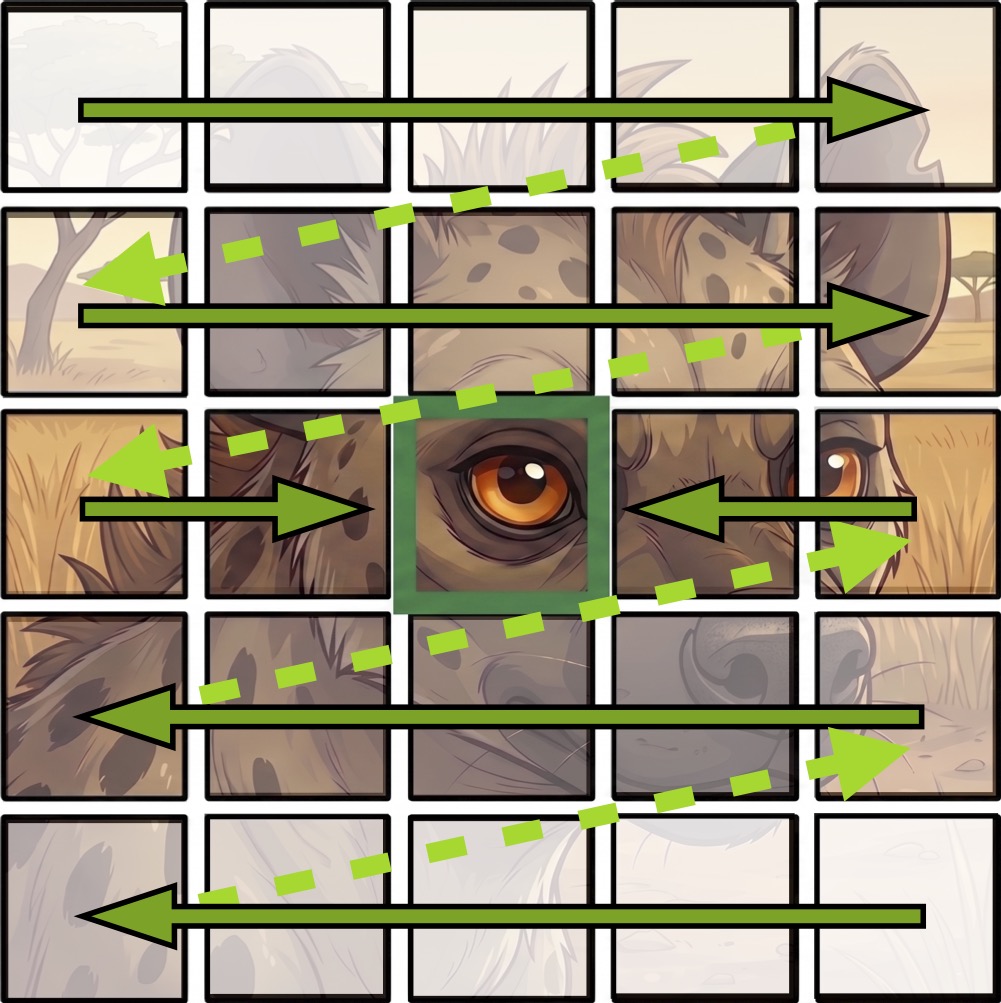} &
            \includegraphics[width=0.245\linewidth, clip]{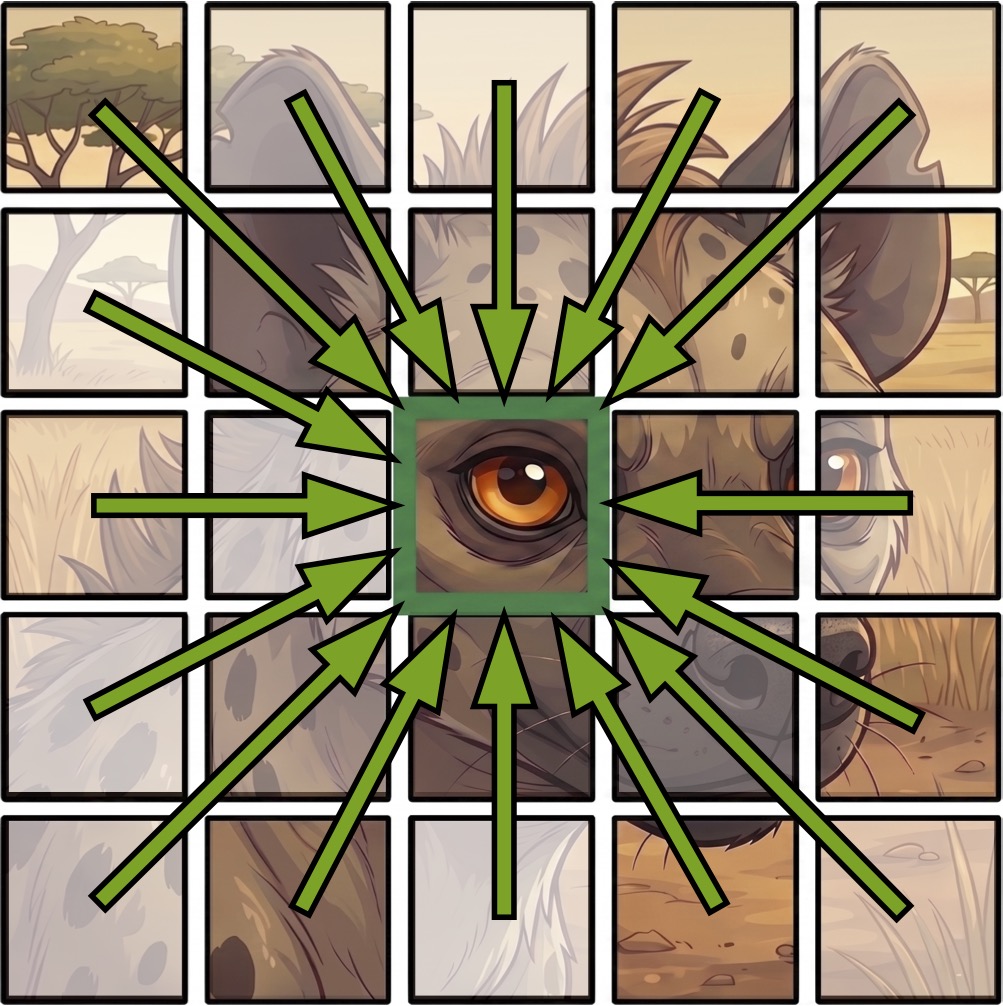} \\
            $\mathcal{O}(L^2)$ & $\mathcal{O}(L)$ & $\mathcal{O}(L \log L)$
        \end{tabular}
    \end{minipage}%
    \hfill%
    \begin{minipage}[t]{0.33\linewidth}
        \centering
        \setlength{\tabcolsep}{3pt}
        \begin{tabular}{c}
        \strut \\[-6pt]
        \begin{tikzpicture}
        \begin{loglogaxis}[
            width=\linewidth,
            height=0.8\linewidth,
            xlabel={Sequence length},
            ylabel={Forward time (ms)},
            xlabel style={font=\small, yshift=2pt},
            ylabel style={font=\small, yshift=-4pt},
            xmin=3000, xmax=2000000,
            ymin=1, ymax=1000000,
            xtick={4096,16384,65536,262144,1048576},
            xticklabels={4K,16K,64K,256K,1M},
            tick label style={font=\small},
            legend style={font=\tiny, at={(0.03,0.97)}, anchor=north west, draw=gray!60,
                fill=white, fill opacity=0.75, text opacity=1,
                inner sep=1pt, row sep=-2pt, legend cell align=left},
            ymajorgrids=true,
            grid style={dotted, gray!40},
            every axis plot/.append style={line width=1.5pt},
            clip=false,
        ]
        \addlegendimage{color=hyenared, mark=*, mark size=2.5pt, line width=1.5pt, mark options={draw=black, fill=hyenared, line width=1.0pt}} 
        \addlegendentry{HyenaND}
        \addplot[color=attnblue, mark=*, mark size=2.5pt, mark options={draw=black, fill=attnblue, line width=1.0pt}] coordinates {
            (4096,   1.977)
            (16384,  23.746)
            (65536,  352.368)
            (262144, 5605.975)
            (1048576,90006.188)
        };
        \addlegendentry{Attention}
        \addplot[color=cnextgreen, mark=*, mark size=2.5pt, mark options={draw=black, fill=cnextgreen, line width=1.0pt}] coordinates {
            (4096,   5.650)
            (16384,  5.939)
            (65536,  22.641)
            (262144, 87.982)
        };
        \addlegendentry{Mamba}
        \addplot[only marks, color=cnextgreen, mark=x, mark size=3.5pt,
            mark options={draw=cnextgreen, line width=1.0pt}]
            coordinates {(1048576, 1)};
        \node[above, font=\small, color=cnextgreen] at (axis cs:1008576,1.2) {OOM};
        \addplot[color=hyenared, mark=*, mark size=2.5pt, mark options={draw=black, fill=hyenared, line width=1.0pt}] coordinates {
            (4096,   2.296)
            (16384,  4.315)
            (65536,  17.369)
            (262144, 65.702)
            (1048576,265.335)
        };
        \draw[<->, gray!70, line width=1pt]
            (axis cs:1048576, 40000) -- (axis cs:1048576, 550)
            node[midway, anchor=east, font=\small, xshift=16pt,
                fill=white, fill opacity=0.75, text opacity=1,
                inner sep=2pt, rounded corners=2pt] {$339{\times}$};
        \end{loglogaxis}
        \end{tikzpicture}
        \end{tabular}
    \end{minipage}
    \vspace{-3mm}
    \caption{\textit{(Left):} Complexity of global multi-dimensional operators by token count $L$. Attention is natively multi-dimensional but scales quadratically. Mamba is subquadratic but inherently $1\rm D$, requiring ad-hoc $1\rm D$ scan orders to process multi-dimensional data. HyenaND (ours) is global, natively multi-dimensional and subquadratic. \textit{(Right):} Forward-pass time vs.\ sequence length using \texttt{flash-attention}, the official \texttt{mamba\_chunk\_scan\_combined} Mamba2 kernel and \texttt{nSubQ}; HyenaND scales gracefully to million-token sequences while attention collapses at long contexts.
    \vspace{-3mm}}
    \label{fig:architectures}
\end{figure}

\vspace{-1mm}
\section{Introduction}
\label{sec:intro}
\vspace{-1mm}

Deep learning models increasingly operate on large, multi-dimensional data: high-resolution images, $3\rm D$ medical scans, high-resolution spatio-temporal fields, long genomic sequences.
On such inputs, the quadratic cost of Transformers~\citep{vaswani2017attention} hinders its applicability to even moderate sizes.
The standard workaround is patchification~\citep{dosovitskiy2020image}, but it trades performance for tractability: patchification scaling laws~\citep{wang2025patchification} show that accuracy improves monotonically as patch size decreases across architectures, tasks and input scales.
However, shrinking the patch size to harvest these gains drives the token count back up, and attention's quadratic cost again becomes binding.
We need a subquadratic operator that can follow the patchification scaling curve past attention's quadratic wall.

State-space models (SSMs) such as S4~\citep{gu2021efficiently} and their input-dependent descendants Mamba~\citep{gu2023mamba, dao2024mamba2} have become the prevailing subquadratic approach, with $\mathcal{O}(L)$ training and $\mathcal{O}(1)$ per-step autoregressive inference.
As recurrences, they are fundamentally causal and $1\rm D$: bidirectional contexts must run the recurrence twice --on the input and its reversal-- doubling compute and memory. 
Worse, recurrences admit no natural extension to multi-dimensional data. As such, existing adaptations must rasterize $N \rm D$ inputs into ad-hoc $1\rm D$ scan orders, destroying their geometric structure (\S\ref{sec:landscape-recurrent}; Fig.~\ref{fig:architectures}).

Mamba's selectivity is one specific way to induce input dependency: per-token dependency at the cost of linear time invariance (LTI). This is a well-motivated trade for 1D autoregressive language modeling, where the data is non-LTI and causality precludes conditioning on the whole input. But for natural and multi-dimensional signals --audio, images, volumes, physical fields-- that's the wrong trade, as these signals are themselves approximately LTI and often available at once. In fact, the original Mamba paper reports that Mamba underperforms the LTI S4 on audio generation~\citep[Appx.~E.4]{gu2023mamba}.

A separate, non-recurrent family of subquadratic operators acts on the input in its native geometry: \textit{long convolutions} \citep{romero2021ckconv, romero2021flexconv, poli2023hyena}. This family has been overlooked for two reasons: First, the dominant SSM formulation views the convolutional form as a trick to speed up training of recurrent operators rather than a distinct architectural family, a conflation we untangle in \S\ref{sec:landscape}. Second, long convolutions lack the input dependency that underlies Mamba's selectivity, and this absence is widely regarded as a fundamental expressiveness limitation. This second objection, however, is specifically tied to autoregressive $1\rm D$ processing, where causality forces input-dependence to be expressed per token. In non-autoregressive settings (diffusion, multi-dimensional data) the input is available at once and the operator can be conditioned on the input as a whole: input dependency at the \emph{sample} level rather than at the \emph{token} level (\S\ref{sec:landscape-conv}). Combined with the LTI inductive bias argued for above, this positions long convolutions as the matched subquadratic primitive for natural and multi-dimensional data.

Generalized convolutional models also carry no recurring hidden state: each output is a direct function of the input weighted by the kernel, so propagating information across long distances is as simple as placing kernel mass at the corresponding offset. Recurrent operators, by contrast, must compress relevant history into a fixed-size state which a bottleneck that becomes especially severe in multi-dimensional settings. In these settings the spatial extent of dependencies grows multiplicatively.

Together, subquadratic scaling, native multi-dimensional structure, sample-level input-dependence, and preserved LTI define a target operator that no existing primitive simultaneously satisfies.
We realize it with \textit{HyenaND} (\S\ref{sec:HyenaND}): an $N{\rm D}$ convolution acts directly on the input geometry (no rasterization), with a global, freely learned kernel conditioned on the full input through FiLM-modulated registers 
(sample-level input-dependence, LTI) evaluated via $N{\rm D}$ FFTs in $\mathcal{O}(L\log L)$ (subquadratic). We pair it with \texttt{nSubQ} (\S\ref{sec:cuda}), an IO-aware CUDA library that turns this asymptotic edge into wall-clock speedups via fused $N{\rm D}$ FFTs, spectral modulations, and mixed precision. 
Across long-context genomics, computer vision, medical imaging, and PDE modeling, pure HyenaND stacks match or outperform strong attention and modality-specialized baselines, while hybrid HyenaND-attention stacks outperform both pure attention and recurrence-based hybrids.
Our primary contributions are:
\begin{itemize}[leftmargin=10pt, topsep=0pt, itemsep=0pt]
    \item A reframing of the subquadratic landscape (\S\ref{sec:landscape}) that deconflates the prevalent SSM-centric framing into \textit{SSMs}: S4 \citep{gu2021efficiently}, H3 \citep{fu2022hungry};  \textit{generalized recurrent}: Mamba \citep{gu2023mamba}, Linear Attention \citep{katharopoulos2020transformers}, GDA \citep{yang2025gated}; and \textit{generalized convolutional models}: CKConv \citep{romero2021ckconv}, Hyena \citep{poli2023hyena}, HyenaND (ours). 
    \item HyenaND (\S\ref{sec:HyenaND}): a strong native $N \rm D$, LTI, sample-level input-dependent subquadratic operator.
    \item \texttt{nSubQ} (\S\ref{sec:cuda}): An IO-aware CUDA implementation that translates the subquadratic theoretical gains of $N{\rm D}$ convolutional operators into wall-clock speedups.
    \item A thorough empirical validation of HyenaND across modalities in 1D, 2D and 3D, showing the strength of both pure and hybrid HyenaND backbones (\S\ref{sec:experiments},~Appendix~\ref{app:experiments}).
\end{itemize}

\vspace{-2.5mm}
\section{From SSMs to Generalized Recurrent and Convolutional Models}
\label{sec:landscape}
\vspace{-1.5mm}

The success of State-Space Models (SSMs) has reshaped how the broader landscape of subquadratic architectures is read --so much that \textbf{\textit{subquadratic operators are routinely labeled as SSMs, regardless of whether they admit a state-space formulation}}. This conflation has two consequences. First, purely recurrent models like Mamba are incorrectly labeled as SSMs despite admitting no convolutional form. Second, purely convolutional models are wrongly assumed to be inherently recurrent --akin to pure SSMs \citep{gu2021efficiently}--, and thus less expressive. We defy both readings and trace them to a single source.

\vspace{-2mm}
\subsection{State-Space Models: The Source of the Confusion}
\label{sec:landscape-ssm}
\vspace{-1mm}

An SSM maps a sequence $\mathbf{x}{=}\{x_t\}_{t=1}^L$ to an output $\mathbf{y}{=}\{y_t\}_{t=1}^L$ via a hidden state $\mathbf{h}_t \in \mathbb{R}^d$ as:
\begin{equation}
\label{eq:rssm-recurrence}
    y_t = \mathbf{c} \ \mathbf{h}_t;\quad \mathbf{h}_t = \mathbf{A}\,\mathbf{h}_{t-1} + \mathbf{b}\ x_t \quad \text{with parameters} \ \ \mathbf{A} \in \mathbb{R}^{d \times d}, \mathbf{b} \in \mathbb{R}^{d}, \mathbf{c}\in \mathbb{R}^{1 \times d}.
\end{equation}
For multi-channel inputs $x_t \in \mathbb{R}^H$, an SSM treats each channel independently, corresponding to a depthwise operation. Unrolling the recurrence in \eqref{eq:rssm-recurrence} yields an \textit{equivalent} convolution $\mathbf{y} {=} K \ast \mathbf{x}$, with a long kernel $K{=}(\mathbf{cb}, \mathbf{cAb}, ..., \mathbf{cA}^{L-1}\mathbf{b})$, evaluable in $\mathcal{O}(L \log L)$ time via FFT. 
\textit{This equivalence is specific to SSMs, and relaxing it in either direction --recurrent or convolutional-- yields two distinct architectural families that share neither computational form nor the same trade-offs.}


\vspace{-2mm}
\subsection{Generalized Recurrent Models}
\label{sec:landscape-recurrent}
\vspace{-1mm}

Generalized recurrent models --with Mamba~\citep{gu2023mamba} as its archetype-- sacrifice LTI and the convolutional form of SSMs in exchange for input dependency. A linear projection $f_\theta$ makes its parameters depend on the input: $(\mathbf{A}_t, \mathbf{b}_t, \mathbf{c }_t) {=} f_\theta(x_t)$. The upside is content-based selectivity, recovering much of attention's behavior on associative-recall and induction-head tasks at $\mathcal{O}(1)$-per-step inference.

The cost is structural and shows up in three places. \textit{(i) No fast convolutional path.} Without a single global kernel, training instead requires \emph{parallel scans} \citep{smith2023s5, blelloch1990prefix} to recover parallelism during training, which require hardware-aware kernels to stay competitive with the FFT path. \textit{(ii) Causal in $1\rm D$.} Any non-autoregressive use ($1\rm D$ vision, audio, BERT-style encoders, diffusion) must run the recurrence twice --on the input and its reversal-- doubling compute and memory. \textit{(iii) No natural genereralization to $N{\rm D}$.} Multi-dimensional adaptations such as Vision Mamba~\citep{zhu2024vim}, VMamba~\citep{liu2024vmamba} and Mamba-ND~\citep{li2024mamband} must rasterize the input into one or several ad-hoc $1\rm D$ scan orders, which destroy the spatial structure of the input and incur costs proportional to the number of scan orders used.

These costs make the recurrent branch a specialist for $1\rm D$ non-LTI data (e.g., text) rather than a universal subquadratic primitive. Other recurrent models such as Mamba-2~\citep{dao2024mamba2}, RetNet~\citep{sun2023retentive}, GDA \citep{yang2025gated}, KDA \citep{team2025kimi}, and linear attention \citep{katharopoulos2020transformers, peng2023rwkv} fall in this family and share the trade-offs outlined here.

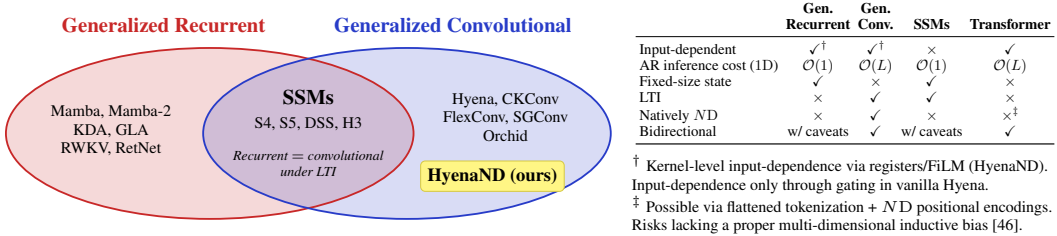
\begin{figure}
\centering
\begin{minipage}[c]{0.58\textwidth}
\centering
\resizebox{\linewidth}{!}{
\begin{tikzpicture}[font=\small]
  \definecolor{recurrentcol}{RGB}{200, 40, 40}
  \definecolor{convcol}{RGB}{40, 60, 200}

  \fill[recurrentcol, opacity=0.18] (-1.5, 0) ellipse [x radius=3.1, y radius=1.3];
  \fill[convcol,      opacity=0.18] ( 1.5, 0) ellipse [x radius=3.1, y radius=1.3];
  \draw[recurrentcol, thick]        (-1.5, 0) ellipse [x radius=3.1, y radius=1.3];
  \draw[convcol,      thick]        ( 1.5, 0) ellipse [x radius=3.1, y radius=1.3];

  \node[recurrentcol, font=\bfseries\small] at (-2.2, 1.65) {Generalized Recurrent};
  \node[convcol,      font=\bfseries\small] at ( 2.2, 1.65) {Generalized Convolutional};

  \node[align=center, font=\scriptsize] at (-3.0, 0.05)
    {Mamba, Mamba-2\\ KDA, GLA\\ RWKV, RetNet};

  \node[align=center, font=\scriptsize] at ( 3.0, 0.25)
    {Hyena, CKConv\\ FlexConv, SGConv\\ Orchid};

  \node[font=\bfseries\small]                   at (0,  0.55) {SSMs};
  \node[font=\scriptsize]                       at (0,  0.15) {S4, S5, DSS, H3};
  \node[font=\tiny\itshape, align=center]       at (0, -0.45)
    {Recurrent $=$ convolutional\\ under LTI};

  \node[fill=yellow!60, draw=yellow!80!black, rounded corners=2pt,
        inner sep=3pt, font=\bfseries\footnotesize] at ( 2.8, -0.65)
    {HyenaND (ours)};
\end{tikzpicture}}
\end{minipage}%
\hfill
\begin{minipage}[c]{0.40\textwidth}
\centering
\resizebox{\linewidth}{!}{%
\setlength{\tabcolsep}{0.25em}
\begin{tabular}{lcccc}
\toprule
                              & \textbf{Gen.}          & \textbf{Gen.}          &               &                      \\
                              & \textbf{Recurrent}     & \textbf{Conv.}         & \textbf{SSMs} & \textbf{Transformer} \\
\midrule
Input-dependent               & \checkmark$^{\dagger}$ & \checkmark$^{\dagger}$ & $\times$         & \checkmark         \\
AR inference cost ($1\rm D$)  & $\mathcal{O}(1)$       & $\mathcal{O}(L)$       & $\mathcal{O}(1)$ & $\mathcal{O}(L)$   \\
Fixed-size state              & \checkmark             & $\times$               & \checkmark    & $\times$             \\
LTI                           & $\times$               & \checkmark             & \checkmark    & $\times$             \\
Natively $N\rm D$               & $\times$               & \checkmark             & $\times$      & $\times^{\ddagger}$  \\
Bidirectional                 & w/ caveats             & \checkmark             & w/ caveats    & \checkmark           \\
\bottomrule
\end{tabular}}
\\[3pt]
{\tiny\raggedright
$^{\dagger}$ Kernel-level input-dependence via registers/FiLM (HyenaND). Input-dependence only through gating in vanilla Hyena. \quad\quad\quad\quad
$^{\ddagger}$ Possible via flattened tokenization + $N \rm D$ positional encodings. Risks lacking a proper multi-dimensional inductive bias \citep{wang2026vit}.\par}
\end{minipage}
\vspace{-1mm}
\caption{The subquadratic landscape. SSMs are simultaneously recurrent and convolutional. Relaxing the equivalence in either direction yields two distinct families, \emph{generalized recurrent} and\break \emph{generalized convolutional} \emph{(left)}. Each family has different computational forms and trade-offs  \emph{(right)}.
\vspace{-7mm}}
\label{fig:landscape}
\end{figure}

\subsection{Generalized Convolutional Models}
\label{sec:landscape-conv}
\vspace{-1mm}

Generalized convolutional models such as CKConv~\citep{romero2021ckconv} and Hyena~\citep{poli2023hyena} take the converse direction: they sacrifice the recurrence in exchange for a freely learned kernel. They retain the convolutional form $\mathbf{y} {=} K \ast \mathbf{x}$ but replace the constraining SSM kernel $(\mathbf{cb}, \mathbf{cAb}, ..., \mathbf{cA}^{L-1}\mathbf{b})$ with an arbitrary function $K_i {=} f_\theta(\mathbf{c}_i)$ learned over coordinates $\mathbf{c} \in \mathbb{R}^N$.

The benefits are structural and show up in three places. \textit{(i) Fast convolutional path.} The single global kernel keeps the fast FFT path intact. 
\textit{(ii) Natural $N{\rm D}$ generalization.} Convolutions extend naturally to multi-dimensional signals: an $N\rm D$ input is convolved with a freely learnable kernel parameterized in $N\rm D$ --no ad-hoc $1\rm D$ rasterization. \textit{(iii) LTI inductive bias.} The operator is shift-equivariant by construction, a useful prior for the approximately-LTI signals that dominate natural and multi-dimensional data such as audio, vision, PDEs, medical volumes, etc.

Beyond the fast path, the absence of a recurrent state has a second structural consequence: long-range routing is direct. An output at position $j$ is a linear combination of inputs weighted by the kernel, so propagating a value across distance $\delta$ is just kernel mass at offset $\delta$ -- with no fixed-size hidden state to compress through. Recurrent operators must instead carry information through state updates, a bottleneck that becomes especially acute in multi-dimensional settings where dependencies span the full grid. The simple\_copy results in \S\ref{sec:exp-spatial-recall} make this concrete: Mamba (bidir) degrades to random-baseline territory on 3D copy, while HyenaND maintains 4-orders-of-magnitude lower error.

The standard concern with this branch is input independency: the same $K$ is used for \textit{all} inputs. This is not an inherent property of the family, but rather an artifact of current instantiations, which inherit the autoregressive $1\rm D$ assumption that conditioning must happen at the token level. In non-autoregressive settings (diffusion, multi-dimensional data) the input is available at once and the kernel can be conditioned on it: $K_i(\mathbf{x}) {=} f_\theta\bigl(\mathbf{c};\, z(\mathbf{x})\bigr)$, to obtain sample-level input dependency while preserving LTI and the FFT path.

Together, the properties outlined in this section make the convolutional branch the natural subquadratic primitive for approximately LTI and multi-dimensional data, as well as bidirectional tasks.

\vspace{-2mm}
\subsection{The Subquadratic Landscape at a Glance}
\vspace{-1mm}

In conclusion, the SSM equivalence breaks into two specialist branches: recurrent for $1\rm D$ non-LTI sequences (text); convolutional for approximately-LTI, multi-dimensional signals and bidirectional tasks. Figure \ref{fig:landscape} consolidates the design space.




\vspace{-2mm}
\section{HyenaND}
\label{sec:HyenaND}
\vspace{-1mm}

\subsection{Background: Hyena and CKConv}
\label{sec:HyenaND-bg}
\vspace{-1mm}

\textbf{Hyena.}
Hyena~\cite{poli2023hyena} replaces attention with a data-controlled linear operator built from long implicit convolutions interleaved with elementwise gating.
For a $1\rm D$ input $\mathbf{x} \in \mathbb{R}^{L \times H}$, an order-2 Hyena operator computes $\mathbf{q}, \mathbf{k}, \mathbf{v} \in \mathbb{R}^{L \times H}$ from $\mathbf{x}$ as the composition of a linear projection and a short causal depthwise convolution, then combines them through gating and a causal long FFT convolution:
\begin{equation}
\label{eq:hyena}
\mathbf{y} = \sigma(\mathbf{v}) \,\odot\, \bigl(K \ast (\mathbf{q} \odot \sigma(\mathbf{k}))\bigr).
\end{equation}
$\sigma$ is a pointwise non-linearity --set to the identity in Hyena--, $\odot$ is elementwise multiplication, and $K \in \mathbb{R}^{L \times H}$ is a global convolutional kernel evaluated via FFT convolution in $\mathcal{O}(L \log L)$ time.

\textbf{Continuous convolutional kernels.}
A key design choice in Hyena is the parameterization of $K$.
Following CKConv~\cite{romero2021ckconv, romero2021flexconv}, Hyena defines the kernel as the pointwise product $K_i {=} w(c_i) \odot f_\theta(c_i)$ of a small MLP $f_\theta : \mathbb{R} \to \mathbb{R}^H$ and a non-learnable exponential decay window $w(c_i) {=} \exp(-\boldsymbol{\alpha}\,c_i)$ with evenly spaced per-channel decay rates $\boldsymbol{\alpha} \in \mathbb{R}^H_{>0}$ evaluated on a grid of coordinates $\{c_i\}_{i=1}^L \subset \mathbb{R}_{\geq 0}$.
$f_\theta$ learns the kernel's functional form, while $w$ promotes learning features at multiple scales.




\vspace{-2mm}
\subsection{The HyenaND Operator}
\label{sec:HyenaND-operator}
\vspace{-1mm}

For an $N$-dimensional input $\mathbf{x} \in \mathbb{R}^{L_1 \times \cdots \times L_N \times H}$, our HyenaND operator (Eq.~\ref{eq:HyenaND}) computes $\mathbf{q}, \mathbf{k}, \mathbf{v} \in \mathbb{R}^{L_1 \times \cdots \times L_N \times H}$ from $\mathbf{x}$ as the composition of a linear projection and a short depthwise convolution.
Before the inner gate, $\mathbf{q}$ is $\ell_2$-normalized along the channel axis to bound the magnitude that the linear projection leaves unconstrained; $\mathbf{k}$ is left unnormalized as $\mathrm{SiLU}$ already plays the magnitude-control role on its side of the gate. 
The gated signal $\ell_2(\mathbf{q}) \odot \mathrm{SiLU}(\mathbf{k})$ is then RMSNormed~\citep{zhang2019root} before the long convolution to stabilize the per-token product magnitudes that enter the FFT --the same reasoning that led HyenaPixel~\citep{spravil2024hyenapixel} to place a LayerNorm at this position.
The outer gate uses $\mathrm{Sigmoid}$ for a strictly bounded $[0,1]$ selection on each channel of the convolution output, so that the gate softly attenuates without further amplifying.\footnote{Gating is more than a stabilizer: the gated Hyena mixer computes exactly the class of \emph{separable} attention matrices, connecting gated Hyena to the literature on low-displacement-rank operators. See Appendix~\ref{app:separable-attn}.}
Finally, as customary in many subquadratic operators \citep{gu2023mamba, yang2025gated}, RMSNorm is used before the residual addition to keep the operator's contribution bounded. Without it, the long convolution could amplify its input by up to $\|K\|_2$ and the residual stream's magnitude would grow with depth. The HyenaND operator is given as:
\begin{equation}
\label{eq:HyenaND}
\mathbf{y} = \mathrm{RMSNorm}\bigl(\mathrm{Sigmoid}(\mathbf{v}) \,\odot\, (K \ast \mathrm{RMSNorm}(\ell_2(\mathbf{q}) \odot \mathrm{SiLU}(\mathbf{k})))\bigr).
\end{equation}

\subsubsection{Multi-Dimensional Continuous Convolutional Kernels}
\label{sec:HyenaND-kernel}
\vspace{-1mm}
Just as in Hyena, we parameterize the convolutional kernel $K$ as the pointwise product  $K_i {=} w(\mathbf{c}_i) \odot f_\theta(\mathbf{c}_i)$ of a coordinate-based MLP $f_\theta$ and a window function $w$, but add several modifications to improve learning and adapt to multi-dimensional signals. 

\textbf{Joint $N$-dimensional parameterization.}
We define $f_\theta$ is a joint function in $\mathbb{R}^N$ rather than a separable axis-product $f_\theta(\mathbf{c}_i) {=} \prod_{n=1}^N f_{\theta,n}(c_{i,n})$ of S4ND~\citep{nguyen2022s4nd} and the product variant of \citet{zimerman2024hyenand}, which can only express rank-1 spatial structure.
This lets the kernel learn multi-\break dimensional patterns such as oriented edges, isotropic blurs, or anisotropic PDE propagation kernels.

\textbf{Coordinate-based MLP $f_\theta$.} HyenaND adapts the 1D Hyena kernel of \S\ref{sec:HyenaND-bg} to $N {\rm D}$ by replacing the 1D MLP $f_\theta : \mathbb{R} \to \mathbb{R}^H$ with a joint $N$-dimensional SIREN \citep{sitzmann2020siren} $f_\theta : \mathbb{R}^N \to \mathbb{R}^H$ evaluated on the kernel grid.
Let $\mathbf{c}_i \in [-1, 1]^N$ be a normalized coordinate vector indexing a position on the kernel grid. A SIREN with $J$ layers maps grid coordinates $\mathbf{c}_i$ to a continuous convolutional kernel as:
\begin{equation}
\label{eq:kernel}
f_\theta(\mathbf{c}_i) = \boldsymbol{\mathrm{W}}_\mathrm{out}\,\sin\bigl( \boldsymbol{\mathrm{W}}_J \sin(\cdots \sin(\boldsymbol{\omega_0} \boldsymbol{\mathrm{W}}_1 \mathbf{c}_i))\bigr),
\end{equation}
where $\boldsymbol{\mathrm{W}}_1 \in \mathbb{R}^{N{\times H_{f_\theta}}} , \boldsymbol{\mathrm{W}}_2, ..., \boldsymbol{\mathrm{W}}_J \in \mathbb{R}^{H_{f_\theta}\times H_{f_\theta}}, \boldsymbol{\mathrm{W}}_\mathrm{out} \in \mathbb{R}^{H_{f_{\theta}} \times H}$ are learned weight matrices and $\boldsymbol{\omega_0} \in \mathbb{R}^{H_{f_\theta}}$ is a frequency scaling factor. In contrast to CKConv and SIREN, which use a single non-learnable scalar $\omega_0$ shared across all channels, we parameterize $\boldsymbol{\omega_0}$ as a learnable frequency scaling vector to allow different channels to specialize in different frequency bands.

\textbf{Block-diagonal multi-frequency initialization.}
A scalar $\omega_0$ shared across all rows confines SIREN's spectrum to a narrow band at initialization, limiting the kernel's ability to simultaneously resolve low- and high-frequency content.
HyenaND broadens this coverage by structuring $\boldsymbol{\omega_0}$ into $B$ equal blocks initialized to base frequencies on a linear $[\omega_0^{\min}, \omega_0^{\max}]$ schedule.
Hidden and output layers are initialized near-block-diagonal so each block approximately behaves like an independent SIREN tuned to its own band at initialization. Off-block entries are scaled by a small value $\rho{=}0.1$ but remain free to be updated via gradient flow.
Each row of $\boldsymbol{\omega_0}$ is further allowed to scale its block's base frequency by a factor $s \in [s_{\min}, s_{\max}]$, where $s_{\min} {>} 0$ prevents the first-layer sine from collapsing to zero, and $s_{\max}$ bounds the highest effective frequency $s_{\max}\,\omega_0^{\max}$ to stays below the grid Nyquist limit to avoid aliasing. A full analysis and parameter choices are provided in Appx.~\ref{app:masking-init}.

\textbf{Window function $w$.}
The window $w$ controls the length at which different channels of the kernel operate. We replace Hyena's static exponential decay $w(c_i) {=} \exp(-\boldsymbol{\alpha}c_i)$ with a zero-centered Gaussian window $w(c_i) {=} \exp(-\tfrac{1}{2}c_i^2 / \boldsymbol{\sigma}^2)$, parameterized by a learnable per-channel variance $\boldsymbol{\sigma}^2 \in \mathbb{R}^H_{>0}$.\break Building on the $\boldsymbol{\omega}_0$ initialization above, we initialize $\boldsymbol{\sigma}^2$ to be block-aligned with the $\boldsymbol{\omega}_0$ schedule such that rows in the lowest-$\omega_0$ block receive the widest initial bandwidths, with bandwidths shrinking as the block frequency increases. We choose Gaussian over exponential decay for two reasons: Gaussian windows are standard in signal processing for their isotropic decay and minimal time-frequency uncertainty product, and they avoid the harsher decay that the exponential axis-product exhibits in multiple dimensions. For $N\rm D$ signals, we give the variance vector an $N$-dimensional form $\boldsymbol{\sigma}^2 \in \mathbb{R}^{N \times H}{>0}$, such that each spatial axis can adapt its receptive field independently during training. Ablations in Appx.~\ref{app:masking-init} empirically validate our choice of Gaussian over exponential windows. 

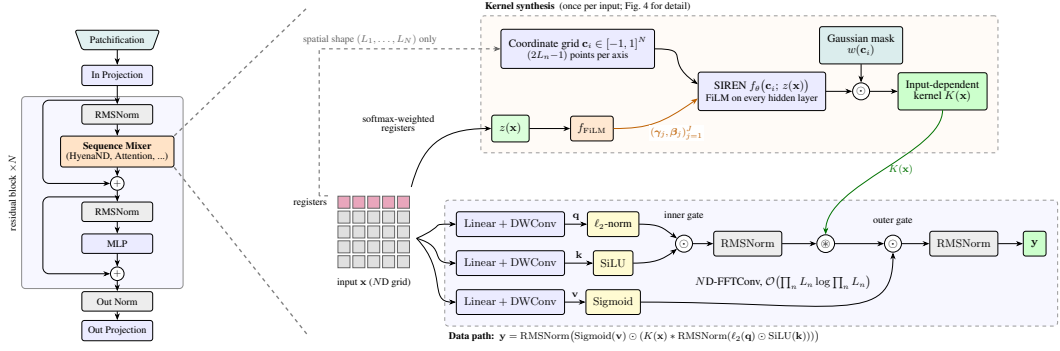
\begin{figure}
    \centering
    \resizebox{\linewidth}{!}{\begin{tikzpicture}[
  font=\small,
  box/.style={draw, rounded corners=2pt, minimum height=7mm, align=center, fill=blue!8, inner sep=2.2mm},
  filmbox/.style={draw, rounded corners=2pt, minimum height=7mm, align=center, fill=orange!18, inner sep=2mm},
  actbox/.style={draw, rounded corners=2pt, minimum height=7mm, align=center, fill=yellow!22, inner sep=2mm},
  normbox/.style={draw, rounded corners=2pt, minimum height=7mm, align=center, fill=gray!14, inner sep=2mm},
  op/.style={draw, circle, inner sep=1.2pt, fill=white, font=\normalsize},
  lbl/.style={font=\footnotesize},
  note/.style={font=\footnotesize\itshape, align=center},
  arr/.style={-{Stealth[length=2mm]}, thick},
  farr/.style={-{Stealth[length=2mm]}, thick, orange!75!black},
  karr/.style={-{Stealth[length=2mm]}, thick, green!45!black},
  tok/.style={minimum size=3.4mm, inner sep=0pt, draw=black!60}
]

\begin{scope}[shift={(-6.4,0)}, font=\footnotesize]
  \tikzset{sbox/.style={draw, rounded corners=2pt, minimum height=5.5mm, minimum width=20mm, align=center, fill=blue!8, inner sep=1.6mm}}
  \node[sbox, fill=teal!12, trapezium, trapezium left angle=110, trapezium right angle=110] (patch) at (0, 5.0) {Patchification};
  \node[sbox] (inproj) at (0, 4.0) {In Projection};
  \node[sbox, fill=gray!14] (n1) at (0, 2.9) {RMSNorm};
  \node[sbox, fill=orange!22, minimum height=8mm] (mix) at (0, 1.85) {\textbf{Sequence Mixer}\\[-1pt](HyenaND, Attention, ...)};
  \node[op, scale=0.8] (p1) at (0, 0.95) {$+$};
  \node[sbox, fill=gray!14] (n2) at (0, 0.15) {RMSNorm};
  \node[sbox] (mlp) at (0, -0.75) {MLP};
  \node[op, scale=0.8] (p2) at (0, -1.6) {$+$};
  \draw[arr] (inproj) -- (n1);
  \draw[arr] (n1) -- (mix);
  \draw[arr] (mix) -- (p1);
  \draw[arr] (p1) -- (n2);
  \draw[arr] (n2) -- (mlp);
  \draw[arr] (mlp) -- (p2);
  \draw[arr, rounded corners=5pt] (0, 3.3) -- (-2.1, 3.3) -- (-2.1, 0.95) -- (p1.west);
  \draw[arr, rounded corners=5pt] (0, 0.57) -- (-2.1, 0.57) -- (-2.1, -1.6) -- (p2.west);
  \coordinate (fitL) at (-2.45, 1);
  \begin{scope}[on background layer]
    \node[draw=black!45, rounded corners=3pt, fit=(n1)(mix)(p1)(n2)(mlp)(p2)(fitL), inner xsep=2.5mm, inner ysep=2.2mm, fill=blue!3] (blk) {};
  \end{scope}
  \node[lbl, anchor=south, rotate=90] at (blk.west) {residual block $\times N$};
  \node[sbox, fill=gray!14, below=3mm of p2] (outnorm) {Out Norm};
  \node[sbox, below=2.5mm of outnorm] (outproj) {Out Projection};
  \draw[arr] (patch) -- (inproj);
  \draw[arr] (p2) -- (outnorm);
  \draw[arr] (outnorm) -- (outproj);
\end{scope}

\draw[black!50, dashed, line width=1.1pt] (mix.north east) -- (-1.05, 5.9);
\draw[black!50, dashed, line width=1.1pt] (mix.south east) -- (-1.05, -3.3);

\foreach \i in {0,...,4} \node[tok, fill=purple!35] at (0.42*\i, 0.42) {};
\foreach \j in {0,...,3} \foreach \i in {0,...,4}
  \node[tok, fill=black!12] at (0.42*\i, -0.42*\j) {};
\node[lbl, anchor=east] at (-0.35, 0.42) {registers};
\node[lbl, anchor=north] at (0.85, -1.6) {input $\mathbf{x}$ {\footnotesize($N$D grid)}};

\node[box, fill=green!14] (z) at (4.7, 2.5) {$z(\mathbf{x})$};
\draw[arr] (1.9, 0.85) .. controls +(0.3,1.1) and +(-1.5,0) .. node[lbl, above left=1.5mm and -1mm, align=center] {softmax-weighted\\[-2pt]registers} (z.west);
\node[filmbox] (film) at (7.0, 2.5) {$f_{\mathrm{FiLM}}$};
\draw[arr] (z) -- (film);

\node[box] (coords) at (6.6, 4.75) {Coordinate grid $\mathbf{c}_i \in [-1,1]^N$\\[-1pt]{\footnotesize $(2L_n{-}1)$ points per axis}};
\draw[arr, dashed, black!60] (-0.7, 0.6) -- (-0.7, 4.75) -- node[lbl, above, pos=0.28] {spatial shape $(L_1, \ldots, L_N)$ only} (coords.west);
\draw[black!60, dashed] (-0.45, 0.6) -- (-0.7, 0.6);
\node[box, minimum width=30mm] (siren) at (11.8, 3.55) {SIREN $f_\theta\bigl(\mathbf{c}_i;\, z(\mathbf{x})\bigr)$\\[-1pt]{\footnotesize FiLM on every hidden layer}};
\draw[arr] (coords.east) .. controls +(1.1,0) and +(-1.0,0.7) .. ([yshift=2mm]siren.west);
\draw[farr] (film.east) .. controls +(1.0,0) and +(-1.0,-0.7) .. node[lbl, below right=0.5mm and -2mm, fill=white, inner sep=1pt] {$(\bm{\gamma}_j, \bm{\beta}_j)_{j=1}^J$} ([yshift=-2mm]siren.west);

\node[op] (khad) at (14.6, 3.55) {$\odot$};
\node[box, fill=teal!14] (win) at (14.6, 4.85) {Gaussian mask\\[-1pt]$w(\mathbf{c}_i)$};
\draw[arr] (siren) -- (khad);
\draw[arr] (win) -- (khad);
\node[box, fill=green!20] (K) at (16.9, 3.55) {Input-dependent\\[-1pt]kernel $K(\mathbf{x})$};
\draw[arr] (khad) -- (K);

\begin{scope}[on background layer]
  \node[draw=black!40, dashed, rounded corners=3pt, fit=(z)(film)(coords)(siren)(khad)(win)(K), inner sep=3mm, fill=orange!3] (toplane) {};
\end{scope}
\node[lbl, anchor=south west] at (toplane.north west) {\textbf{Kernel synthesis} \ {\footnotesize(once per input; Fig.~\ref{fig:film-conditioning} for detail)}};

\node[box] (pq) at (4.7, -0.2) {$\mathrm{Linear} + \mathrm{DWConv}$};
\node[box] (pk) at (4.7, -1.3) {$\mathrm{Linear} + \mathrm{DWConv}$};
\node[box] (pv) at (4.7, -2.4) {$\mathrm{Linear} + \mathrm{DWConv}$};
\draw[arr] (2.05, -0.6) .. controls +(0.9,0.4) and +(-0.8,0) .. (pq.west);
\draw[arr] (2.05, -0.6) .. controls +(0.9,-0.25) and +(-0.8,0) .. (pk.west);
\draw[arr] (2.05, -0.6) .. controls +(0.9,-0.95) and +(-0.8,0) .. (pv.west);

\node[actbox] (l2) at (7.6, -0.2) {$\ell_2$-norm};
\node[actbox] (silu) at (7.6, -1.3) {$\mathrm{SiLU}$};
\node[actbox] (sig) at (7.6, -2.4) {$\mathrm{Sigmoid}$};
\draw[arr] (pq) -- node[lbl, above] {$\mathbf{q}$} (l2);
\draw[arr] (pk) -- node[lbl, above] {$\mathbf{k}$} (silu);
\draw[arr] (pv) -- node[lbl, above] {$\mathbf{v}$} (sig);

\node[op] (gate1) at (9.6, -0.75) {$\odot$};
\draw[arr] (l2.east) .. controls +(0.6,0) and +(-0.6,0.45) .. (gate1);
\draw[arr] (silu.east) .. controls +(0.6,0) and +(-0.6,-0.45) .. (gate1);
\node[lbl, above=2.8mm of gate1] {inner gate};

\node[normbox] (rms1) at (11.4, -0.75) {$\mathrm{RMSNorm}$};
\draw[arr] (gate1) -- (rms1);

\node[op, minimum size=5mm] (conv) at (13.6, -0.75) {$\circledast$};
\draw[arr] (rms1) -- (conv);
\draw[karr] (K.south) .. controls +(0,-1.2) and +(0,1.4) .. node[lbl, right, pos=0.5] {$K(\mathbf{x})$} (conv.north);
\node at (12.4, -1.85) {$N$D-FFTConv, $\mathcal{O}\bigl(\textstyle\prod_n L_n \log \prod_n L_n\bigr)$};

\node[op] (gate2) at (15.5, -0.75) {$\odot$};
\draw[arr] (conv) -- (gate2);
\draw[arr] (sig.east) .. controls +(5.5,0) and +(0,-2.1) .. (gate2.south);
\node[lbl, above=1.5mm of gate2] {outer gate};

\node[normbox] (rms2) at (17.5, -0.75) {$\mathrm{RMSNorm}$};
\draw[arr] (gate2) -- (rms2);
\node[box, fill=green!20] (y) at (19.5, -0.75) {$\mathbf{y}$};
\draw[arr] (rms2) -- (y);

\begin{scope}[on background layer]
  \node[draw=black!40, dashed, rounded corners=3pt, fit=(pq)(pv)(l2)(sig)(gate1)(rms1)(conv)(gate2)(rms2)(y), inner sep=3.2mm, fill=blue!3] (botlane) {};
\end{scope}
\node[lbl, anchor=north west] at (botlane.south west) {\textbf{Data path:}\ \ $\mathbf{y} = \mathrm{RMSNorm}\bigl(\mathrm{Sigmoid}(\mathbf{v}) \odot (K(\mathbf{x}) \ast \mathrm{RMSNorm}(\ell_2(\mathbf{q}) \odot \mathrm{SiLU}(\mathbf{k})))\bigr)$};

\end{tikzpicture}}
    \vspace{-4mm}
    \caption{The HyenaND operator. The input is an $N{\rm D}$ token grid with $R$ register tokens prepended along the first axis. \emph{Top (kernel synthesis):} patch coordinates feed a SIREN MLP $f_\theta$, are masked by a learned Gaussian window $w$, and FiLM-conditioned on the a control variable $z(\mathbf{x})$, yielding the input-dependent kernel $K(\mathbf{x})$ (\S\ref{sec:HyenaND-kernel}--\ref{sec:HyenaND-registers}). \emph{Bottom (data path):} the input is projected into $\mathbf{q},\mathbf{k},\mathbf{v}$; the inner gate $\ell_2(\mathbf{q}) \odot \mathrm{SiLU}(\mathbf{k})$ is normalized and convolved with $K(\mathbf{x})$ in $\mathcal{O}\bigl(\prod_n L_n \log \prod_n L_n\bigr)$. Then the outer gate $\mathrm{Sigmoid}(\mathbf{v})$ conditions the output and is RMSNormed to yield $\mathbf{y}$. \emph{Left:} the operator sits as the sequence mixer of pre-norm residual blocks ($\times N$), interchangeable with attention or Mamba, between patchification/in-projection and out-norm/out-projection.
    \vspace{-3mm}}
    \label{fig:placeholder}
\end{figure}

\begin{figure}
    \centering
    \resizebox{0.98\linewidth}{!}{\begin{tikzpicture}[
  font=\small,
  box/.style={draw, rounded corners=2pt, minimum height=7mm, align=center, fill=blue!8, inner sep=2.2mm},
  filmbox/.style={draw, rounded corners=2pt, minimum height=7mm, align=center, fill=orange!18, inner sep=2mm},
  op/.style={draw, circle, inner sep=1.2pt, fill=white},
  lbl/.style={font=\footnotesize},
  note/.style={font=\footnotesize\itshape, align=center},
  arr/.style={-{Stealth[length=2mm]}, thick},
  farr/.style={-{Stealth[length=2mm]}, thick, orange!75!black},
  tok/.style={minimum size=3.4mm, inner sep=0pt, draw=black!60}
]

\foreach \i in {0,...,5} \node[tok, fill=purple!35] at (0.42*\i, 0.42) {};
\foreach \j in {0,...,3} \foreach \i in {0,...,5}
  \node[tok, fill=black!12] at (0.42*\i, -0.42*\j) {};
\node[lbl, anchor=south west] at (-0.25, 0.72) {registers $\{\mathbf{e}_r\}_{r=1}^{R}$ (prepended rows)};
\node[lbl, anchor=north west] at (-0.25, -1.55) {input tokens $\mathbf{x}$};

\node[op] (sum) at (5.3, 0.42) {$\Sigma$};
\node[lbl, below=1mm of sum] {$\operatorname{softmax}(\bm{\lambda})_r$};
\draw[arr] (2.35, 0.42) -- node[lbl, above] {$\{\mathbf{e}'_r(\mathbf{x})\}_{r=1}^{R}$} (sum);

\node[box, fill=green!14] (z) at (7.6, 0.42) {$z(\mathbf{x})$};
\draw[arr] (sum) -- (z);

\node[filmbox] (film) at (10.3, 0.42) {$f_{\mathrm{FiLM}}$ {\footnotesize(2-layer MLP)}};
\draw[arr] (z) -- (film);

\node[filmbox, fill=orange!28] (gb) at (14.2, 0.42) {$(\bm{\gamma}_j,\, \bm{\beta}_j)_{j=1}^{J}$};
\draw[arr] (film) -- (gb);
\node[note] at (14.2, 1.35) {at init $(\bm{\gamma}_j,\bm{\beta}_j){=}(\mathbf{1},\mathbf{0})$:\ $K(\mathbf{x}){=}K_0$};

\node[box] (coords) at (1.2, -3.6) {coordinates $\mathbf{c}_i$};
\node[box] (l1) at (4.4, -3.6) {$\mathrm{Linear}$, $\sin(\omega_0 \cdot)$};
\node[filmbox] (f1) at (8.15, -3.6) {$\mathbf{h}_1 \leftarrow \bm{\gamma}_1 {\odot} \mathbf{h}_1 {+} \bm{\beta}_1$};
\node[box] (l2) at (11.5, -3.6) {$\mathrm{Linear}$, $\sin$};
\node[filmbox] (f2) at (14.7, -3.6) {$\mathbf{h}_2 \leftarrow \bm{\gamma}_2 {\odot} \mathbf{h}_2 {+} \bm{\beta}_2$};
\node (dots) at (17.05, -3.6) {$\cdots$};
\node[box] (lout) at (18.4, -3.6) {$\mathrm{Linear}$};
\node[op] (had) at (20.3, -3.6) {$\odot$};
\node[box, fill=teal!14] (win) at (17.6, -5.3) {Gaussian window $w(\mathbf{c}_i)$};
\node[box, fill=green!20] (K) at (23.9, -3.6) {$K_i(\mathbf{x}) = w(\mathbf{c}_i) \odot f_\theta\bigl(\mathbf{c}_i;\, z(\mathbf{x})\bigr)$};

\draw[arr] (coords) -- (l1);
\draw[arr] (l1) -- (f1);
\draw[arr] (f1) -- (l2);
\draw[arr] (l2) -- (f2);
\draw[arr] (f2) -- (dots);
\draw[arr] (dots) -- (lout);
\draw[arr] (lout) -- (had);
\draw[arr] (win.east) .. controls +(1.2,0) and +(0,-0.8) .. (had.south);
\draw[arr] (had) -- (K);

\begin{scope}[on background layer]
  \node[draw=black!45, dashed, rounded corners=3pt, fit=(l1)(f1)(l2)(f2)(dots)(lout), inner ysep=3.5mm, inner xsep=2.5mm, fill=blue!3] (siren) {};
\end{scope}
\node[lbl, anchor=north] at (siren.south) {SIREN kernel MLP $f_\theta(\,\cdot\,;\, z(\mathbf{x}))$: FiLM modulates every hidden layer};

\draw[farr] (gb.south) .. controls +(0,-0.9) and +(0,1.1) .. node[lbl, pos=0.45, above left=1pt and 2pt, fill=white, inner sep=1pt] {$\bm{\gamma}_1, \bm{\beta}_1$} (f1.north);
\draw[farr] (gb.south) .. controls +(0.4,-0.7) and +(0,1.1) .. node[lbl, pos=0.4, right=2pt, fill=white, inner sep=1pt] {$\bm{\gamma}_2, \bm{\beta}_2$} (f2.north);

\node[note] at (24.3, -4.95) {one kernel per input, identical at all positions\\$\Rightarrow$ the single $N$D-FFT convolution is preserved};

\end{tikzpicture}}
    \caption{FiLM-conditioned kernel synthesis in HyenaND. \emph{Top} (control path): register activations are aggregated into the control variable $z(\mathbf{x})$, which $f_{\mathrm{FiLM}}$ maps to per-layer scales and shifts $(\bm{\gamma}_j, \bm{\beta}_j)_{j=1}^{J}$. \emph{Bottom} (kernel synthesis): the SIREN $f_\theta$ maps coordinates to filter values, with each hidden layer FiLM-modulated (Eq.~\ref{eq:film}); the Gaussian window $w$ then yields the input-dependent kernel $K(\mathbf{x})$. Details in \S\ref{sec:HyenaND-registers}.}
    \vspace{-3mm}
    \label{fig:film-conditioning}
\end{figure}
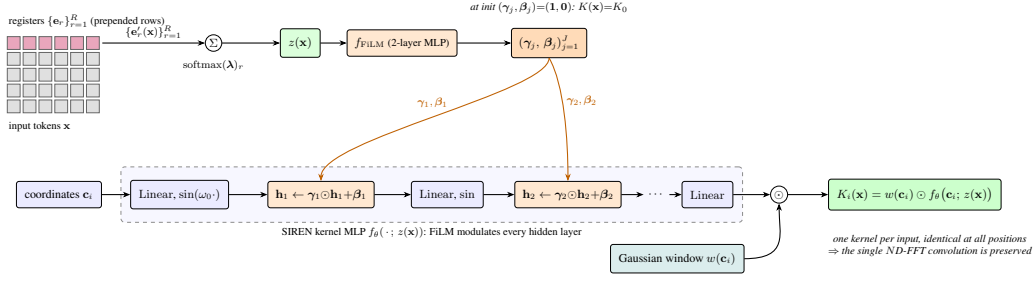

\textbf{Support for circular and non-circular convolutions.}
FFT-based convolution is naturally circular: it implicitly treats the grid as a torus, so a kernel positioned near the right edge of an image pulls in values from the left edge of the same row, as if they were stitched together.
Many physical signals -- fluid simulations, magnetohydrodynamics, reaction-diffusion systems -- exhibit periodic boundary conditions, for which circular convolutions naturally match their geometry. In such cases, global context on a signal $\mathbf{x} \in \mathbb{R}^{L_1 \times \cdots \times L_N \times H}$ is obtained by evaluating the kernel $K$ on a grid $L_1 {\times} ... {\times} L_N$ (equal to the input). For data with no adjacency between opposite edges -- natural images, $3\rm D$ medical scans, physical systems with closed boundaries -- circular warping contradicts the data's geometry: the operator would treat opposite edges as neighbors, leaking content across an artificial seam.
In such cases, we evaluate the kernel $K$ on a double grid $(2L_1 {-} 1) {\times} ... {\times} (2L_N {-} 1)$ and zero-pad the input to the same shape before the FFT to recover a non-circular convolution at constant overhead.
\texttt{nSubQ} (\S\ref{sec:cuda}) fuses padding, FFT, kernel evaluation and spectral modulation into a single CUDA kernel that keeps intermediates in on-chip shared memory at the SM level, making this overhead negligible.

\textbf{Resolution independence.}
Since $f_\theta$ and $w$ are both continuous, the same
convolutional kernel can be evaluated on grids of any size or aspect ratio without retraining \citep{romero2021ckconv}. We prove a $\mathcal{O}(\mathcal{L}_K\, b\, l\, |\Delta l| / L_0)$ bound on the output shift when the input length $L$ deviates from the training length $L_0$ by $|\Delta l|$. $\mathcal{L}_K$ denotes the Lipschitz constant of the SIREN kernel and $b$ bounds the input norms (Appx.~\ref{app:length-gen}).

\textbf{Scaling $\boldsymbol{\omega_0}$ across resolutions and dimensionalities.}
While a trained kernel can generalize to new resolutions seamlessly, initializing a network for a new target grid resolution $L^{\prime}$ or input dimensionality $N^{\prime}$ requires care. The per-block schedule on $\omega_{0}$ must be adapted at initialization to preserve the same fractional spectral coverage relative to the grid's Nyquist limit. We derive a uniform schedule $\boldsymbol{\omega_0}' {=} \boldsymbol{\omega_0}(L'/L) \sqrt{N'/N}$ (full derivation in Appx.~\ref{app:masking-init}): the $L'/L$ factor preserves the Nyquist-relative coverage as the grid density changes, and the $\sqrt{N'/N}$ factor compensates for the $1/\sqrt{N}$ shrinkage of the spectrum that arises from the SIREN initialization $\boldsymbol{\mathrm{W}}_1 \sim \mathcal{U}(-1/N, +1/N)^N$.

\vspace{-1mm}
\subsubsection{Input-Dependent Kernels via Registers and FiLM (Fig. \ref{fig:film-conditioning})}
\label{sec:HyenaND-registers}
\vspace{-1mm}

Hyena uses the \emph{same} convolutional kernel $K$ for every input $\boldsymbol{\mathrm{x}}$, limiting input adaptation. HyenaND\break removes this restriction extending the static kernel of \S\ref{sec:HyenaND-kernel} to an input-dependent kernel $K_i(\mathbf{x}) = w(\mathbf{c}_i) \odot f_\theta\bigl(\mathbf{c}_i;\, z(\mathbf{x})\bigr)$, where $z(\mathbf{x}) \in \mathbb{R}^H$ is a \textit{control} variable extracted from $\boldsymbol{\mathrm{x}}$. Since $z(\mathbf{x})$ does not depend on specific coordinates $i$, $K(\mathbf{x})$ is computed once per input and applied at every position.

\textbf{FiLM-conditioned kernel.}
We use Feature-wise Linear Modulation (FiLM)~\citep{perez2018film} to condition the kernel $f_\theta$ on the control variable $z(\mathbf{x})$. A 2-layer MLP $f_{\mathrm{FiLM}} : \mathbb{R}^H \to \mathbb{R}^{2 J H_{f_\theta}}$ uses $z(\mathbf{x})$ to produce per-layer scales and shifts $(\boldsymbol{\gamma}_j, \boldsymbol{\beta}_j)_{j=1}^J$, that modulate the SIREN hidden activations $\{\mathbf{h}_j\}_{j=1}^J$ as:
\begin{equation}
\label{eq:film}
\mathbf{h}_j \;\leftarrow\; \boldsymbol{\gamma}_j \odot \mathbf{h}_j + \boldsymbol{\beta}_j,
\qquad (\boldsymbol{\gamma}_j, \boldsymbol{\beta}_j)_{j=1}^J = f_{\mathrm{FiLM}}\bigl(z(\mathbf{x})\bigr).
\end{equation}
Following \citet{perez2018film}, we initialize the FiLM scales and shifts as $(\boldsymbol{\gamma}_j, \boldsymbol{\beta}_j) {=} (\mathbf{1}, \mathbf{0})$ to reduce to an unmodulated kernel at initialization.

\textbf{The control variable $z(\mathbf{x})$.} Squeeze-and-Excitation networks~\citep{hu2018squeeze, Zhong_2020_CVPR} modulate activations via global\break pooling. We find such summaries to be too crude to for kernel conditioning, as pooling discards the spatial structure that the kernel is meant to specialize on. Instead, we concatenate $R$ learnable \emph{register} tokens~\citep{darcet2023register} $\{\mathbf{e}_r\}_{r=1}^R {\in}\ \mathbb{R}^H$ to the input and use them for construction of $z(\mathbf{x})$. Being processed jointly at every layer, registers integrate information from the entire input, and training shapes them into the summaries that are the most useful for kernel conditioning. Let $\{\mathbf{e}'_r(\mathbf{x})\}_{r=1}^R$ denote the activations of the registers the present layer. The control variable $z(\mathbf{x})$ is computed as a softmax-weighted average over the registers $
z(\mathbf{x}) {=} \sum_{r=1}^{R} \operatorname{softmax}(\boldsymbol{\lambda})_r\, \mathbf{e}'_r(\mathbf{x})$, with $\boldsymbol{\lambda} \in \mathbb{R}^R$ learnable.

To preserve the $N{\rm D}$ structure of the input, registers are prepended as additional rows above to the input grid. For a $2 {\rm D}$ input of shape $L_1 {\times} L_2$, the $R$ registers
(with minimal zero-padding to align boundaries) fill $\lceil R / L_2 \rceil$ additional rows above the original $L_1 {\times} L_2$ patch grid.

\textbf{FFT path preservation.}
Since $z(\mathbf{x})$ is computed once per input instance, the modulated kernel $K(\mathbf{x})$ is used across all input positions. Consequently, the FFT-convolution in \eqref{eq:HyenaND} proceeds identically to the unmodulated case with the only difference that it now uses $B$ different convolutional kernels, one for each of the inputs in a batch of size $B$. Concretely, $K(\mathbf{x})$ is evaluated on the kernel grid (sized per \S\ref{sec:HyenaND-kernel}) and applied via a single $N{\rm D}$ FFT call in $\mathcal{O}\left(\prod_n L_n \log \prod_n L_n\right)$. This is in contrast to per-position input dependency, which would be incompatible with to the FFT path (\S\ref{sec:landscape-ssm}). Figure~\ref{fig:film-conditioning} (Appendix) schematizes the full conditioning path.

\vspace{-2mm}
\section{CUDA Implementation: \texttt{nSubQ}}
\label{sec:cuda}
\vspace{-2mm}
\label{sec:cuda}
Asymptotic complexity is necessary but not sufficient for practical scaling of an ML architecture: a theoretically cheaper operator is useful only when its implementation maps efficiently onto modern accelerator hardware.
This distinction is especially sharp for FFT-based operators.
GPU matrix-multiply throughput has generally grown faster than off-chip memory bandwidth~\cite{nvidia_v100_2017,nvidia_a100_2020,nvidia_h100_2022,nvidia_blackwell_2024}, which strongly benefits standard attention: the dominant SDPA operations ($\mathbf{QK}^\top$ and $\mathbf{PV}$) are GEMM-like, and FlashAttention-style fused kernels avoid materializing the full $L{\times}L$ attention matrix in HBM, making quadratic attention far more competitive than its $O(L^2)$ complexity alone would suggest~\cite{dao2022flashattention,zadouri2026flashattention4algorithmkernelpipelining}.
FFT-based operators face a different performance profile: their fast path depends less on GEMM throughput and more on batched FFT efficiency, memory layout, and bandwidth. Attempts to use tensor-cores for FFT generally lag behind cuFFT and cuFFTDx, especially for real FFTs~\cite{fu2023flashfftconv,nvidia_cufft,nvidia_cufftdx}, making the practical crossover point over fused SDPA hardware- and implementation-dependent.


To close this asymptotic-vs-wall-clock gap, \texttt{nSubQ} implements a suite of CUDA kernels for FFT-convolution that fuse forward FFT, spectral modulation, inverse FFT, and output segmentation into a single operation.
The fused kernels keep intermediate representations within fast but capacity-limited on-chip memory (see Appx.~\ref{app:cuda} for the full kernel description). \texttt{nSubQ} provides a unified API for fast long convolutions in $1\rm D$, $2\rm D$ and $3\rm D$. Furthermore, we fuse a critical path often overlooked in Hyena operators: the generation of the kernel $K$. As demonstrated in our Evo2 ablations (Appx.~\ref{app:cuda-filter-gen}), our fused generation pipeline reduces runtime by up to $40{\times}$ and HBM traffic by up to $32{\times}$. \
\vspace{-2mm}
\section{Experiments}
\label{sec:experiments}
\vspace{-1mm}

\subsection{Controlled Spatial-Recall Experiments}
\label{sec:exp-spatial-recall}
\vspace{-1mm}

Before turning to downstream tasks, we use a controlled family of probes to read off the $N$D spatial abilities of Attention, HyenaND and Mamba at parameter parity ($\sim$1.9M, 4 blocks, 50k iters). Two tasks, defined identically in 1D/2D/3D, stress complementary demands: \texttt{simple\_copy} isolates long-range copy capabilities, whereas \texttt{color\_cond} evaluates spatial associative recall (Fig.~\ref{fig:spatial_recall_tasks}).

As shown in Table \ref{tab:1d_recall}, HyenaND outperforms both Attention and Mamba across all tasks. The advantage over Mamba is one of geometry: Mamba must rasterize $N$D data into an ad-hoc 1D scan order, and this reduces accuracy when the task structure lives in the native geometry. On 3D \texttt{simple\_copy} it barely improves on the mean-predictor baseline ($0.451$ vs.\ $0.482$), while HyenaND, operating directly on the native geometry, solves the task ($1.2\mathrm{e}{-4}$). Attention fares worst: at native resolution it is only marginally better than random (e.g., $0.253$ vs.\ $0.284$ on 2D \texttt{color\_cond}) and cannot run at all in 3D. We attribute this to attention's database-inspired design: retrieval works by comparing queries against keys, which requires individual tokens to be information-dense. Yet, on a raw canvas each token is a single pixel, so the information-to-token ratio is far too low. Consistent with this reading, non-overlapping patchification, which packs many pixels into each token, makes attention competitive (though the optimal patch size remains a function of the data and the task), and the same mechanism explains why patchification is a key component on ViTs in practice (see Appendix~\ref{app:spatial-recall} for additional experiments and details about spatial recall tasks).     

\begin{figure*}[t]
    \centering
    \includegraphics[width=0.95\linewidth]{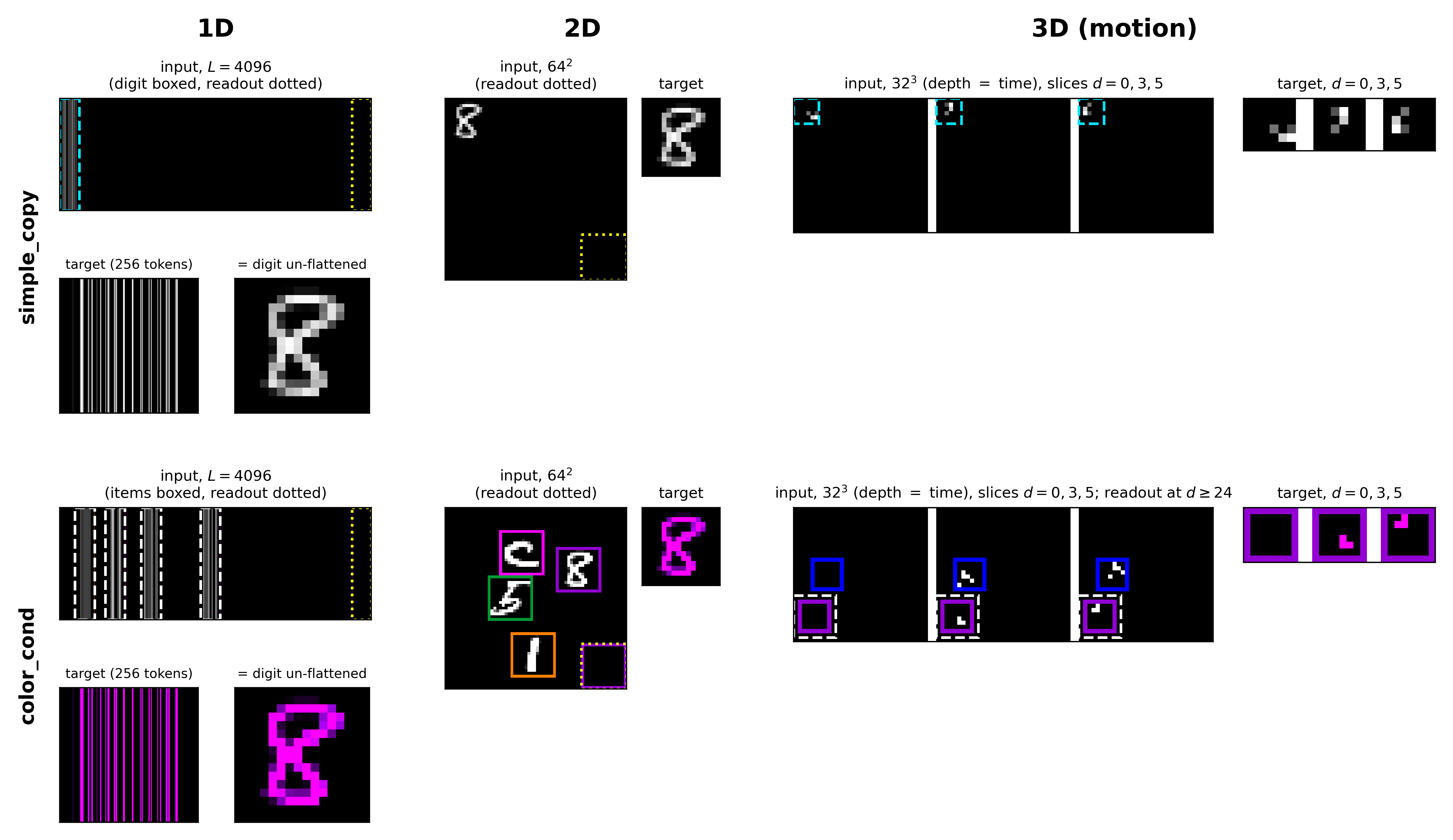}
    \vspace{-2mm}
    \caption{Spatial-recall tasks across dimensionalities (real training samples). Every example is a mostlyempty canvas -- a 1D sequence ($L{=}4096$), a 2D image ($64^2$), or a 3D volume ($32^3$, where depth describes time) -- containing one or more small EMNIST digits plus a designated readout region (dotted). The model processes the whole canvas and is trained with MSE to reproduce the correct content at the readout region. \emph{Top} (\texttt{simple\_copy}): a single digit (boxed) must be \emph{transported}, unchanged, to the readout --- pure long-range routing, with no ambiguity about what to copy; in 3D the digit additionally drifts and spins through a $6^3$ spatio-temporal tube (time slices $d{=}0,3,5$ shown for both input and target). \emph{Bottom} (\texttt{color\_cond}): four items carry distinct color markers; the readout displays one of those colors, and the model must select the item whose marker matches and reproduce it recolored: associative recall (retrieve-by-key) laid out in space, with color as the key and content as the value. The two tasks thus separate routing capacity from input-dependent selection.
    \vspace{-4mm}}
    \label{fig:spatial_recall_tasks}
\end{figure*}

\begin{table}[h]
\centering
\scriptsize
\setlength{\tabcolsep}{4pt}
\caption{Spatial-recall summary at native resolution (no patchification). MSE.
Lower is better.
}
\label{tab:spatial_recall_summary}
\begin{tabular}{l cc cc cc cc}
\toprule
 & \multicolumn{2}{c}{1D ($L{=}4096$)} & \multicolumn{2}{c}{2D ($64^2$)} & \multicolumn{2}{c}{3D motion ($32^3$)} \\
\cmidrule(lr){2-3}\cmidrule(lr){4-5}\cmidrule(lr){6-7}\cmidrule(lr){8-9}
 & simple & color & simple & color & simple & color \\
\midrule
Random (mean pred.) & \num{0.549} & \num{0.284} & \num{0.549} & \num{0.284} & \num{0.482} & \num{0.257} \\
\midrule
Attention (RoPE) & \num{0.150}            & \num{0.274}            & \num{0.092}            & \num{0.253}            & OOM                    & OOM                \\
Mamba (bidir)    & \num{0.085}            & \num{6.5e-3}           & \num{0.349}            & \num{8.3e-3}          & \num{0.451}            & \textbf{\num{1.0e-4}} \\
HyenaND (ours)   & \textbf{\num{5.2e-5}}  & \textbf{\num{6.2e-3}}  & \textbf{\num{1.3e-4}}  & \textbf{\num{2.5e-3}}     & \textbf{\num{1.2e-4}}  & \num{5.7e-5} \\
\bottomrule
\end{tabular}
\vspace{-2mm}
\end{table}

\vspace{-1mm}
\subsection{Long-Context Genomics ($1\rm D$)}
\label{sec:exp-genomics}
\label{sec:genomicsmodeling}
We train three replicates each of 1B-parameter language models at five Hyena-to-attention mixing ratios ($H_0$ through $H_4$, from fully convolutional to 4-MHA hybrid) and a full-attention baseline ($T$) on the OpenGenome2 metagenomics subset~\cite{brixi2025evo2}. Architecture and training details are in Appendix~\ref{app:genomics-setup}.

Table~\ref{tab:parametric_summary_corrected_holm} reports final validation perplexity after Holm--Bonferroni correction. Every striped Hyena configuration significantly outperforms the full transformer, consistent with the Evo scaling law~\cite{nguyen2024evo}. $H_2$ (2 MHA layers) achieves the lowest mean perplexity ($2.77 \pm 0.0006$, $p = 4.3 \times 10^{-4}$ after correction). A pre-divergence comparison at step 9{,}899 confirms the Hyena advantage is not an artifact of the LR reduction applied to the transformer runs (Appendix~\ref{app:genomics}, Table~\ref{tab:parametric_summary_corrected_holm_early}).

\begin{table}[t]
\centering
\caption{Parametric statistical comparison of model Architectures using one-sided Welch's t-test (H$_1$: mean(model) $<$ mean(reference)), with Holm--Bonferroni correction applied per reference column (5 tests each; self-comparisons excluded). While we lack the power to confirm the nominal improvement of 2 MHA layers over any of the other striped hyena configurations, every hyena configuration shows superior accuracy to the full transformer on this task, even after multiple hypothesis correction.}
\label{tab:parametric_summary_corrected_holm}
\resizebox{\columnwidth}{!}{%
\begin{tabular}{l|r@{\hspace{2pt}}l|cc|cc}
\toprule
\textbf{Model Architecture} & \multicolumn{2}{c|}{\textbf{Mean $\pm$ SD}} & \textbf{p vs. $T$} & \textbf{p vs. $H_0$} & \textbf{Holm p vs. $T$} & \textbf{Holm p vs. $H_0$} \\
\midrule
$T$ (Full Transformer, 25 MHA) & \multicolumn{2}{c|}{\num{2.9235} $\pm$ \num{0.0039}} & \num{5.000e-01} & \num{9.873e-01} & \textemdash & \num{1.000e+00} \\
$H_4$ (Mixed, 4 MHA layers) & \multicolumn{2}{c|}{\num{2.8312} $\pm$ \num{0.0088}} & \num{3.800e-04} & \num{5.650e-01} & \num{1.520e-03} & \num{1.000e+00} \\
$H_3$ (Mixed, 3 MHA layers) & \multicolumn{2}{c|}{\num{2.8214} $\pm$ \num{0.0313}} & \num{1.417e-02} & \num{3.974e-01} & \num{2.546e-02} & \num{1.000e+00} \\
\textbf{$H_2$ (Mixed, 2 MHA layers)} & \multicolumn{2}{c|}{\textbf{\num{2.7729} $\pm$ \num{0.0006}}} & \textbf{\num{8.663e-05}} & \textbf{\num{3.760e-02}} & \textbf{\num{4.332e-04}} & \textbf{\num{1.880e-01}} \\
$H_1$ (Mixed, 1 MHA layer) & \multicolumn{2}{c|}{\num{2.8308} $\pm$ \num{0.0108}} & \num{9.468e-04} & \num{5.558e-01} & \num{2.840e-03} & \num{1.000e+00} \\
$H_0$ (Full Hyena, 0 MHA) & \multicolumn{2}{c|}{\num{2.8282} $\pm$ \num{0.0279}} & \num{1.273e-02} & \num{5.000e-01} & \num{2.546e-02} & \textemdash \\
\bottomrule
\end{tabular}%
}
\vspace{-2mm}
\end{table}


\vspace{-1mm}

\vspace{-0.5mm}
\subsubsection{Context Length Scalability}
\label{sec:genomiclongcontext}
\vspace{-0.5mm}
Leveraging HyenaND's ability to enable long-context training, we take the
Evo2-inspired 1B striped-hyena model through a staged context-extension schedule
from 8K to 16M tokens on GB200, distributing the sequence dimension across nodes
via context parallelism (CP). On real, end-to-end training runs (not fixed-step
smoke tests), per-GPU throughput is stable across context length: the 1B model
holds \textbf{443--461~TFLOP/s/GPU from 128K to 16M} (a $128\times$ range), with
the single-GPU stages (16K--64K, TP=1) higher at 584--606 and only a small
decline at the largest CP degrees from inter-node collectives rather than the FFT
path (full schedule in Appendix~\ref{app:scaling}, Table~\ref{tab:scalability}).
This is a memory-conservative training config (activation recompute enabled);
with recompute disabled the operator sustains \textbf{$\sim$585--620~TFLOP/s/GPU}
up to 8M (Table~\ref{tab:scalability}, \emph{Peak}), so the as-run $\sim$450 is a
memory-conservative choice, not a throughput ceiling.
Crucially, these runs train to \emph{convergence}: validation NLL descends and
plateaus at every length (e.g.\ $1.075\!\to\!1.058$ at 16M under post-extension
long-convergence; Fig.~\ref{fig:genomics_convergence}), and the converged 16M
model's per-token NLL keeps decreasing with position depth
(Fig.~\ref{fig:genomics_posbin})---showing it \emph{uses} distant context rather
than merely tolerating the length. At these lengths an entire bacterial genome or
a chromosome-scale eukaryotic fragment fits inside a single training window,
putting genomic-scale dependencies within reach of end-to-end training at scale.%

\vspace{-1mm}
\subsection{Computer Vision: ImageNet ($2\rm D$)}
\label{sec:exp-imagenet}
\vspace{-1mm}

We evaluate HyenaND on ImageNet-1K~\cite{deng2009imagenet} classification using the ViT-5-Small backbone~\cite{dosovitskiy2020image} (22M parameters).
We compare pure HyenaND, two hybrid configurations ((H\,A)$^{\times6}$ and (H\,H\,H\,A)$^{\times3}$), and the attention baseline, none of which use patch-merge blocks. 
Training details and baseline descriptions are in Appendix~\ref{app:imagenet-setup}.

Table~\ref{tab:imagenet} reports patch $16{\times}16$ results to match the standard training regime of external baselines; the primary result is the patch-size ablation in Table~\ref{tab:imagenet_patch} (Appendix~\ref{app:imagenet_patch}), sweeping $\{16,8,4,2\}$ and reporting top-1, GFLOPs, and throughput.
Because attention converges slowly, we trained all our models for 800 epochs to ensure full convergence, while external baselines (DeiT, Mamba-ViT, etc.) report 300-epoch results. At patch $16{\times}16$, pure HyenaND achieves 81.5\% top-1, matching ViT-5-Small (81.8\%), the first native-2D subquadratic operator to match ViT on ImageNet without rasterizing the spatial grid or patch-merging blocks.
In the patch-size ablation, both our attention baseline and HyenaND clearly outperform all external baselines at comparable compute; HyenaND additionally offers FLOP savings compared to attention that grow with sequence length, 14\% at patch $8{\times}8$ (39.0 vs.\ 45.5 GFLOPs), 51\% at patch $4{\times}4$ (155 vs.\ 317 GFLOPs), and 82\% at patch $2{\times}2$ (623 vs.\ 3{,}444 GFLOPs), while remaining tractable at very small patch sizes enabled by the fused 2D FFT kernels of Appendix~\ref{app:cuda-fftconv2d}.
The hybrid (H\,A)$^{\times6}$ makes this trade-off concrete: it matches the attention baseline to within $0.1\%$ top-1 while using ${\sim}40\%$ fewer FLOPs (e.g.\ $85.2\%$ at $2{,}033$ vs.\ $85.3\%$ at $3{,}444$ GFLOPs for $p{=}2$), inheriting attention's selectivity at a fraction of its cost.
Beyond this FLOP trade-off, at matched hybrid ratios the HyenaND-attention hybrids also outperform recurrence-based Mamba-attention hybrids (Table~\ref{tab:imagenet}; full patch sweep in Table~\ref{tab:imagenet_patch}): $82.1\%$ vs.\ $81.8\%$ at the $1{:}1$ ratio and $82.0\%$ vs.\ $81.0\%$ at the $3{:}1$ ratio (patch $16{\times}16$), even though the bidirectional Mamba hybrids carry markedly more parameters ($30$M/$34$M vs.\ $22$M).

\begin{figure}[!ht]
\centering
\begin{minipage}[t]{0.50\linewidth}
\vspace{0pt}
\centering
\captionof{table}{ImageNet-1K top-1 accuracy at $224\times224$. For our models, results are at the best-performing patch size; see Table~\ref{tab:imagenet_patch} for the full ablation (incl.\ GFLOPs and img/s). All HyenaND models use the ViT-5-Small backbone. $^\dagger$Same backbone and training recipe as HyenaND.}
\label{tab:imagenet}
\scriptsize
\setlength{\tabcolsep}{3pt}
\resizebox{\linewidth}{!}{
\begin{tabular}{lccccc}
\toprule
\textbf{Model} & \textbf{Type} & \textbf{Patch} & \textbf{Merge} & \textbf{Params} & \textbf{Top-1 (\%)} \\
\midrule
DeiT-S~\cite{touvron2021training}          & Attention     & $16{\times}16$ & False & 22M & 79.8   \\
Swin-T~\cite{liu2021swin}                  & Attention     & $4{\times}4$   & True  & 29M & 81.3   \\
ConvNeXt-T~\cite{liu2022convnet}           & Convolution   & $4{\times}4$   & True  & 29M & 82.1   \\
\midrule
Vim-S~\cite{zhu2024vim}                    & GRM (1D scan) & $16{\times}16$ & False & 26M & 80.3   \\
VMamba-T~\cite{liu2024vmamba}              & GRM (1D scan) & $4{\times}4$   & True  & 30M & 82.6   \\
\midrule
ViT-5-Small (Attention)$^\dagger$          & Attention     & $16{\times}16$ & False & 22M & 81.80  \\
Mamba-S (MA)$^{\times6}$             & Hybrid        & $16{\times}16$ & False & 30M & 81.80  \\
Mamba-S (MMMA)$^{\times3}$         & Hybrid        & $16{\times}16$ & False & 34M & 81.00  \\
HyenaND-S (pure)                     & GCM (2D)      & $16{\times}16$ & False & 22M & 81.50  \\
HyenaND-S (HA)$^{\times6}$           & Hybrid        & $16{\times}16$ & False & 22M & 82.10  \\
HyenaND-S (HHHA)$^{\times3}$       & Hybrid        & $16{\times}16$ & False & 22M & 82.00  \\
\bottomrule
\end{tabular}
}
\end{minipage}\hfill
\begin{minipage}[t]{0.48\linewidth}
\vspace{0pt}
%
%

\definecolor{vit5blue}{HTML}{3B6FB6}
\definecolor{hahyellow}{HTML}{E0A337}
\definecolor{hhhateal}{HTML}{3F8C7F}
\definecolor{hyenared}{HTML}{C0392B}
\label{fig:imagenet_gflops_acc}
\raggedright
\begin{tikzpicture}
\begin{axis}[
    width=\linewidth, height=5.0cm,
    scale only axis=false,
    xmode=log,
    xlabel={GFLOPs}, ylabel={Top-1 Acc.},
    xlabel style={font=\small, yshift=2pt},
    ylabel style={font=\small, yshift=-4pt},
    tick label style={font=\small},
    ymajorgrids=true,
    grid style={dotted, gray!40},
    every axis plot/.append style={line width=2.0pt},
    ytick={0.82,0.83,0.84,0.85},
    yticklabels={82,83,84,85},
    xmin=8, xmax=4500,
    ymin=0.8115, ymax=0.858,   
    xtick={10,40,300,3000}, xticklabels={10,40,300,3000},
    legend columns=1,
    legend style={
        at={(0.97,0.03)}, anchor=south east,
        draw=gray!60, fill=white, fill opacity=0.75, text opacity=1,
        font=\tiny,
        row sep=-2pt,
        inner sep=1pt,
        legend cell align=left,
    },
]
\addplot[color=vit5blue, mark=*, mark size=2.5pt,
    mark options={draw=black, fill=vit5blue, line width=1.0pt}]
    coordinates {(9.41,0.818)(45.51,0.843)(317.34,0.851)(3443.93,0.853)};
\addlegendentry{ViT-5-S}
\addplot[color=hahyellow, mark=square*, mark size=2.5pt,
    mark options={draw=black, fill=hahyellow, line width=1.0pt}]
    coordinates {(9.69,0.821)(42.27,0.842)(236.33,0.850)(2033.48,0.852)};
\addlegendentry{HyenaND-S(HA)}
\addplot[color=hhhateal, mark=triangle*, mark size=2.5pt,
    mark options={draw=black, fill=hhhateal, line width=1.0pt}]
    coordinates {(9.83,0.820)(40.66,0.840)(195.83,0.844)(1328.25,0.846)};
\addlegendentry{HyenaND-S(HHHA)}
\addplot[color=hyenared, mark=diamond*, mark size=2.5pt,
    mark options={draw=black, fill=hyenared, line width=1.0pt}]
    coordinates {(9.97,0.815)(39.04,0.837)(155.32,0.840)(623.02,0.840)};
\addlegendentry{HyenaND-S}
\node[font=\tiny, anchor=south] at (axis cs:10.5,0.8235) {p16}; 
\node[font=\tiny, anchor=south] at (axis cs:42,0.8448)   {p8};  
\node[font=\tiny, anchor=south] at (axis cs:273,0.8525)  {p4};  
\node[font=\tiny, anchor=north] at (axis cs:174,0.8382)  {p4};  
\node[font=\tiny, anchor=south] at (axis cs:2650,0.8538) {p2};  
\node[font=\tiny, anchor=south] at (axis cs:1328,0.8471) {p2};  
\node[font=\tiny, anchor=north] at (axis cs:623,0.8382)  {p2};  
\end{axis}
\end{tikzpicture}
\captionof{figure}{\textbf{ImageNet Top-1 vs.\ compute.} Each curve sweeps patch size $p{\in}\{16,8,4,2\}$; smaller patches yield more tokens and more FLOPs.}
\end{minipage}
\vspace{-2mm}
\end{figure}

\vspace{-1mm}
\subsection{PDE Surrogate Modeling on \thewell{} ($2\rm D$/$3\rm D$)}
\vspace{-1mm}

We assess HyenaND on eight PDE datasets from \thewell~\cite{ohana2024well}, spanning fluid dynamics, MHD, astrophysics, acoustics, and biological pattern formation. The datasets were chosen to cover resolutions of $64$--$1024$-pixel on $2\rm D$ and $3\rm D$ Cartesian grids. We compare HyenaND against a full-resolution ConvNeXt U-Net baseline and a matched Attention model across patch sizes $p \in \{2, 4, 8\}$ for total run times of 24 hours per experiment. Full results are compiled in Table~\ref{tab:well} (Appendix~\ref{app:well-setup}), together with architecture and training details. Following the benchmark protocol, we report variance-normalized RMSE (VRMSE), scaled so that the mean-field predictor scores $1$.

\begin{figure}[t]
    \centering
    \resizebox{\linewidth}{!}{%
    \begin{tikzpicture}
    \begin{groupplot}[
        group style={group size=5 by 2, horizontal sep=0.45cm, vertical sep=1.5cm},
        width=3.2cm, height=3.2cm,
        xmode=log,
        xlabel={Seq. length},
        xlabel style={font=\footnotesize},
        ylabel style={font=\footnotesize},
        tick label style={font=\scriptsize},
        title style={font=\scriptsize\bfseries, align=center, yshift=1pt},
        ymajorgrids=true,
        grid style={dotted, gray!40},
        every axis plot/.append style={line width=2.5pt},
    ]
    \nextgroupplot[title={Gray-Scott\\React.-Diff.}, ylabel={VRMSE}, xlabel={},
        scaled y ticks=base 10:1, ytick={0.1,0.2},
        xmin=128, xmax=8192,
        xtick={256,1024,4096}, xticklabels={256,1K,4K}]
    \addplot[color=attnblue, mark=*, mark size=2.5pt, mark options={draw=black, fill=attnblue, line width=1.0pt}]
        coordinates {(256,0.0520)(1024,0.0538)(4096,0.0974)};
    \addplot[color=hyenared, mark=square*, mark size=2.5pt, mark options={draw=black, fill=hyenared, line width=1.0pt}]
        coordinates {(256,0.0092)(1024,0.0090)(4096,0.0091)};
    \addplot[color=cnextgreen, dashed, mark=none]
        coordinates {(128,0.2319)(8192,0.2319)};
    \nextgroupplot[title={Active\\Matter}, xlabel={},
        scaled y ticks=base 10:2, ytick={0.02,0.04,0.06,0.08},
        xmin=512, xmax=32768,
        xtick={1024,4096,16384}, xticklabels={1K,4K,16K}]
    \addplot[color=attnblue, mark=*, mark size=2.5pt, mark options={draw=black, fill=attnblue, line width=1.0pt}]
        coordinates {(1024,0.0586)(4096,0.0616)(16384,0.0914)};
    \addplot[color=hyenared, mark=square*, mark size=2.5pt, mark options={draw=black, fill=hyenared, line width=1.0pt}]
        coordinates {(1024,0.0073)(4096,0.0080)(16384,0.0070)};
    \addplot[color=cnextgreen, dashed, mark=none]
        coordinates {(512,0.0347)(32768,0.0347)};
    \nextgroupplot[title={Acoustic\\Scattering}, xlabel={},
        scaled y ticks=base 10:2, ytick={0.03,0.06,0.09},
        xmin=512, xmax=32768,
        xtick={1024,4096,16384}, xticklabels={1K,4K,16K}]
    \addplot[color=attnblue, mark=*, mark size=2.5pt, mark options={draw=black, fill=attnblue, line width=1.0pt}]
        coordinates {(1024,0.0456)(4096,0.0569)(16384,0.1057)};
    \addplot[color=hyenared, mark=square*, mark size=2.5pt, mark options={draw=black, fill=hyenared, line width=1.0pt}]
        coordinates {(1024,0.0086)(4096,0.0068)(16384,0.0062)};
    \addplot[color=cnextgreen, dashed, mark=none]
        coordinates {(512,0.0082)(32768,0.0082)};
    \nextgroupplot[title={MHD}, xlabel={},
        scaled y ticks=base 10:1, ytick={0.1,0.2,0.3},
        xmin=256, xmax=65536,
        xtick={512,4096,32768}, xticklabels={512,4K,32K}]
    \addplot[color=attnblue, mark=*, mark size=2.5pt, mark options={draw=black, fill=attnblue, line width=1.0pt}]
        coordinates {(512,0.3044)(4096,0.2164)(32768,0.3037)};
    \addplot[color=hyenared, mark=square*, mark size=2.5pt, mark options={draw=black, fill=hyenared, line width=1.0pt}]
        coordinates {(512,0.2810)(4096,0.1088)(32768,0.0543)};
    \addplot[color=cnextgreen, dashed, mark=none]
        coordinates {(256,0.2108)(65536,0.2108)};
    \nextgroupplot[hide axis, xmode=normal, ymode=normal, xmin=0, xmax=1, ymin=0, ymax=1, xshift=4pt,
        legend style={at={(0.5,0.5)}, anchor=center, draw=none, fill=none,
            font=\footnotesize, row sep=4pt}]
    \addlegendimage{color=attnblue, mark=*, mark size=2.5pt, line width=2.5pt, mark options={draw=black, fill=attnblue, line width=1.0pt}}
    \addlegendentry{Attention}
    \addlegendimage{color=hyenared, mark=square*, mark size=2.5pt, line width=2.5pt, mark options={draw=black, fill=hyenared, line width=1.0pt}}
    \addlegendentry{HyenaND}
    \addlegendimage{color=cnextgreen, dashed, line width=2.5pt}
    \addlegendentry{CNextU-net}
    \nextgroupplot[title={Supernova\\Explosion}, ylabel={VRMSE},
        scaled y ticks=base 10:1, ytick={0.2,0.4,0.6},
        xmin=256, xmax=65536,
        xtick={512,4096,32768}, xticklabels={512,4K,32K}]
    \addplot[color=attnblue, mark=*, mark size=2.5pt, mark options={draw=black, fill=attnblue, line width=1.0pt}]
        coordinates {(512,0.6117)(4096,0.3879)(32768,0.3000)};
    \addplot[color=hyenared, mark=square*, mark size=2.5pt, mark options={draw=black, fill=hyenared, line width=1.0pt}]
        coordinates {(512,0.6151)(4096,0.3578)(32768,0.1943)};
    \addplot[color=cnextgreen, dashed, mark=none]
        coordinates {(256,0.7400)(65536,0.7400)};
    \nextgroupplot[title={Shear\\Flow},
        scaled y ticks=base 10:1, ytick={0.05,0.10,0.15}, ymax=0.15,
        xmin=1024, xmax=65536,
        xtick={2048,8192,32768}, xticklabels={2K,8K,32K}]
  \addplot[color=attnblue, mark=*, mark size=2.5pt, mark options={draw=black, fill=attnblue, line width=1.0pt}]
      coordinates {(2048,0.0354)(8192,0.1049)};
  \addplot[color=hyenared, mark=square*, mark size=2.5pt, mark options={draw=black, fill=hyenared, line width=1.0pt}]
      coordinates {(2048,0.0268)(8192,0.0823)(32768,0.1185)};
  \addplot[color=cnextgreen, dashed, mark=none]
      coordinates {(1024,0.0262)(65536,0.0262)};
  \nextgroupplot[title={Euler Multi\\Quadrants},
      scaled y ticks=base 10:2, ytick={0.05,0.10},
      xmin=2048, xmax=131072,
      xtick={4096,16384,65536}, xticklabels={4K,16K,64K}]
  \addplot[color=attnblue, mark=*, mark size=2.5pt, mark options={draw=black, fill=attnblue, line width=1.0pt}]
      coordinates {(4096,0.1293)};
  \addplot[color=hyenared, mark=square*, mark size=2.5pt, mark options={draw=black, fill=hyenared, line width=1.0pt}]
      coordinates {(4096,0.0311)(16384,0.0378)(65536,0.1088)};
  \addplot[color=cnextgreen, dashed, mark=none]
      coordinates {(2048,0.0332)(131072,0.0332)};
  \nextgroupplot[title={Helmholtz\\Staircase},
      ymode=log,
      ytick={0.005, 0.01, 0.05},
      yticklabels={0.005, 0.01, 0.05},
      xmin=2048, xmax=131072,
      xtick={4096,16384,65536}, xticklabels={4K,16K,64K}]
  \addplot[color=attnblue, mark=*, mark size=2.5pt, mark options={draw=black, fill=attnblue, line width=1.0pt}]
      coordinates {(4096,0.0050)(16384,0.0147)};
  \addplot[color=hyenared, mark=square*, mark size=2.5pt, mark options={draw=black, fill=hyenared, line width=1.0pt}]
      coordinates {(4096,0.0042)(16384,0.0480)(65536,0.0443)};
  \addplot[color=cnextgreen, dashed, mark=none]
      coordinates {(2048,0.0045)(131072,0.0045)};
  \nextgroupplot[hide axis, xmode=normal, ymode=normal, xmin=0, xmax=1, ymin=0, ymax=1]
  \end{groupplot}
  \end{tikzpicture}
  }
  \caption{\textbf{VRMSE vs.\ sequence length on \thewell{}.} Sequence length is controlled by patch size: smaller patches $p$ yield more tokens. Each panel shows one dataset (lower is better). HyenaND (red) improves as sequence length grows on most datasets. CNextU-net (green dashed) is a fixed full-resolution baseline with no patch-size axis. For Shear Flow, the Attention $p{=}2$ result is omitted (undertrained run).}
  \label{fig:well_vrmse_scatter}
  \vspace{-3mm}
\end{figure}
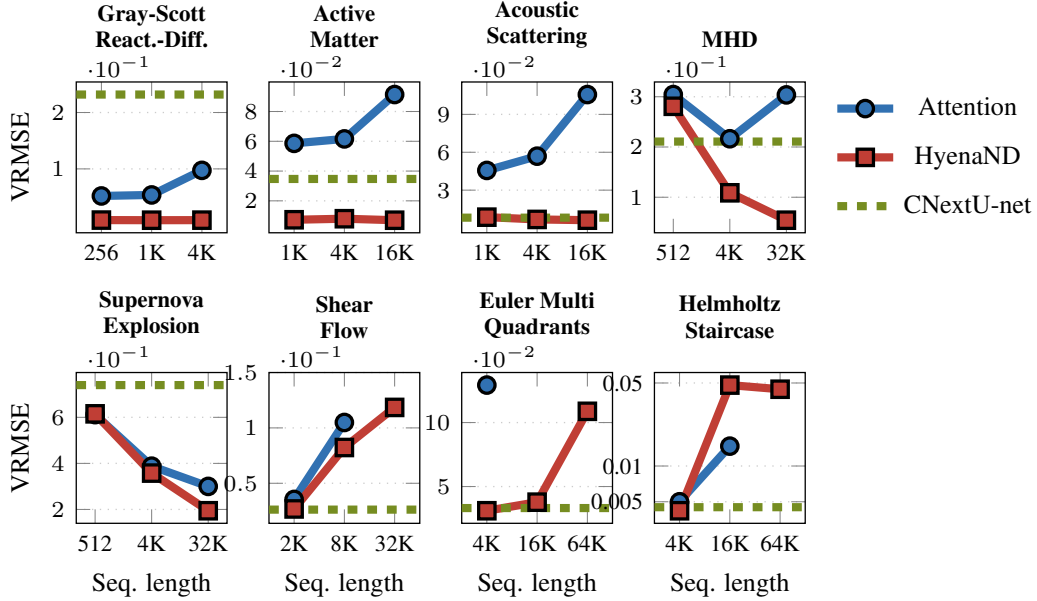

HyenaND achieves the lowest VRMSE on all eight datasets, as detailed in Table~\ref{tab:well} (Appendix~\ref{app:well-setup}). With increasing input length HyenaND improves over the best Attention configuration by $5.8\times$ on \pattern{}, $8.4\times$ on \activematter{}, $7.4\times$ on \acoustic{}, $4.0\times$ on \MHD{}, and $1.5\times$ on \supernova{}, and beats the full-resolution CNextU-net baseline on every dataset. For both \shearflow{}, \euler{} and \staircase{}, best performance was achieved with larger patch sizes but still outperforms attention and CNextU-net. We attribute this to the fixed frequency parameter, which was kept constant across all settings without individual tuning. Figure~\ref{fig:well_vrmse_scatter} shows VRMSE per dataset as a function of sequence length. The scaling superiority of HyenaND in throughput becomes evident with longer sequences. For example, at $L{>}16$k, HyenaND achieves more samples/sec than Attention by factors ranging from $2{\times}$--$10{\times}$. Against CNextU-net, HyenaND's lead depends on sequence length: CNextU-net is comparable on \acoustic{} and \MHD{} at $p{=}8$, with the gap opening up as patch-size shrinks.

\vspace{-2mm}
\subsection{$3\rm D$ Medical Image Segmentation}
\label{sec:exp-medical}
\vspace{-1mm}

We retrofit HyenaND into a SwinUNETR-style four-stage volumetric encoder, replacing per-stage windowed attention with $D{=}3$ HyenaND according to four attention/Hyena stage patterns $s \in \{A,H\}^{4}$. We evaluate on the PanTS dataset \citep{li2025pants}, a $28$-class abdominal segmentation benchmark using $901$ held-out volumes. Training details and dataset statistics are in Appendix~\ref{app:medical-setup}.

All three Hyena-containing variants consume ${\sim}\mathbf{58.7}$~GB of peak allocated memory per GPU during training, a $\mathbf{10.8\%}$ reduction relative to the $65.79$~GB attention baseline, invariant to which stages carry Hyena. The dominant memory cost of windowed attention at moderate ROI is the materialized attention matrix; an FFT-based global mixer eliminates that term without changing activation width. The decoder is identical across variants and accounts for ${\sim}590$ of the ${\sim}625$~GFLOPs, so total savings look modest ($2$--$3\%$); encoder-only savings are ${\sim}30\%$.

\begin{table}[!h]
    \centering
    \caption{PanTS $3\rm D$ segmentation (ROI$=96^{3}$, $100$ epochs). Mean Dice across $28$ classes and $901$ held-out test volumes; peak per-GPU allocated memory during training (batch$=32$, $4{\times}$GB200); forward-pass GFLOPs at batch $1$, including an analytical FFT-convolution correction for Hyena.}
    \label{tab:monai_pants}
    \vspace{0.4em}    
    \footnotesize
    \setlength{\tabcolsep}{1.0em}
    \begin{tabular}{@{}l c c c c@{}}
        \toprule
        \textbf{Variant}
            & \textbf{Mean Dice $\uparrow$}
            & \textbf{$\pm$ SEM}
            & \textbf{Peak alloc.\ mem.\ (GB) $\downarrow$}
            & \textbf{Fwd.\ GFLOPs $\downarrow$} \\
        \midrule
        $\mathbf{AAAA}$ (SwinUNETR baseline)   & 0.7496          & 0.0053 & 65.79         & 638 \\
        $\mathbf{HHHH}$ (all-Hyena)            & 0.7510          & 0.0110 & \textbf{58.69}         & \textbf{621} \\
        $\mathbf{HAHA}$ (striped hybrid)       & 0.7535          & 0.0054 & 58.70         & 627 \\
        $\mathbf{HHAA}$ (hierarchical hybrid)  & \textbf{0.7559} & \textbf{0.0052} & 58.67         & 623 \\
        \bottomrule
    \end{tabular}
\end{table}

The all-Hyena variant matches the attention baseline ($0.7510$ vs.\ $0.7496$), sufficient to establish Hyena as a drop-in replacement, and both hybrids strictly outperform both pure variants: $\mathbf{HAHA}$ reaches $0.7535$ and $\mathbf{HHAA}$ reaches $0.7559$: $+0.63$~pp and $+0.49$~pp over the pure attention and Hyena baselines. Paired Wilcoxon signed-rank tests on per-subject Dice differences confirm that all three Hyena-containing variants significantly outperform the attention baseline: $\mathbf{HHHH}$ vs.\ $\mathbf{AAAA}$: $p{=}2.4{\times}10^{-6}$; $\mathbf{HAHA}$ vs.\ $\mathbf{AAAA}$: $p{=}9.1{\times}10^{-16}$; $\mathbf{HHAA}$ vs.\ $\mathbf{AAAA}$: $p{=}5.7{\times}10^{-24}$.


\vspace{-1mm}
\section{Conclusion}
\label{sec:conclusion}
\vspace{-1mm}

The prevailing view of the subquadratic landscape conflates an operator's class with its computational form. We argued this equivalence is a coincidence of one class, the LTI SSM: recognizing SSMs as a special case yields two complementary branches. The recurrent branch (Mamba and descendants) buys per-token selectivity at the price of linear time-invariance (LTI), and remains causal and one-dimensional: bidirectional use requires running the recurrence twice, multi-dimensional use forces the input to be rasterized into an ad-hoc $1\rm D$ scan order that breaks its spatial structure, and all information must be routed through a fixed-size hidden state that bottlenecks long-range dependencies -- a bottleneck that worsens as those dependencies span a full $N\rm D$ grid. The convolutional branch instead retains LTI, generalizes naturally to $N\rm D$, and carries no hidden state; hence, propagating information across any distance is simply kernel mass at the corresponding offset. On structure alone, these properties seem to make it the natural subquadratic primitive for approximately-LTI $\rm ND$ data. However, two objections have historically kept it from that role, which we remove. The first is conceptual: long convolutions have been regarded as input-independent operators, an objection inherited from autoregressive $1\rm D$ processing, where per-token kernel variation is the only form of input-dependence \citep{poli2023hyena}. In non-autoregressive $\rm ND$ settings, however, the input is available at once, and the kernel can be conditioned on the input as a whole. HyenaND instantiates this: register-based FiLM conditioning makes the convolutional kernel input-dependent while preserving the single-FFT path and native geometry, recovering selectivity without sacrificing either. The second is practical: the $\mathcal{O}(L \log L)$ cost of FFT convolutions is repeatedly defeated by HBM round-trips that dwarf the arithmetic. \texttt{nSubQ}, our CUDA library, fuses the FFT-convolution path under a unified $1\rm D$/$2\rm D$/$3\rm D$ API, turning the asymptotic advantage into wall-clock speedups.

With these gaps closed, pure HyenaND matches strong attention and specialized baselines on genomics, ImageNet, $3\rm D$ medical segmentation, and PDE modeling on \thewell{}, where it attains the lowest VRMSE on every dataset and widens its lead over attention as sequence length grows. Hybrids interleaving HyenaND with attention layers outperform both pure attention and recurrence-based hybrids, confirming the convolutional branch as the natural subquadratic primitive for $\rm ND$ data.


Today's foundation models, by contrast, rest on a primitive whose quadratic cost forces patchification. This reframes a common assumption and opens an opportunity. The assumption that patchification is a necessary preprocessing step does not hold. It is a workaround for attention's quadratic cost, one that discards fine-grained structure before learning begins, and a native $\rm ND$, input-dependent, subquadratic operator removes the need for it, letting models ingest data in its intrinsic geometry at native resolution. HyenaND's scaling makes this concrete: in our $3\rm D$ experiments it mixes directly over ${\sim}110\text{k}$-token volumes, where windowed attention is the usual workaround. The opportunity is that HyenaND could also replace patchification: a few layers could act as a learned tokenizer, aggregating raw pixels or voxels into information-dense tokens for any downstream backbone --a dynamic, global alternative to the static, local patch embeddings used today. 




\bibliography{neurips_2026}
\bibliographystyle{plainnat}

\newpage

\appendix

\section{Appendix: Theory}
\label{app:theory}

\subsection{Gated Hyena Computes Separable Attention}
\label{app:separable-attn}

We show that a gated HyenaND layer computes exactly the class of \emph{separable attention matrices}---attention patterns that factor as a Toeplitz matrix sandwiched between two diagonal matrices.
This gives a precise characterization of the inductive bias of gated Hyena in terms of structured linear algebra, connecting it to the literature on low-displacement-rank operators.

We work with a single-channel, $1\mathrm{D}$ formulation; the multi-channel and multi-dimensional cases follow by applying the result independently per channel.

\subsubsection{Formal Setup}

\begin{definition}[SIREN kernel]\label{def:siren-kernel}
An $L$-layer SIREN kernel with hidden width $m$, output dimension $N$, first-layer frequency $\omega_0 > 0$, and hidden frequency $\omega_h > 0$ is a function $K_\theta : \mathbb{R} \to \mathbb{R}^d$ computed as
\begin{align*}
  h_0(\tau) &= \sin(\omega_0 W_0 \tau + b_0), \\
  h_\ell(\tau) &= \sin(\omega_h W_\ell\, h_{\ell-1}(\tau)), \quad \ell = 1, \ldots, L-2, \\
  K_\theta(\tau) &= W_{\mathrm{out}}\, h_{L-2}(\tau),
\end{align*}
with parameters $W_0 \in \mathbb{R}^{m \times 1}$, $b_0 \in \mathbb{R}^m$, $W_\ell \in \mathbb{R}^{m \times m}$, and $W_{\mathrm{out}} \in \mathbb{R}^{d \times m}$.
The input $\tau$ is a normalized displacement $(i-j)/n \in [-1,1]$.
\end{definition}

\begin{definition}[Hyena mixer]\label{def:hyena-mixer}
Given an input sequence $x_1, \ldots, x_L \in \mathbb{R}^d$ and a SIREN kernel $K_\theta$, the Hyena mixer with shortcut vector $s \in \mathbb{R}^d$ produces outputs
\[
  y_i^{(L)} = \sum_{j=1}^{L} K_\theta\!\!\left(\frac{i-j}{L}\right) x_j + s \odot x_i, \qquad i = 1, \ldots, L.
\]
Normalizing by $L$ maps all displacements into $[-1,1]$ regardless of sequence length.
\end{definition}

\begin{definition}[Gated Hyena mixer]\label{def:gated-hyena}
A gated Hyena mixer extends Definition~\ref{def:hyena-mixer} with projection matrices $W_Q, W_K, W_V \in \mathbb{R}^{d \times d}$ and a pointwise nonlinearity $\sigma$.
The convolution input is gated element-wise,
\[
  z_j = (W_Q x_j) \odot \sigma(W_K x_j),
\]
and an output gate is applied after the convolution:
\[
  y_i = \left(\sum_{j=1}^{L} K_\theta\!\!\left(\frac{i-j}{L}\right) z_j\right) \odot \sigma(W_V x_i).
\]
\end{definition}

\subsubsection{Separable Attention Matrices}

\begin{definition}[Separable attention matrix]\label{def:sep-attn}
An $L \times L$ matrix $A \in \mathbb{R}^{L \times L}$ is \emph{separable} if
\[
  A = D_l\, T\, D_R,
\]
where $T \in \mathbb{R}^{L \times L}$ is Toeplitz ($T_{ij} = t_{i-j}$ for some sequence $t$) and $D_l = \mathrm{diag}(r_1, \ldots, r_n)$, $D_R = \mathrm{diag}(c_1, \ldots, c_L)$ are diagonal.
Equivalently, $A_{ij} = T_{i-j} \cdot r_i \cdot c_j$.
\end{definition}

\subsubsection{Main Result}

\begin{tcolorbox}[
    colback=blue!5,
    colframe=blue!5,
    boxrule=0pt,
    arc=2mm,
    left=2mm,
    right=2mm,
    top=2mm,
    bottom=2mm
]
\begin{theorem}[Gated Hyena computes separable attention]\label{thm:sep-attn}
Fix a channel $h \in [d]$ and let $\sigma : \mathbb{R} \to \mathbb{R}$ be any pointwise nonlinearity.
The effective attention weight that a gated Hyena mixer places from token $j$ to position $i$ at channel $h$ is
\[
  A_{ij}^{(h)}
  = \underbrace{K_\theta^{(h)}\!\!\left(\frac{i-j}{n}\right)}_{T_{i-j}}
  \cdot \underbrace{\sigma(W_V x_i)^{(h)}}_{r_i}
  \cdot \underbrace{(W_Q x_j)^{(h)} \cdot \sigma(W_K x_j)^{(h)}}_{c_j}.
\]
That is, $A^{(h)}$ is a separable attention matrix (Definition~\ref{def:sep-attn}) with Toeplitz component $T_{i-j} = K_\theta^{(h)}((i-j)/L)$, row weights $r_i = \sigma(W_V x_i)^{(h)}$, and column weights $c_j = (W_Q x_j)^{(h)} \cdot \sigma(W_K x_j)^{(h)}$.
Separability holds for any pointwise $\sigma$; no linearity assumption is required.
\end{theorem}
\end{tcolorbox}

\begin{proof}
Expanding the gated Hyena mixer output (Definition~\ref{def:gated-hyena}) at channel $h$, the Hadamard products reduce to scalar multiplication:
\begin{align*}
  y_i^{(h)}
  &= \sigma(W_V x_i)^{(h)} \cdot
     \sum_{j=1}^n K_\theta^{(h)}\!\!\left(\frac{i-j}{n}\right)
     (W_Q x_j)^{(h)} \cdot \sigma(W_K x_j)^{(h)}.
\end{align*}
Set $r_i = \sigma(W_V x_i)^{(h)}$, $c_j = (W_Q x_j)^{(h)} \cdot \sigma(W_K x_j)^{(h)}$, and $T_{i-j} = K_\theta^{(h)}((i-j)/L)$.
Then $y_i^{(h)} = r_i \sum_j T_{i-j}\, c_j$, so $A_{ij}^{(h)} = T_{i-j} \cdot r_i \cdot c_j$.
Since $\sigma$ is applied point-wise, $r_i$ depends only on $x_i$ and $c_j$ depends only on $x_j$, so the dependencies are multiplicatively separable.
The matrix $T$ is Toeplitz because $T_{i-j}$ depends only on the displacement $i - j$.
Hence $A^{(h)} = D_l\, T\, D_R$ is separable by Definition~\ref{def:sep-attn}.
\end{proof}

\subsection{Length Generalization of HyenaND}
\label{app:length-gen}

The Hyena mixer normalizes the displacement between positions $i$ and $j$ by the sequence length $L$, mapping all offsets to $[-1,1]$ regardless of $L$ (Definition~\ref{def:hyena-mixer}).
This makes the same SIREN network meaningful across sequence lengths.
The theorem below quantifies how much the output changes when a model is evaluated at a length $L$ different from its training length $L_0$.

\subsubsection{Lipschitz Constant of a SIREN}

\begin{tcolorbox}[
    colback=blue!5,
    colframe=blue!5,
    boxrule=0pt,
    arc=2mm,
    left=2mm,
    right=2mm,
    top=2mm,
    bottom=2mm
]
\begin{lemma}[SIREN Lipschitz bound]\label{lem:siren-lip}
For any SIREN kernel $K_\theta$ (Definition~\ref{def:siren-kernel}) and any $\tau, \tau' \in \mathbb{R}$,
\[
  \norm{K_\theta(\tau) - K_\theta(\tau')}_2 \leq \mathcal{L}_K \cdot \abs{\tau - \tau'},
\]
where $\mathcal{L}_K = \omega_0\, \omega_h^{L-2}\, \norm{W_{\mathrm{out}}}_2 \prod_{\ell=0}^{L-2} \norm{W_\ell}_2$.
\end{lemma}
\end{tcolorbox}

\begin{proof}
We bound the operator norm of the Jacobian $dK_\theta/d\tau$ and apply the mean value theorem.
By the chain rule:
\[
  \frac{dK_\theta}{d\tau}
  = W_{\mathrm{out}}
    \left[\prod_{\ell=L-2}^{1}
      \omega_h\,\mathrm{diag}\!\bigl(\cos(\omega_h W_\ell h_{\ell-1}(\tau))\bigr)\, W_\ell
    \right]
    \omega_0\,\mathrm{diag}\!\bigl(\cos(\omega_0 W_0 \tau + b_0)\bigr)\, W_0.
\]
Since $|\cos(\cdot)| \leq 1$ point-wise, each diagonal factor has operator norm at most $1$.
Sub-multiplicativity of the operator norm then gives
\[
  \norm{\tfrac{dK_\theta}{d\tau}}_2
  \leq \norm{W_{\mathrm{out}}}_2
    \prod_{\ell=1}^{L-2}\!\bigl(\omega_h \norm{W_\ell}_2\bigr)
    \cdot \omega_0 \norm{W_0}_2
  = \mathcal{L}_K,
\]
and the mean value theorem yields $\norm{K_\theta(\tau) - K_\theta(\tau')}_2 \leq \LK \abs{\tau - \tau'}$.
\end{proof}

\subsubsection{Main Bound}

\begin{tcolorbox}[
    colback=blue!5,
    colframe=blue!5,
    boxrule=0pt,
    arc=2mm,
    left=2mm,
    right=2mm,
    top=2mm,
    bottom=2mm
]
\begin{theorem}[Length generalization]\label{thm:length-gen}
Let $x_1, \ldots, x_L$ be an input sequence with $\norm{x_j}_2 \leq B$ for all $j$, and let $H$ be a Hyena mixer (Definition~\ref{def:hyena-mixer}) with SIREN Lipschitz constant $\mathcal{L}_K$.
Evaluating $H$ with length normalizations $L$ and $L_0$ respectively, the outputs at the last position satisfy
\begin{equation}\label{eq:length-bound}
  \norm{y_L^{(nL} - y_n^{(L_0)}}_2
  \;\leq\;
  \frac{\mathcal{L}_K\, B\, (L-1)\, |L - L_0|}{2\, L_0}.
\end{equation}
In particular, for $|\Delta L| = |L - L_0| \ll L_0$:
\[
  \norm{y_L^{(L)} - y_L^{(L_0)}}_2
  \;\leq\; \frac{\mathcal{L}_K\, B\, L\, |\Delta L|}{2\, L_0}
       + O\!\left(\frac{\Delta L^2}{L_0^2}\right).
\]
\end{theorem}
\end{tcolorbox}

\begin{proof}
Since both evaluations share the same input sequence $x_1, \ldots, x_L$, the shortcut terms $s \odot x_L$ are identical and cancel in the difference:
\begin{align*}
  y_n^{(n)} - y_n^{(n_0)}
  &= \left(\sum_{j=1}^{L} K_\theta\!\!\left(\frac{L-j}{n}\right) x_j + s \odot x_n\right)
   - \left(\sum_{j=1}^{L} K_\theta\!\!\left(\frac{L-j}{n_0}\right) x_j + s \odot x_n\right) \\
  &= \sum_{j=1}^{L}
    \left[K_\theta\!\!\left(\frac{L-j}{L}\right)
          - K_\theta\!\!\left(\frac{L-j}{L_0}\right)\right] x_j.
\end{align*}
The triangle inequality and Lemma~\ref{lem:siren-lip} give
\begin{align*}
  \norm{y_n^{(L)} - y_L^{(L_0)}}_2
  &\leq \mathcal{L}_K\, B \sum_{j=1}^{L} (L-j)\, \left|\frac{1}{L} - \frac{1}{L_0}\right|
   = \mathcal{L}_K\, B \left|\frac{1}{L} - \frac{1}{L_0}\right| \cdot \frac{L(L-1)}{2},
\end{align*}
using $|(L-j)/L - (L-j)/L_0| = (L-j)|1/L - 1/L_0|$ and $\sum_{j=1}^n (L-j) = L(L-1)/2$.
Substituting $|1/L - 1/L_0| = |L - L_0|/(L\, L_0)$ yields~\eqref{eq:length-bound}.
\end{proof}

\begin{remark}
The bound grows linearly in $\LK$, which increases with the first-layer frequency $\omega_0$.
High-frequency SIREN kernels are therefore less stable under length shifts than smoother, lower-frequency ones.
This is consistent with the masking and initialization strategy of \S\ref{app:masking-init}, which controls the spectral behavior of the kernel at initialization.
\end{remark}

\begin{remark}[Practical scope of the bound]
When $L > L_0$, the displacement $(L-j)/L_0$ exceeds $1$ for tokens near $j = 1$, placing SIREN queries outside the $[-1,1]$ domain on which the kernel was trained.
The bound in \eqref{eq:length-bound} remains mathematically valid, Lemma~\ref{lem:siren-lip} holds for all $\tau \in \mathbb{R}$, but its practical interpretation relies on the kernel behaving predictably outside the training domain.
The periodicity of the sine activations provides some continuity beyond $[-1,1]$, yet does not eliminate the concern.
Theorem~\ref{thm:length-gen} is therefore most meaningful in the regime $|L - L_0| \ll L_0$, where the displacements remain close to the trained range.
\end{remark}

\subsection{HyenaND Design Details: Masking and Initialization}
\label{app:masking-init}

Naively applying the $1\rm D$ Hyena masking and SIREN initialization to $N{>}1$ data
introduces two failure modes: the effective receptive field collapses near the origin as $N$
grows, and early training is unstable from over-large activations in the kernel-MLP's
embedding layer.
We detail the three design choices that resolve both: a learnable Gaussian mask
for spatial data, a corrected SIREN first-layer initialization for $N$D
coordinates, and zero-start initialization for the FiLM conditioning heads.

\paragraph{Masking strategy.}
For spatial tasks ($N \geq 2$), the primary setting of this work, we use a learnable \textit{Gaussian mask} akin to \citet{romero2021flexconv}, but where channel $c$ learns its own per-axis standard deviation $\sigma_{d,c} {>} 0$:
\begin{equation}
\label{eq:mask}
M_c(\mathbf{r}) = \prod_{d=1}^N \exp\!\left(-\frac{r_d^2}{2\sigma_{d,c}^2}\right)
                = \exp\!\left(-\tfrac{1}{2}\textstyle\sum_d
                  \left(\tfrac{r_d}{\sigma_{d,c}}\right)^{\!2}\right).
\end{equation}
When all axes share the same $\sigma_c$ this is a radially symmetric Gaussian; anisotropic
grids (e.g.\ a volume with $N_z \ll H,W$) are handled through per-axis initialization of the $\sigma_{d,c}$. The bandwidths are initialized on a log-spaced ramp across channels, so at the start of training the operator spans from narrow, local filters to near-global ones; all $\sigma_{d,c}$ are learned during training. Grid boundaries are handled by zero-padding each spatial dimension by $N_d - 1$ before the FFT and cropping the output, following standard linear-convolution practice.

The Gaussian mask has two structural advantages over the exponential in the multi-dimensional setting. First, the decay rate is governed by Euclidean distance from the origin when $\sigma_{d,c}$ is isotropic, so the effective support scales with $\sigma_c$ independent of $N$. The naive axis-product exponential $M(\mathbf{r}) = \prod_{d} e^{-\alpha_d|r_d|}$, by contrast, has magnitude $e^{-\alpha N}$ at unit distance from the origin, collapsing the receptive field by a factor of $N$. Second, $\sigma_{d,c}$ is a directly interpretable bandwidth parameter per channel and axis, making the effective receptive field explicitly learnable.

For causal $1\rm D$ data (e.g.\ autoregressive genomic sequence modeling), we instead retain the original Hyena exponential mask $K(L) \leftarrow K(L) \cdot e^{-\alpha|L|}$, additionally enforcing causality by zeroing negative offsets: $M(r) \leftarrow M(r) \cdot \mathbf{1}[r \geq 0]$. The exponential form is the natural choice here: it provides a monotone decay matched to the intuition of recency bias in causal sequence modeling.

\paragraph{Kernel-MLP (SIREN) initialization.}
The SIREN initialization scheme~\cite{sitzmann2020siren} calibrates weight variances so that the pre-activations at every hidden layer have unit variance at initialization, which is essential for gradient propagation through sinusoidal activations.

\textit{Review of the $1\rm D$ case.} For a scalar coordinate $r \in [-1,1]$ and first-layer weight matrix $W_1 \in \mathbb{R}^{d_1 \times 1}$ with entries drawn i.i.d.\ from $\mathcal{U}(-a,a)$, the pre-activation is $w_i r$ for each row $w_i$. Since $\mathbb{E}[w_i^2] = a^2/3$ and $\mathbb{E}[r^2] = 1/3$ for $r \sim \mathcal{U}(-1,1)$,
\begin{equation}
\operatorname{Var}(w_i r) = \mathbb{E}[w_i^2]\,\mathbb{E}[r^2] = \frac{a^2}{9}.
\end{equation}
SIREN sets $a = 1$ (i.e., $\mathcal{U}(-1,1)$), giving $\operatorname{Var}(w_i r) = 1/9$, consistent with well-behaved activations for $1\rm D$ coordinate inputs.

\textit{Extension to $N$ dimensions.}
Our embedding layer maps a $N$-dimensional coordinate $\mathbf{r} \in [-1,1]^N$ to $\mathbb{R}^{d_\mathrm{emb}}$ via
\begin{equation}
\phi = \sin(\omega_0 W_0 \mathbf{r}), \qquad W_0 \in \mathbb{R}^{d_\mathrm{emb} \times N}.
\end{equation}
For row $w_0^\top \in \mathbb{R}^N$ of $W_0$ with entries drawn i.i.d.\ from $\mathcal{U}(-a,a)$ and coordinate components $r_d \sim \mathcal{U}(-1,1)$ independently,
\begin{equation}
\operatorname{Var}(w_0^\top \mathbf{r})
= \sum_{d=1}^N \operatorname{Var}(w_{0,d}\,r_d) = N \cdot \frac{a^2}{9}.
\end{equation}
The variance grows linearly with $N$: the standard $1\rm D$ choice $a = 1$ yields $\operatorname{Var} = N/9$, causing the $\sin$ activations to saturate immediately at initialization and producing unstable early gradients. To recover $\operatorname{Var}(w_0^\top \mathbf{r}) = 1/9$ (matching the $1\rm D$ case), we require $N \cdot a^2/9 = 1/9$, giving $a = 1/\sqrt{N}$. We therefore initialize
\begin{equation}
W_0 \sim \mathcal{U}\!\left(-\frac{1}{\sqrt{N}},\, \frac{1}{\sqrt{N}}\right).
\end{equation}
The subsequent hidden layers $W_1,...,W_L$ receive a $d_\mathrm{emb}$-dimensional input and can follow the standard SIREN prescription unchanged, since the embedding already normalizes the variance to the expected regime.

\paragraph{FiLM parameter initialization.}
At initialization, the modulated kernel should equal the base kernel so that training begins from the same well-characterized regime as the unmodulated model. We achieve this by initializing the final linear layers of $\gamma_\theta$ and $\beta_\theta$
to zero with a constant identity residual, giving $\gamma_\theta(c)\big|_{t=0} \equiv \mathbf{1}$ and $\beta_\theta(c)\big|_{t=0} \equiv \mathbf{0}$ for all $c$.
This ensures $K(\mathbf{r};\,\mathbf{x})\big|_{t=0} = K_0(\mathbf{r})$ for every input, so the FiLM heads start from the same fixed point as the unmodulated baseline rather than a random perturbation.

\section{CUDA Implementation and Benchmarking}
\label{app:cuda}
Although our developmental framework is broadly applicable, we focus on Evo2 and 2D ImageNet classification to demonstrate our contributions to the Hyena architecture and other machine learning models relying on similar primitive operations. All our custom kernels are implemented using CuTe and cuFFTDx \cite{nvidia_cutlass,nvidia_cufftdx}, and the benchmarking is performed on an NVIDIA RTX PRO 6000 Blackwell Server Edition. Prior work, such as ``Systems and Algorithms for Convolutional Multi-Hybrid Language Models at Scale,'' introduced similar operators based on hardware-algorithm co-design principles to maximize hardware utilization \cite{ku2025systemsalgorithmsconvolutionalmultihybrid}. Building upon this foundation, we introduce novel kernels for the 1D operators utilized in Evo2, and extend them to 2D cases, closing the gap between the theoretical performance of Hyena Layers and their real runtime.

\subsection{Hardware Context and IO Profile}
\label{app:cuda-io}

To fully contextualize our methodology, it is essential to consider the strict memory hierarchy of modern GPUs. This hierarchy spans from massive but relatively slow global High Bandwidth Memory (HBM), up through progressively smaller and faster layers---L2 cache, L1 cache/Shared Memory (SRAM)---and finally, ultra-fast thread-level registers. Standard PyTorch implementations are frequently bottlenecked by this hierarchy, forced to materialize intermediate tensors and perform costly HBM read/write round-trips between sequential operations. Our approach is therefore primarily oriented toward minimizing HBM accesses through kernel fusion, keeping intermediate computations localized within SRAM and registers.

Two hardware trends shape the practical performance comparison between attention and FFT-based operators.
First, GPU matrix-multiply throughput has generally grown faster than off-chip memory bandwidth \cite{nvidia_v100_2017, nvidia_a100_2020, nvidia_h100_2022, nvidia_blackwell_2024}.
This strongly benefits standard attention, because the dominant operations in SDPA, $QK^\top$ and $PV$, are GEMM-like.
Modern fused SDPA kernels further improve efficiency by tiling the computation and avoiding materialization of the full $N \times N$ attention matrix in HBM.
FlashAttention-style kernels, therefore, make quadratic attention far more competitive in practice than its asymptotic complexity alone would suggest \cite{dao2022flashattention,zadouri2026flashattention4algorithmkernelpipelining}.

FFT-based Hyena operators face a different performance profile.
Their fast path depends less directly on tensor-core GEMM throughput and more on batched FFT efficiency, memory layout, fusion, and bandwidth.
Although there have been attempts to enable the usage of tensor-cores for FFT, they usually lag behind the routines provided by cuFFT and cuFFTDx, especially for real FFT \cite{fu2023flashfftconv,nvidia_cufft,nvidia_cufftdx}.
FFTs introduce nontrivial constant factors, often require complex-valued intermediate representations, and may involve transposes or layout conversions that increase HBM traffic.
As a result, an $\mathcal{O}(L \log L)$ sequence mixer may only outperform fused quadratic attention beyond a hardware- and implementation-dependent sequence length.

A second relevant hardware trend is the increasing support for explicit cooperation across thread blocks.
Hopper-class NVIDIA GPUs \cite{nvidia_h100_2022} introduced thread-block clusters and distributed shared memory, allowing CTAs in the same cluster to be co-scheduled, synchronize at cluster scope, and access one another's shared-memory segments.
Blackwell \cite{nvidia_blackwell_2024} separately introduces Cluster Launch Control, which provides more dynamic control over CTA work assignment through cluster-level cancellation and work-stealing mechanisms.
These features should not be conflated: distributed shared memory is a Hopper-era cluster cooperation mechanism, while Cluster Launch Control is a Blackwell-era scheduling mechanism.

For dense attention, these cluster-level mechanisms are less central than the highly optimized GEMM and fused-SDPA paths.
For FFT-based subquadratic sequence models, they may become more relevant because the dominant bottleneck is often not arithmetic, but HBM traffic induced by padding, complex-valued intermediate tensors, precision conversion, and kernel-boundary materialization.
In particular, FlashAttention-style SDPA keeps the $L \times L$ score matrix off HBM and exposes bf16 input/output interfaces to the surrounding model.
A naive FFT-convolution path has a less favorable IO profile.
Even when the model activations are stored in bf16, practical library FFT paths often require either fp32/complex-fp32 intermediates or constrained bf16 transform layouts; in the unrestricted layouts used by HyenaND, this precision widening alone can double the traffic per real-valued element relative to bf16 attention activations.

The effect is amplified by the real-to-complex FFT layout used for the long convolution.
A 1D sequence of logical length $l$ requires padding to a transform length $l \geq 2L-1$; in practice we use $l \approx 2L$.
The forward real-to-complex transform therefore materializes
\begin{equation}
    \widehat{x} = \mathcal{F}_{l}(x)
    \in \mathbb{C}^{l/2+1}
    \approx \mathbb{C}^{L+1},
\end{equation}
followed by complex multiplication with the spectral filter,
\begin{equation}
    \widehat{y}
    =
    \widehat{x} \odot \widehat{h}.
\end{equation}
The inverse complex-to-real transform then produces an $l \approx 2L$ real-valued buffer, from which the valid $L$ outputs must be cropped or chunked.
A naive implementation writes each of these objects to HBM: the padded real input, the activation spectrum, the filter spectrum, the inverse-FFT buffer, and the final segmented output.
The nominal $\mathcal{O}(L \log L)$ arithmetic advantage can therefore be hidden by a large constant factor in memory traffic and by multiple small kernel launches around the FFT calls.

The multi-dimensional case worsens this IO pressure.
For a 2D input of size $H \times W$, linear convolution pads each axis to $L_H \approx 2H$ and $L_W \approx 2W$, and the real-to-complex spectrum has shape
\begin{equation}
    \widehat{x}
    \in
    \mathbb{C}^{L_H \times (L_W/2+1)}
    \approx
    \mathbb{C}^{2H \times (W+1)}.
\end{equation}
The padded real-space footprint grows by roughly $4\times$, and the complex spectrum carries another factor from storing real and imaginary components.
In $D$ dimensions, padding alone contributes an approximate $2^D$ expansion before accounting for complex storage, temporary workspaces, or layout conversions.
This is the practical source of the HBM explosion in naive FFT-based Hyena implementations, especially for 2D images and 3D volumes.
In this setting, Hopper and Blackwell cluster-cooperation features are useful not because they accelerate dense GEMM, but because they provide additional mechanisms for staging non-GEMM work and coordinating CTAs that cooperatively perform FFT.

\subsection{Filter Generation}
\label{app:cuda-filter-gen}

Through Nsight Systems (\texttt{nsys}) profiling, we observed that $30\%$ of the computational cost in the LI layer stems from filter generation. To remedy this, we developed a dedicated kernel specifically for this operation.\footnote{Source code and documentation omitted for double-blind review.} As illustrated in Figure \ref{fig:implicit_filter}, this dedicated kernel accelerates implicit filter runtime by over $40\times$ at the module level compared to a vanilla PyTorch implementation. Concurrently, it reduces memory overhead by a factor proportional to the order (16) of the implicit modal filter ($32\times$), thereby freeing up crucial memory to accommodate larger batch sizes and scale to significantly longer context lengths. Additionally, compared to the first release of Evo2, we shard the parameters of implicit modal filter instead of generated filter, saving $CP\times$ memory for context extension. Through both optimizations, we achieve $16M$ context length for $1B$ Evo2 Model.

\begin{figure}[htbp]
    \centering
    \includegraphics[width=0.9\linewidth]{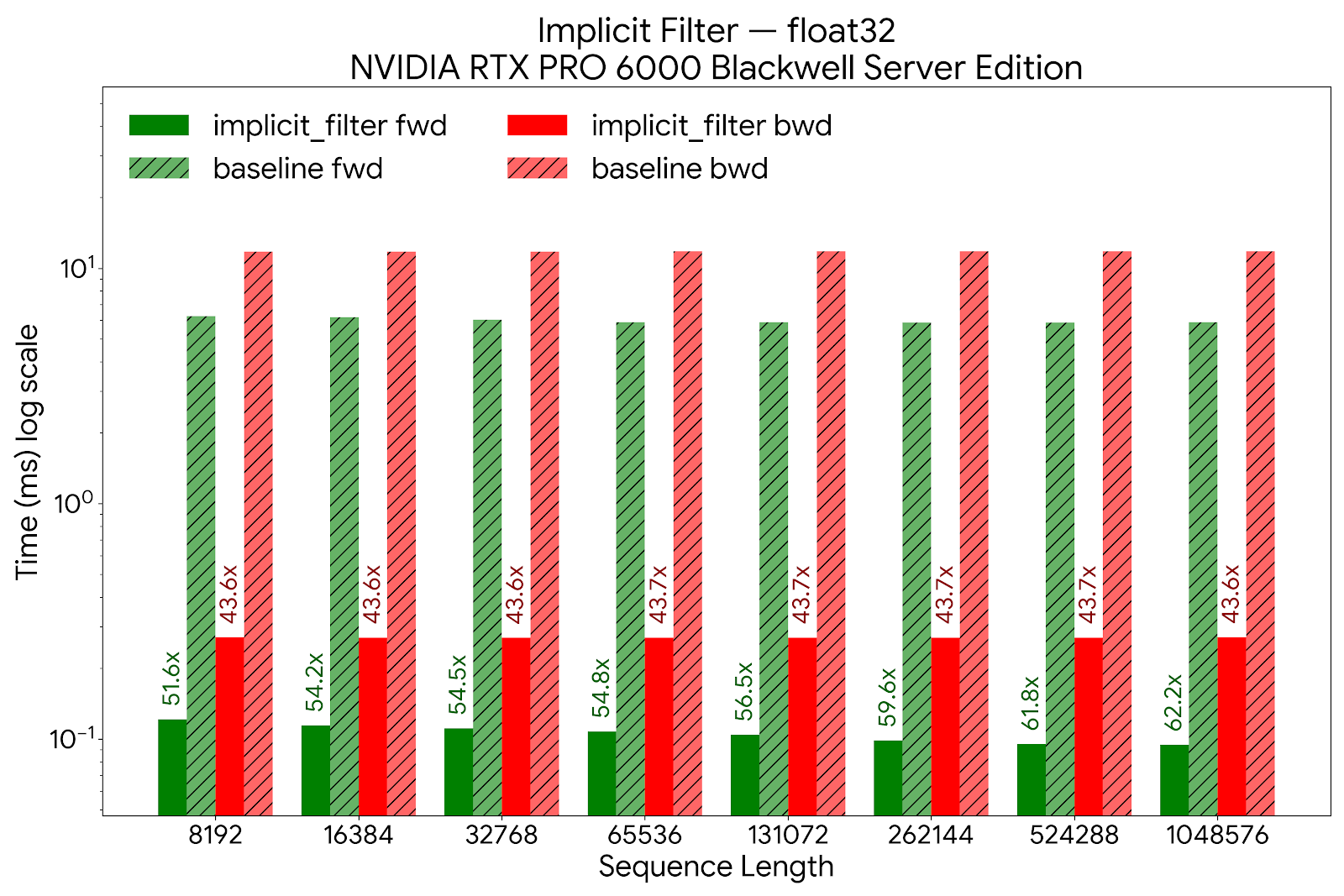}
    \caption{Performance comparison of the implicit filter generation. Profiling results demonstrate that our dedicated, fused kernel accelerates filter generation runtime by over $40\times$ compared to the baseline PyTorch implementation. Furthermore, the kernel reduces High Bandwidth Memory (HBM) overhead by a factor of $32\times$, a reduction proportional to the order of the modal filter.}
    \label{fig:implicit_filter}
\end{figure}

\subsection{1D FFT Convolution}
\label{app:cuda-fftconv1d}
For the LI layer of Evo2, and long implicit layers in general, we developed a highly optimized, fused causal FFTConv1D kernel. For sequence lengths up to 16K, this single kernel fuses the Real Fast Fourier Transform (RFFT) of the input and filter, element-wise complex multiplication, the Inverse RFFT (IRFFT), and chunking. This extensive fusion yields a $6\times$ runtime speedup and a $4\times$ reduction in memory consumption compared to standard implementations. For longer sequence lengths, we employ a three-kernel approach based on Cooley-Tukey (CT) decomposition, effectively casting the 1D FFT into a 2D FFT. In this regime, we fuse the Complex-to-Complex (C2C) FFT, the Inverse C2C (IC2C) FFT, and the complex multiplication. The performance improvements and scaling behavior of this architecture are detailed in Figure \ref{fig:ct_fft_performance}. While the three-stage CT approach achieves $2-4\times$ speedup and memory reduction compared to native baselines, it never approaches the performance of fully-fused single kernel. 

Fortunately, a sequence length of $16\text{K}$ marks the critical point where Hyena's asymptotic scaling dominates quadratic SDPA. However, to extend the superior performance of the fully-fused single kernel to longer sequences, our future work will leverage distributed shared memory (DSM). Specifically, the DSM capabilities of the Hopper and Blackwell architectures provide vital staging mechanisms for this non-GEMM workload, enabling us to circumvent the massive memory traffic and multiple HBM round-trips inherent to naive FFT-based convolutions. This will fix the speedup drops for the sequence lengths of $[16,\space 64] \text{K}$. Note: the drop in speedup is hardware dependent.

\begin{figure}[htbp]
    \centering
    \includegraphics[width=0.9\linewidth]{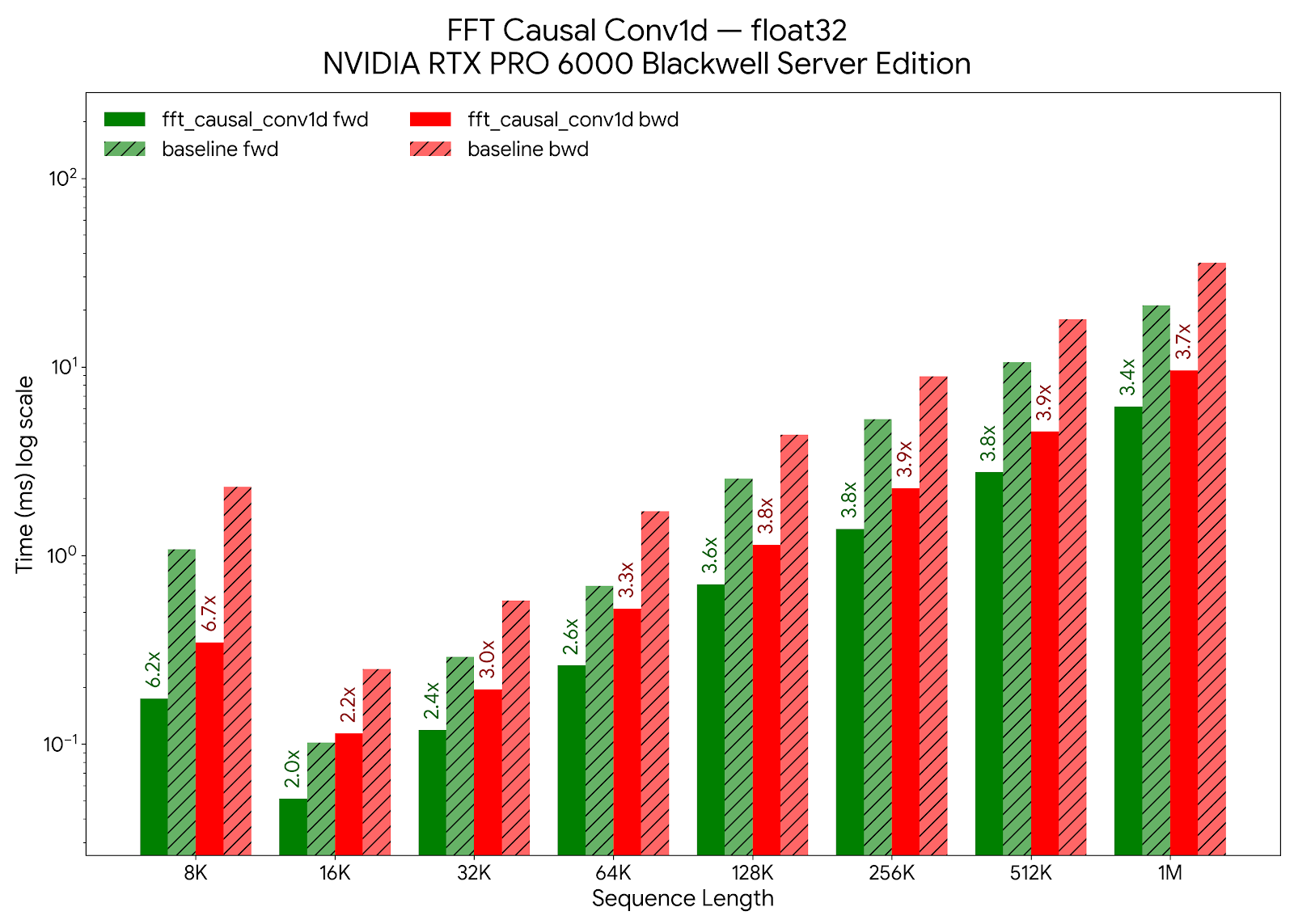}
    \caption{Performance evaluation of the three-kernel causal FFTConv1D approach based on Cooley-Tukey (CT) decomposition. The plot illustrates the scaling behavior and efficiency gains for sequence lengths exceeding $8\text{K}$. Except for $8\text{K}$ sequence length that uses fully-fused kernels, the rest of points use three different stages.}
    \label{fig:ct_fft_performance}
\end{figure}

We also developed \texttt{causal-conv1d} kernels that support filter sizes up to 256 or 128, depending on whether the input layout is channel-first or channel-last, respectively. As shown in Figure \ref{fig:causal_conv1d}, these kernels provide a substantially more efficient implementation than native PyTorch code and support a wider range of filter lengths compared to the Dao-AILab \texttt{causal-conv1d} implementation.

\begin{figure}[htbp]
    \centering
    \includegraphics[width=0.9\linewidth]{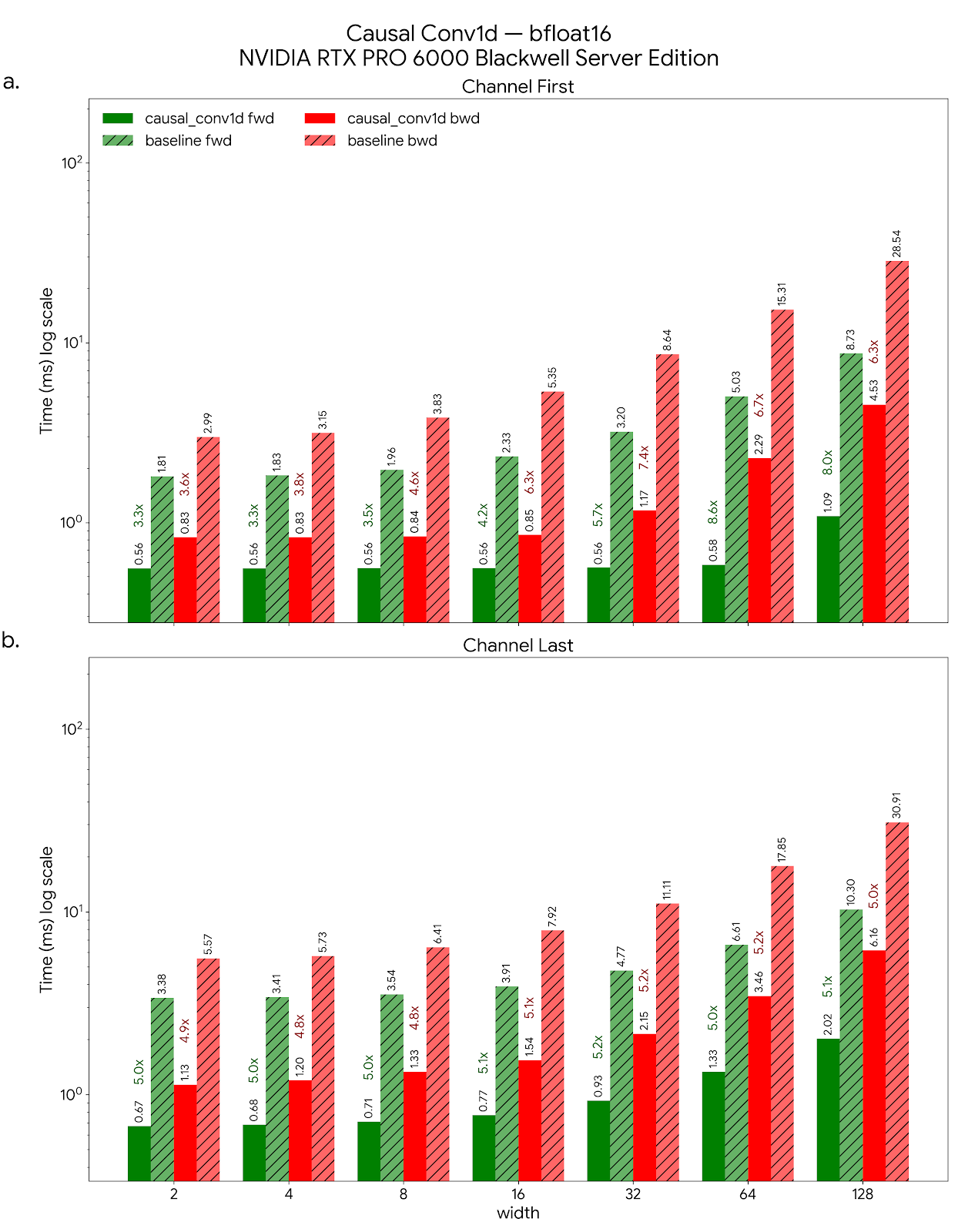}
    \caption{Performance comparison of our custom \texttt{causal-conv1d} kernels against native PyTorch implementations. Our kernels deliver substantial efficiency gains while extending support for filter sizes up to 256 (for channel-first layouts) and 128 (for channel-last layouts), surpassing the filter length constraints of existing optimized packages.}
    \label{fig:causal_conv1d}
\end{figure}

For the SE and MR layers, both of which utilize projection and mixer convolutions with filter lengths of $\le 128$, we developed a highly fused \texttt{b2b causal-conv1d} kernel. This custom kernel seamlessly integrates the projection convolution, pre-gating, mixer convolution, and post-gating steps into a single operation. As shown in Figure \ref{fig:b2b_conv}, this fusion improves runtime by more than $7.5\times$ and reduces the overall memory footprint and HBM traffic. Additionally, in context-parallel setups, this fused approach effectively reduces the number of required communication points by half, substantially mitigating communication overhead across devices.

\begin{figure}[htbp]
    \centering
    \includegraphics[width=0.9\linewidth]{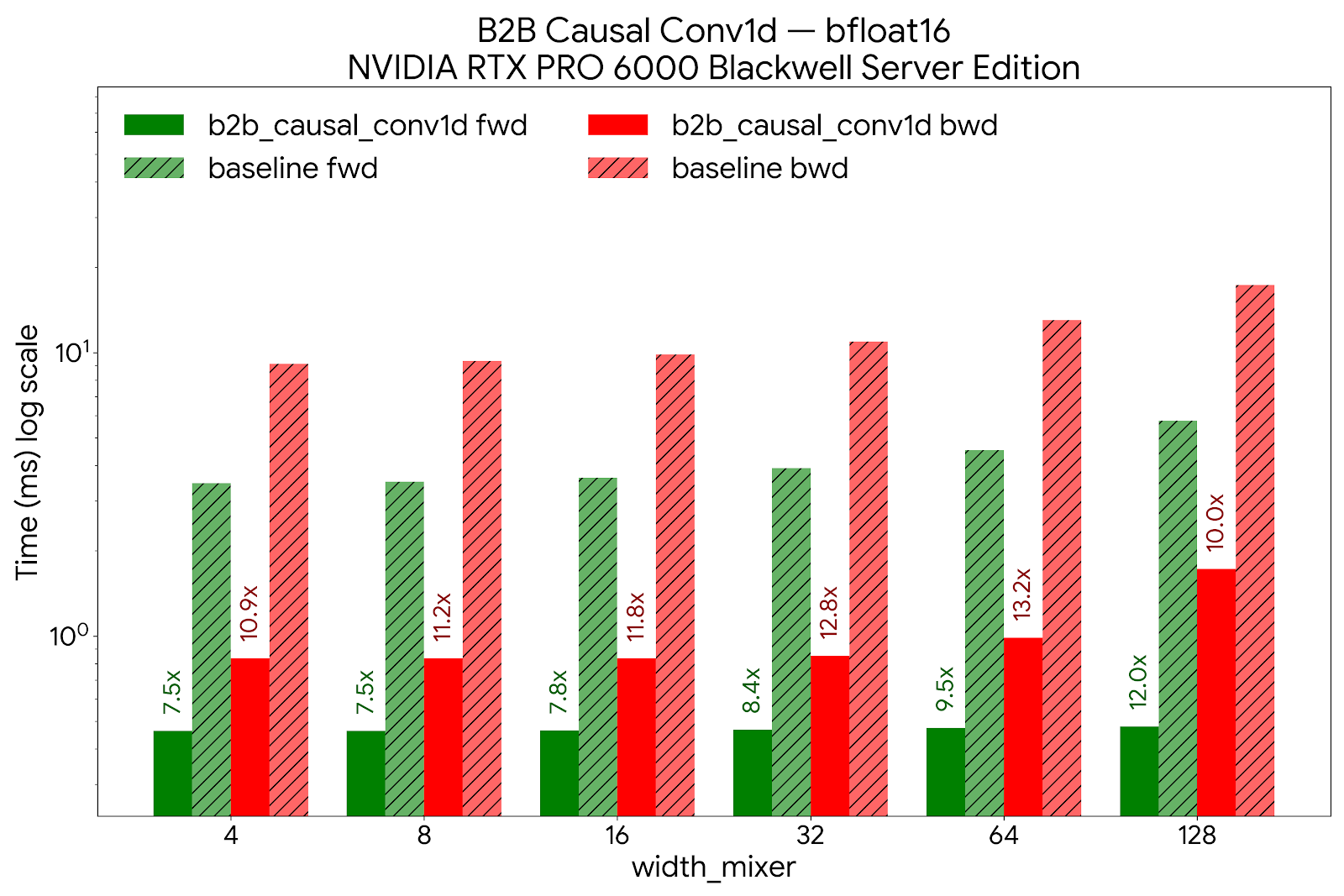}
    \caption{Performance evaluation of the fused \texttt{b2b causal-conv1d} kernel for the SE and MR layers. By integrating the projection convolution, pre-gating, mixer convolution, and post-gating into a single operation, the fused kernel achieves more than a $7.5\times$ runtime speedup and lowers the memory footprint. It also halves the necessary communication points in context-parallel configurations for the hyena operator.}
    \label{fig:b2b_conv}
\end{figure}

\subsection{2D FFT Convolution}
\label{app:cuda-fftconv2d}

Beyond our optimizations for 1D sequence modeling, we extend our hardware-algorithm co-design framework to 2D spatial data to accelerate vision-based architectures, specifically targeting ImageNet classification workloads. To this end, we developed highly optimized 2D FFT-based convolution kernels (\texttt{fft-conv2d}) to efficiently handle large-scale spatial filters.

In standard implementations, computing a 2D FFT convolution requires materializing multiple intermediate frequency-domain tensors, resulting in severe HBM bottlenecks. To circumvent this, we designed a pipeline that executes an initial Real-to-Complex FFT (RFFT) along the first dimension, followed by a fused core kernel for Complex-to-Complex (C2C) transformations, and concludes with an Inverse Real-to-Complex FFT (IRFFT). Our custom core kernel fuses the C2C FFT along the second dimension, the element-wise complex multiplication with the filter, and the Inverse C2C FFT. 

Our benchmarking focuses specifically on forward-pass execution. As demonstrated in Figure \ref{fig:2d_image_fft}, comparing this approach against an unfused baseline reveals that our memory-aware fusion delivers an over $5\times$ increase in runtime throughput and an over $2\times$ reduction in overall memory footprint. Similar to the 1D FFT case, our plan is to extend this workflow to use a single-stage kernel for the RFFT, C2C FFT, complex multiplication, and the Inverse C2C FFT and IRFFT.

\begin{figure}[htbp]
    \centering
    \includegraphics[width=0.9\linewidth]{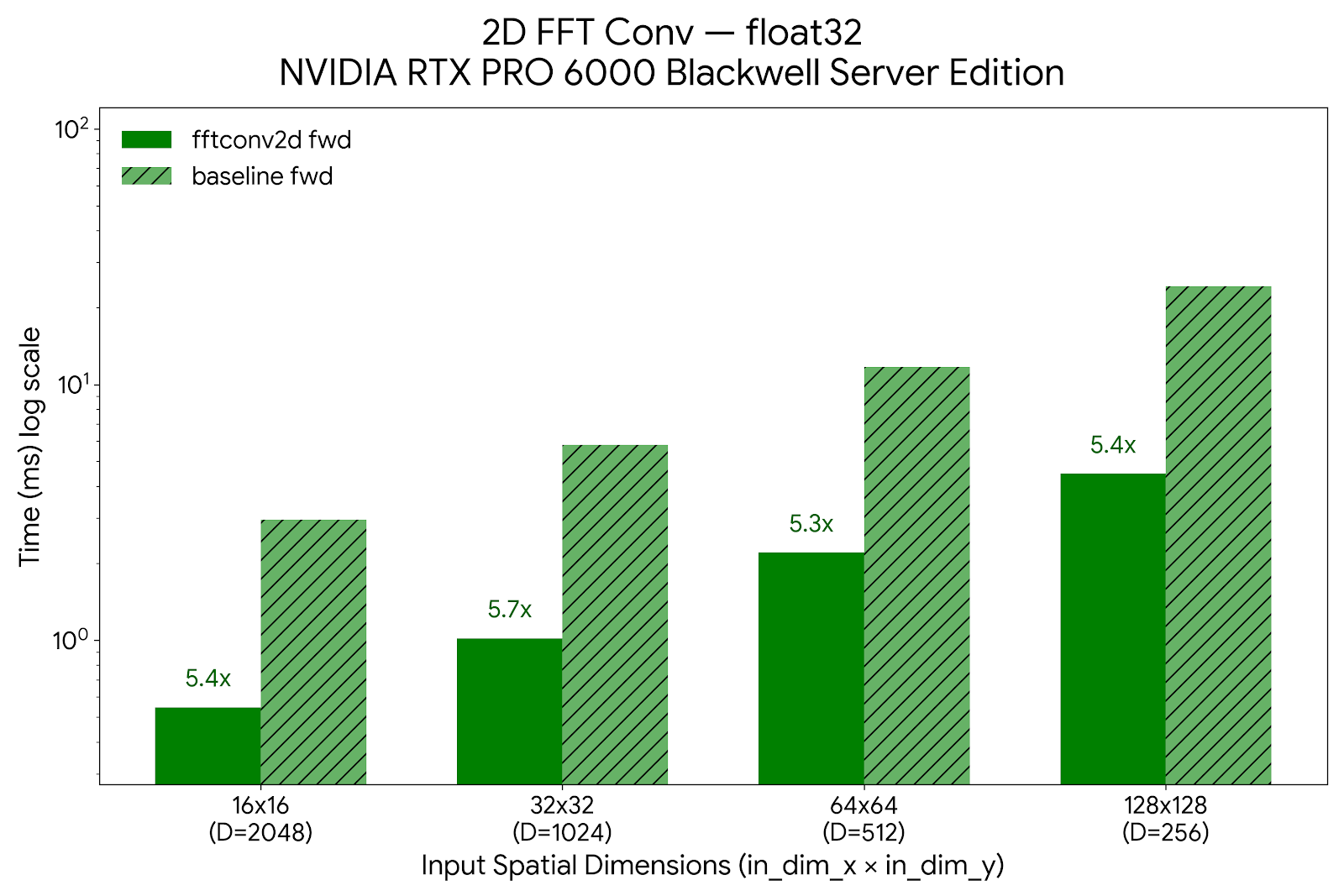}
    \caption{Forward-pass performance evaluation of our 2D FFT convolution kernels for image classification. We compare our approach against an unfused baseline. By successfully fusing the C2C FFT, complex multiplication, and IC2C FFT into a single core operation, intermediate computations are kept within fast Shared Memory and registers, achieving over $5\times$ runtime speedup and over $2\times$ reduction in memory footprint.}
    \label{fig:2d_image_fft}
\end{figure}

\FloatBarrier

\section{Appendix: Experiments}
\label{app:experiments}

\subsection{Controlled Spatial-Recall Experiments}
\label{app:spatial-recall}

We introduce a family of controlled \emph{spatial-recall} tasks that isolate the two primitive abilities any spatial token mixer needs: (i) \emph{transport} --- moving content unchanged across a long canvas (\texttt{simple\_copy}) --- and (ii) \emph{content-keyed retrieval} --- selecting the right item among distractors via an associative key and routing it to the readout (\texttt{color\_cond}). The tasks span 1D sequences, 2D images, and, in the hardest setting, 3D spatio-temporal volumes in which the target itself moves, so recall must additionally capture continuous latent factors (trajectory, spin) rather than a static pattern; a patch-size sweep in 1D/2D further probes how each operator trades token count against per-token detail. Together the probes read off \emph{operator inductive bias} under controlled conditions: HyenaND's global convolution encodes translation equivariance architecturally; attention must recover positional structure from its encodings (RoPE in our setup) and, as retrieval over per-token summaries, depends critically on the information each token carries; and scan-based SSMs must rasterize the canvas into a 1D order --- even bidirectional variants impose an arbitrary serialization of space. Parameter count and training budget are fixed across operators, and sensitivity to the optimization regime is measured explicitly (the learning-rate $\times$ grad-clip factorial below), so remaining gaps are attributable to the operator itself rather than to capacity or tuning.

\paragraph{Task family and dataset construction.}
All canvases are built from EMNIST digits on an empty background; the model reads the full canvas and regresses the target content at a \emph{designated readout region} in a fixed corner (MSE is computed on the readout region only). \texttt{simple\_copy} places a single item and isolates \emph{long-range positional recall}; \texttt{color\_cond} places four items with distinct color markers and isolates \emph{spatial binding with distractors}. Construction per dimensionality:
\begin{itemize}
    \item \textbf{1D ($L{=}4{,}096$).} A $16{\times}16$ digit is flattened to a $256$-token segment and placed at a uniformly random offset (the readout span excluded). In \texttt{color\_cond}, four such segments are placed without overlap, each flanked by $2$-token markers in a distinct color drawn without replacement from a fixed $8$-color palette; the readout --- the last $256$ tokens --- is flanked by markers in the \emph{target's} color. Labels are the target segment, recolored for \texttt{color\_cond}. For \texttt{simple\_copy} we report both a causal (autoregressive) and a non-causal variant.
    \item \textbf{2D ($64^2$).} $16{\times}16$ digits are placed at uniformly random, non-overlapping positions anywhere outside the bottom-right quadrant, which is reserved for the readout. In \texttt{simple\_copy}, a single digit is placed and must be reproduced at the fixed $16{\times}16$ readout square in the bottom-right corner. In \texttt{color\_cond}, four digits are placed, each wrapped in a one-pixel outline in a distinct color from the same $8$-color palette. The readout square is outlined in the color of one of them, and the model must identify the digit whose outline matches and reproduce it recolored to that color. Labels are the $16{\times}16$ digit -- grayscale for \texttt{simple\_copy}, recolored for \texttt{color\_cond}.
    \item \textbf{3D motion ($32^3$, depth $=$ time).} The recall target is a \emph{spatio-temporal tube}: a $4{\times}4$ digit is stamped on every depth slice of a $6^3$ block while translating along a random continuous path and turning in-plane in crisp $90^\circ$ steps. Movement paths are designed such that the motion range is bounded by the digit's ink rather than its stamp box and follow a monotone full-span sweep with a per-slice step of at most $2$ voxels to ensure the digit is present in every slice. The tube interior is filled with the digit's own background value, leaving no visible stamp-box boundary. In \texttt{simple\_copy}, a single tube is placed at a fixed position and must be reproduced at the $6^3$ readout corner. In \texttt{color\_cond}, four tubes are each wrapped in a one-voxel colored ring (in-plane on every slice, depth faces open; $8^3$ framed) and scattered at random non-overlapping positions. The $8^3$ readout corner carries a ring in the target's color, and the label is the target's framed block with its content recolored.
\end{itemize}
Dataset-mean predictors give the random performance floors reported in Table~\ref{tab:spatial_recall_summary}: $0.549$/$0.284$ (\texttt{simple\_copy}/\texttt{color\_cond}) in 1D/2D, and $0.482$/$0.257$ on the 3D motion volumes. Figure~\ref{fig:spatial_recall_tasks} shows one exemplar per task and dimensionality; Figure~\ref{fig:spatial_recall_motion_detail} shows full slice sequences for the motion tasks.

\begin{figure*}[!htb]
    \centering
    \includegraphics[width=0.95\linewidth]{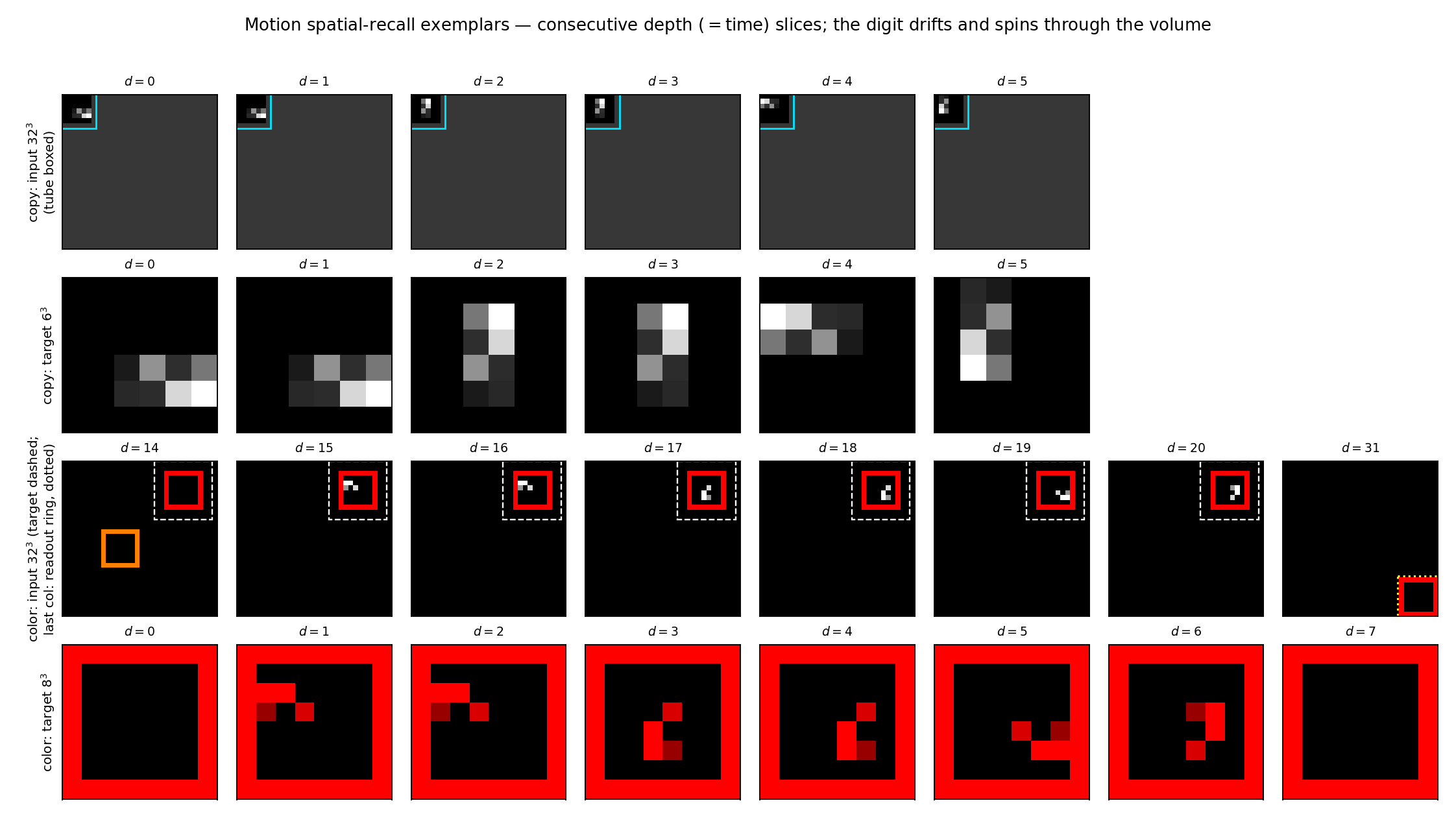}
    \caption{Motion dataset construction (consecutive depth ($=$time) slices from real training samples): \emph{Rows 1--2} (\texttt{simple\_copy}): input canvas with the $6^3$ tube (cyan box) and the corresponding target tube: the digit drifts along its sweep path (such that they never leave the tube) and turns in $90^\circ$ steps. \emph{Rows 3--4} (\texttt{color\_cond}): the target tube's ring (dashed white) matches the color of the readout ring (last column, $d{=}31$, dotted yellow); a distractor's ring (Three distractors in total) is visible at $d{=}14$ (orange). The label (row 4) is the target's framed block with its content recolored.}
    \vspace{-4mm}
    \label{fig:spatial_recall_motion_detail}
\end{figure*}

\paragraph{Models, capacity, and training.}
We compare three token mixers at parameter parity ($\sim 1.9$M, 4 residual blocks): (i) multi-dimensional Hyena (hidden dim $256$), (ii) multi-head self-attention with RoPE ($256$ in 1D/2D; $240$ in 3D, where head\_dim must be divisible by $6$ for axial RoPE), and (iii) Mamba-2 in both unidirectional (causal, dim $208$) and bidirectional (dim $160$) variants. All models are trained for $50$k iterations with AdamW, cosine schedule and $5\%$ warmup, \texttt{bf16}-mixed precision, and RMSNorm. Hyena and Attention run with \texttt{torch.compile} (\texttt{max-autotune-no-cudagraphs}); Mamba is run eagerly.
For 2D experiments, we additionally test a HyenaND variant with a learnable Gaussian mask on the convolutional kernel (replacing the default identity mask).

\paragraph{Patchification as a probe.}
Because the three operators respond very differently to sequence length, we additionally sweep a \texttt{Patchify}/\texttt{Unpatchify} in/out projection over patch sizes $p$, reducing the number of tokens by a factor of $p^N$ while packing the discarded spatial detail into the channel dimension. This exposes the \emph{token-count vs. spatial-detail} trade-off: large $p$ shortens the sequence (helping attention's cost and tightening the effective context) at the cost of per-token spatial resolution.

\paragraph{1D results.}
Table~\ref{tab:1d_recall} reports non-patched results on the $L{=}4{,}096$ canvas and the best patch size per model/task; the full patch-size sweep is in \ref{tab:1d_recall_full}. Three patterns emerge:

\emph{(i) Hyena dominates the long-context regime.} On the raw $L{=}4{,}096$ canvas Hyena is 3--4 orders of magnitude below attention on \texttt{simple\_copy} and attention never leaves the random floor on \texttt{color\_cond} ($0.274$ vs.\ baseline $0.284$). \emph{(ii) Attention requires heavy patchification to become competitive} on \texttt{simple\_copy} (a $1250\times$ improvement from $p{=}1 \to p{=}128$) and remains essentially stuck at the random floor on \texttt{color\_cond} at every patch size. \emph{(iii) Causal Mamba cannot solve} \texttt{simple\_copy} (test loss $0.739 >$ random $0.549$), but \emph{bidirectional Mamba is strong}: it matches Hyena on non-patched \texttt{color\_cond} ($0.0065$ vs.\ $0.0062$) and, once patched, matches or beats Hyena on \texttt{simple\_copy} and improves \texttt{color\_cond} to $2.4\mathrm{e}{-4}$. 

\begin{table}[!htb]
\centering
\small
\caption{\textbf{1D spatial recall.} MSE (lower is better) on a $L=4{,}096$ canvas, $50$k iterations, $\sim\!1.9$M parameters. ``Best $p$'' reports the best patch size per model/task; the non-patched regime exposes long-range behavior. Random baselines: $0.549$ (\texttt{simple\_copy}), $0.284$ (\texttt{color\_cond}).}
\label{tab:1d_recall}
\resizebox{\linewidth}{!}{%
\begin{tabular}{l l l l l l}
\toprule
Task & Causality & Hyena & Attention (RoPE) & Mamba (causal) & Mamba (bidir) \\
\midrule
\multicolumn{6}{l}{\emph{No patchification ($L=4{,}096$ tokens)}} \\
\texttt{simple\_copy} & causal     & \textbf{\num{1.32e-5}} & \num{0.128} & \num{0.739} & --- \\
\texttt{simple\_copy} & non-causal & \textbf{\num{5.20e-5}} & \num{0.150} & --- & \num{0.0854} \\
\texttt{color\_cond}  & causal     & \textbf{\num{0.0629}}  & \num{0.270} & \num{0.205} & --- \\
\texttt{color\_cond}  & non-causal & \textbf{\num{0.0062}}  & \num{0.274} & --- & \num{0.00649} \\
\midrule
\multicolumn{6}{l}{\emph{Best patch size per model (non-causal)}} \\
\texttt{simple\_copy} & --- & \num{3.34e-5}~\tiny($p{=}32$) & \num{1.20e-4}~\tiny($p{=}128$) & --- & \num{4.66e-5}~\tiny($p{=}64$) \\
\texttt{color\_cond}  & --- & \textbf{\num{5.0e-3}}~\tiny($p{=}16$) & \num{0.256}~\tiny($p{=}128$) & --- & \num{2.38e-4}~\tiny($p{=}8$) \\
\bottomrule
\end{tabular}
}
\end{table}

\paragraph{2D results.}
Extending the tasks to a genuine 2D canvas with 2D Hyena and axial-RoPE Attention yields the results in Table~\ref{tab:2d_recall}. The qualitative picture mirrors 1D but is sharper: on the full $4{,}096$-token grid, Hyena is $\sim\!400\times$ better than attention on \texttt{simple\_copy} and $\sim\!100\times$ better on \texttt{color\_cond}, while attention again saturates at the \texttt{color\_cond} random floor regardless of patch size. Bidirectional Mamba with $p{=}4$ now establishes a new best on \texttt{color\_cond} ($\num{4.06e-4}$), beating Hyena's best patched result by $5\times$; on \texttt{simple\_copy}, Mamba at $p{=}8$ matches Hyena.
Adding a Gaussian mask slightly improves Hyena on \texttt{simple\_copy} no-patch ($\num{1.31e-4}$ vs $\num{2.08e-4}$) but offers no advantage on \texttt{color\_cond} in 2D; the full sweep is in Table~\ref{tab:2d_recall_full}. 

\begin{table}[!htb]
\centering
\small
\caption{\textbf{2D spatial recall.} MSE on a $64\times64$ canvas, $50$k iterations, $\sim\!1.9$M parameters. ``Best $p$'' for Hyena (Gauss) \texttt{color\_cond}: no-patch is already optimal.}
\label{tab:2d_recall}
\resizebox{\linewidth}{!}{%
\begin{tabular}{l l l l l l}
\toprule
Task & Variant & Hyena & Hyena (Gauss) & Attention (RoPE) & Mamba (bidir) \\
\midrule
\texttt{simple\_copy} & No patch ($L{=}4{,}096$)  & \num{2.08e-4} & \textbf{\num{1.31e-4}} & \num{0.0923} & \num{0.349} \\
\texttt{simple\_copy} & Best patched               & \textbf{\num{6.53e-5}}~\tiny($p{=}8$) & \num{1.01e-4}~\tiny($p{=}4$) & \num{1.96e-4}~\tiny($p{=}16$) & \num{5.70e-5}~\tiny($p{=}8$) \\
\texttt{color\_cond}  & No patch ($N{=}4{,}096$)  & \textbf{\num{2.54e-3}} & \num{3.18e-3} & \num{0.253}  & \num{8.34e-3} \\
\texttt{color\_cond}  & Best patched               & \textbf{\num{2.09e-3}}~\tiny($p{=}2$) & \num{3.18e-3}~\tiny(no patch) & \num{0.127}~\tiny($p{=}8$) & \num{4.06e-4}~\tiny($p{=}4$) \\
\bottomrule
\end{tabular}
}
\end{table}

\begin{table}[h]
\centering
\small
\caption{\textbf{Full 1D patch-size sweep} (non-causal, 50k iterations, ${\sim}1.9$M parameters). MSE lower is better; bold = best per task/column. Random baselines: $0.549$ (\texttt{simple\_copy}), $0.284$ (\texttt{color\_cond}).}
\label{tab:1d_recall_full}
\begin{tabular}{l r l l l}
\toprule
Task & $p$ ($L$ tokens) & Hyena & Attention (RoPE) & Mamba (bidir) \\
\midrule
\multicolumn{5}{l}{\emph{\texttt{simple\_copy}}} \\
 & 1 (4096) & \num{5.20e-5} & \num{0.150}  & \num{0.0854} \\
 & 2 (2048) & \num{1.12e-4} & \num{0.197}  & \num{0.0153} \\
 & 4 (1024) & \num{1.20e-4} & \num{0.041}  & \num{2.57e-4} \\
 & 8 (512)  & \num{9.56e-5} & \num{9.00e-3}& \num{1.40e-3} \\
 & 16 (256) & \num{4.47e-5} & \num{5.38e-3}& \num{1.27e-4} \\
 & 32 (128) & \textbf{\num{3.34e-5}} & \num{9.96e-4} & \num{5.27e-5} \\
 & 64 (64)  & \num{3.42e-5} & \num{1.51e-4} & \textbf{\num{4.66e-5}} \\
 & 128 (32) & \num{3.73e-5} & \textbf{\num{1.20e-4}} & \num{7.45e-5} \\
 & 256 (16) & \num{5.70e-5} & \num{1.76e-4} & \num{1.86e-3} \\
\midrule
\multicolumn{5}{l}{\emph{\texttt{color\_cond}}} \\
 & 1 (4096) & \num{6.20e-3} & \num{0.274}  & \num{6.49e-3} \\
 & 2 (2048) & \num{0.0445}  & \num{0.275}  & \num{0.0121} \\
 & 4 (1024) & \num{0.0278}  & \num{0.274}  & \num{7.88e-3} \\
 & 8 (512)  & \num{7.80e-3} & \num{0.269}  & \textbf{\num{2.38e-4}} \\
 & 16 (256) & \textbf{\num{5.00e-3}} & \num{0.266} & \num{8.51e-4} \\
 & 32 (128) & \num{7.90e-3} & \num{0.264}  & \num{2.48e-3} \\
 & 64 (64)  & \num{0.0235}  & \num{0.261}  & \num{7.57e-3} \\
 & 128 (32) & \num{0.142}   & \textbf{\num{0.256}} & \num{0.0348} \\
 & 256 (16) & \num{0.246}   & \num{0.256}  & \num{0.113} \\
\bottomrule
\end{tabular}
\end{table}

\begin{table}[h]
\centering
\small
\caption{\textbf{Full 2D patch-size sweep} ($64{\times}64$ canvas, 50k iterations, ${\sim}1.9$M parameters). MSE lower is better; bold = best per task/column. Token counts: no patch = 4096, $p{=}2$ = 1024, $p{=}4$ = 256, $p{=}8$ = 64, $p{=}16$ = 16.}
\label{tab:2d_recall_full}
\begin{tabular}{l r l l l l}
\toprule
Task & $p$ & Hyena & Hyena (Gauss) & Attention (RoPE) & Mamba (bidir) \\
\midrule
\multicolumn{6}{l}{\texttt{simple\_copy}} \\
 & 1    & \num{2.08e-4} & \num{1.31e-4} & \num{0.0923} & \num{0.349} \\
 & 2    & \num{3.27e-4} & \num{2.98e-4} & \num{0.0632} & \num{5.25e-3} \\
 & 4    & \num{1.33e-4} & \textbf{\num{1.01e-4}} & \num{4.30e-3} & \num{2.92e-4} \\
 & 8    & \textbf{\num{6.53e-5}} & \num{1.03e-4} & \num{2.56e-4} & \textbf{\num{5.70e-5}} \\
 & 16   & \num{8.52e-5} & \num{1.44e-4} & \textbf{\num{1.96e-4}} & \num{1.87e-3} \\
\midrule
\multicolumn{6}{l}{\texttt{color\_cond}} \\
 & 1    & \textbf{\num{2.54e-3}} & \textbf{\num{3.18e-3}} & \num{0.253}  & \num{8.34e-3} \\
 & 2    & \num{2.09e-3} & \num{4.07e-3} & \num{0.271}  & \num{8.61e-4} \\
 & 4    & \num{5.61e-3} & \num{7.81e-3} & \num{0.216}  & \textbf{\num{4.06e-4}} \\
 & 8    & \num{1.69e-2} & \num{2.27e-2} & \textbf{\num{0.127}}  & \num{1.06e-2} \\
 & 16   & \num{6.38e-2} & \num{6.16e-2} & \num{0.131}  & \num{5.50e-2} \\
\bottomrule
\end{tabular}
\end{table}

\paragraph{3D motion results.}
The 3D suite runs at native $32^3 = 32{,}768$ voxels, where dense self-attention is computationally inadmissible: \texttt{torch.compile} attempts a $>\!500$\,GiB allocation during inductor lowering, so no practical non-patched 3D Attention baseline exists (and patched attention sits at the random floor on every \texttt{color} task in 1D--2D). HyenaND and bidirectional Mamba are thus the only operators evaluated at native 3D resolution; Table~\ref{tab:spatial_recall_summary} reports the outcome. On \texttt{simple\_copy}, HyenaND solves the task (\num{1.2e-4}) while Mamba barely improves on the mean-predictor floor throughout training (\num{0.451}; floor \num{0.482}): rasterizing a moving 3D target into a 1D scan breaks positional recall. On \texttt{color\_cond}, both operators solve the task, and the ordering is governed by the optimization regime rather than by architecture, as the factorial below shows.

\paragraph{Discussion: what the two probes tell us.}

On raw canvases each token is a single pixel, and the three operator families respond to this regime in characteristically different ways. Attention --- a retrieval operation over per-token summaries --- collapses: with one pixel per query and key there is nothing to compare, and it recovers only when patchification packs many pixels into each token. This explains its monotone improvement with patch size on \texttt{simple\_copy}, its failure on \texttt{color\_cond} at every patch size, and, we argue, why ViTs require patchification in practice. Scan-based recurrence survives the low-information regime on selection but not on transport: serializing a (moving) $N$-D target into a 1-D scan destroys the positional structure that copying requires. The global convolution is the only operator \emph{at home} at native resolution: it transports across $32{,}768$ tokens with pixel precision, binds as well as recurrence once the optimization regime is matched, and needs no resolution-dependent preprocessing to do either.

The practical reading is that patchification is not a modeling choice but a workaround for attention's information-to-token requirement, and an operator that does not need the workaround changes what the front of a network can look like.  The downstream experiments already point this way: in the 3D motion tasks, HyenaND operates directly on raw $32{,}768$-voxel volumes with no patchification at all, and in $3\rm D$ medical segmentation (\S\ref{sec:exp-medical}) it drops into a SwinUNETR-style encoder as a global mixer at ${\sim}110$k-token stages (a scale where dense attention is infeasible and windowed attention is the standard workaround) matching accuracy at ${\sim}11\%$ lower peak training memory. This suggests a role for HyenaND beyond a drop-in mixer: \emph{as a learned, native-resolution replacement for the tokenizer}. A few HyenaND layers operating directly on raw pixels or voxels could aggregate fine-grained structure into information-dense tokens for a backbone (including the attention hybrids of \S\ref{sec:exp-imagenet}) turning patch size from a hand-tuned hyperparameter into something the network learns. 

\FloatBarrier

\subsection{Long-Context Genomics ($1\rm D$)}
\label{app:genomics-experiments}

\subsubsection{Genomics Modeling: Experimental Details}
\label{app:genomics-setup}

We trained three replicates each of 1B-parameter LLMs of two varieties: full transformer LLMs, and Evo2-inspired mixed transformer/Hyena LLMs~\cite{brixi2025evo2}. For both architectures we applied the following modifications: embedding weights initialized to stdev$=1.0$ and not tied to the output projection~\cite{takase2025spikenomore}; activations included in all MLP layers (standard practice); bias excluded from projection layers (Evo2 only); bias added to short convolution layers (Evo2 only). The Evo2 variant is implemented as \texttt{striped\_hyena\_1b\_nv} in BioNeMo~\cite{stjohn2025bionemo}. For the full-transformer runs we used the same MHA layer configuration as those in the striped Hyena architecture; the full transformer is composed entirely of such MHA layers. The striped or full Hyena models had between 0 and 4 transformer layers mixed in. All models were trained for 72{,}926 steps (no early stopping) with Adam, global batch size 960, sequences padded to 8{,}192 bp, sampled from the OpenGenome2 metagenomics subset~\cite{brixi2025evo2}. Full-transformer runs at learning rate $3 \times 10^{-4}$ diverged consistently around step 10{,}000; we reduced their learning rate to $3 \times 10^{-5}$ for the final reported results.

\subsubsection{Genomics Modeling: Pre-divergence comparison}
\label{app:genomics}

Section~\ref{sec:genomicsmodeling} reports a $10\times$ LR reduction for the full-transformer runs after they diverged around step 10{,}000. Table~\ref{tab:parametric_summary_corrected_holm_early} reports the same statistical comparison at step 9{,}899, before any divergence, with the transformer evaluated at both the original ($3{\times}10^{-4}$) and reduced ($3{\times}10^{-5}$) learning rates. The hyena ranking matches Table~\ref{tab:parametric_summary_corrected_holm}: every striped configuration outperforms the transformer at the higher LR, and $H_2$ remains the top performer. This rules out the reduced LR as a confound for the headline result.

\begin{table}[h]
\centering
\caption{Parametric Statistical Comparison at Step 9{,}899 (before transformer loss explosion) using One-Sided Welch's T-Test (H$_1$: mean(model) $<$ mean(reference)). Holm--Bonferroni correction is applied per reference column (6 tests each; self-comparisons excluded). The transformer is shown at both the high LR ($3 \times 10^{-4}$, which later diverged) and the low LR ($3 \times 10^{-5}$, used for the final results). We use $T$ to denote the transformer with the lower LR to align with that model's  name in table~\ref{tab:parametric_summary_corrected_holm}.}
\label{tab:parametric_summary_corrected_holm_early}
\resizebox{\columnwidth}{!}{%
\begin{tabular}{l|r@{\hspace{2pt}}l|cc|cc}
\toprule
\textbf{Model Architecture} & \multicolumn{2}{c|}{\textbf{Mean $\pm$ SD}} & \textbf{p vs. $T_{\text{normal-lr}}$} & \textbf{p vs. $H_0$} & \textbf{Holm p vs. $T_{\text{normal-lr}}$} & \textbf{Holm p vs. $H_0$} \\
\midrule
$T_{\text{normal-lr}}$ (Full Transformer, LR=$3{\times}10^{-4}$) & \multicolumn{2}{c|}{\num{3.0717} $\pm$ \num{0.0049}} & \num{5.000e-01} & \num{9.833e-01} & \textemdash & \num{1.000e+00} \\
$T$ (Full Transformer, LR=$3{\times}10^{-5}$) & \multicolumn{2}{c|}{\num{3.2682} $\pm$ \num{0.0034}} & \num{1.000e+00} & \num{9.994e-01} & \num{1.000e+00} & \num{1.000e+00} \\
$H_4$ (Mixed, 4 MHA layers) & \multicolumn{2}{c|}{\num{3.0464} $\pm$ \num{0.0086}} & \num{9.530e-03} & \num{9.286e-01} & \num{3.812e-02} & \num{1.000e+00} \\
$H_3$ (Mixed, 3 MHA layers) & \multicolumn{2}{c|}{\num{3.0449} $\pm$ \num{0.0227}} & \num{8.617e-02} & \num{8.591e-01} & \num{1.723e-01} & \num{1.000e+00} \\
\textbf{$H_2$ (Mixed, 2 MHA layers)} & \multicolumn{2}{c|}{\textbf{\num{3.0031} $\pm$ \num{0.0014}}} & \textbf{\num{4.004e-04}} & \textbf{\num{8.366e-02}} & \textbf{\num{2.402e-03}} & \textbf{\num{5.020e-01}} \\
$H_1$ (Mixed, 1 MHA layer) & \multicolumn{2}{c|}{\num{3.0381} $\pm$ \num{0.0052}} & \num{6.053e-04} & \num{8.511e-01} & \num{3.027e-03} & \num{1.000e+00} \\
$H_0$ (Full Hyena, 0 MHA) & \multicolumn{2}{c|}{\num{3.0242} $\pm$ \num{0.0173}} & \num{1.670e-02} & \num{5.000e-01} & \num{5.009e-02} & \textemdash \\
\bottomrule
\end{tabular}%
}
\end{table}

\subsubsection{Context Length Scalability}
\label{app:scaling}

The configurations in Table~\ref{tab:scalability} are steady-state throughput
measurements from \emph{real, end-to-end training runs} of a 1B striped-hyena
model on GB200 GPUs---not fixed-step smoke tests. Each row is the median per-GPU
throughput over the hundreds-to-thousands of optimizer steps of one stage of a
staged context-extension schedule, run with the full data pipeline (FP8
current-scaling precision, \texttt{--use-subquadratic-ops}). Starting from an
8K-pretrained 1B model, we extend the context through a sequence of stages, each
doubling the sequence length and re-interpolating the rotary embeddings,
following the gradual context-extension policy of Evo2~\cite{brixi2025evo2} and
consistent with our own length-generalization bound (Theorem~\ref{thm:length-gen}:
the output shift grows with the length deviation $|\Delta L|$, so small per-stage
doublings keep the model within the regime where the bound is tight).

\paragraph{Throughput and capability.} The \emph{Training} column of
Table~\ref{tab:scalability} is the per-GPU throughput of the actual converged
runs, which use a deliberately memory-conservative configuration (activation
recompute enabled) so the full model state fits and training converges; it is
essentially flat across context length ($\sim$450~TFLOP/s/GPU from 128K to 16M,
higher at the single-GPU TP=1 stages that incur no tensor-parallel all-reduce).
This as-run figure is \emph{not} the operator's ceiling: disabling activation
recompute (a mock-data throughput benchmark on the same compute path; \emph{Peak}
column) raises throughput by \textbf{$\sim$1.3$\times$, to
$\sim$585--620~TFLOP/s/GPU} at every stage up to 8M, and above 800 at the short
TP=1 stages. \textbf{16M is the exception}: without recompute the activation
footprint no longer fits and the run OOMs, so 16M is genuinely memory-bound at
$\sim$449~TFLOP/s/GPU (a TP/CP rebalance does not change this). In short, the
$\sim$450~TFLOP/s/GPU of the training runs reflects a configuration chosen so the
model fits and converges, not a hardware or operator throughput ceiling; 8M is the
last length that fits recompute-free, and the two columns coincide at 16M because
that length is memory-bound.

\paragraph{Convergence at length.} Unlike the earlier fixed-step configurations,
these runs train to \emph{convergence}. Validation loss descends and plateaus at
every stage; under post-extension long-convergence training the validation NLL
improves from 1.048 to 1.014 at 1M and from 1.075 to 1.058 at 16M
(Figure~\ref{fig:genomics_convergence}). Beyond loss, the converged 16M model
demonstrably \emph{uses} its context: its per-token NLL continues to decrease
with position depth well beyond the short-range ramp of the first few thousand
tokens (Figure~\ref{fig:genomics_posbin}), indicating the model exploits distant
context rather than merely tolerating the sequence length.

\begin{table}[h]
\centering
\caption{Context-length scalability and throughput capability of the 1B
striped-hyena model on GB200 (FP8 current-scaling, \texttt{--use-subquadratic-ops}).
\emph{Training} = per-GPU throughput of the converged runs (activation recompute
\emph{on} --- the memory-conservative config used so the full state fits and training
converges). \emph{Peak} = throughput with recompute \emph{off} (a mock-data benchmark
on the same compute path), i.e.\ the operator's capability when not memory-constrained.
Recompute-off OOMs at 16M, so 16M is memory-bound ($^\dagger$: value is with recompute on).}
\label{tab:scalability}
\resizebox{0.82\columnwidth}{!}{%
\begin{tabular}{l|r|r|r|r|r|r}
\toprule
\textbf{Model} & \textbf{Seq Len} & \textbf{TP} & \textbf{CP} & \textbf{Nodes (GPUs)} & \textbf{Training} & \textbf{Peak} \\
 & & & & & \textbf{(recompute on)} & \textbf{(recompute off)} \\
\midrule
1B & 16K  & 1 & 1  & 1~(4)    & 584 & 796 \\
1B & 32K  & 1 & 1  & 1~(4)    & 606 & 823 \\
1B & 64K  & 1 & 1  & 1~(4)    & 586 & 778 \\
\midrule
1B & 128K & 2 & 2  & 1~(4)    & 458 & 605 \\
1B & 256K & 2 & 4  & 2~(8)    & 458 & 605 \\
1B & 512K & 2 & 8  & 4~(16)   & 459 & 604 \\
1B & 1M   & 2 & 8  & 4~(16)   & 461 & 616 \\
1B & 2M   & 2 & 16 & 8~(32)   & 457 & 614 \\
1B & 4M   & 2 & 32 & 16~(64)  & 454 & 608 \\
1B & 8M   & 2 & 64 & 32~(128) & 442 & 584 \\
1B & 16M  & 2 & 64 & 32~(128) & 443 & 449$^\dagger$ \\
\bottomrule
\end{tabular}%
}
\end{table}

\begin{figure}[h]
\centering
\includegraphics[width=0.62\linewidth]{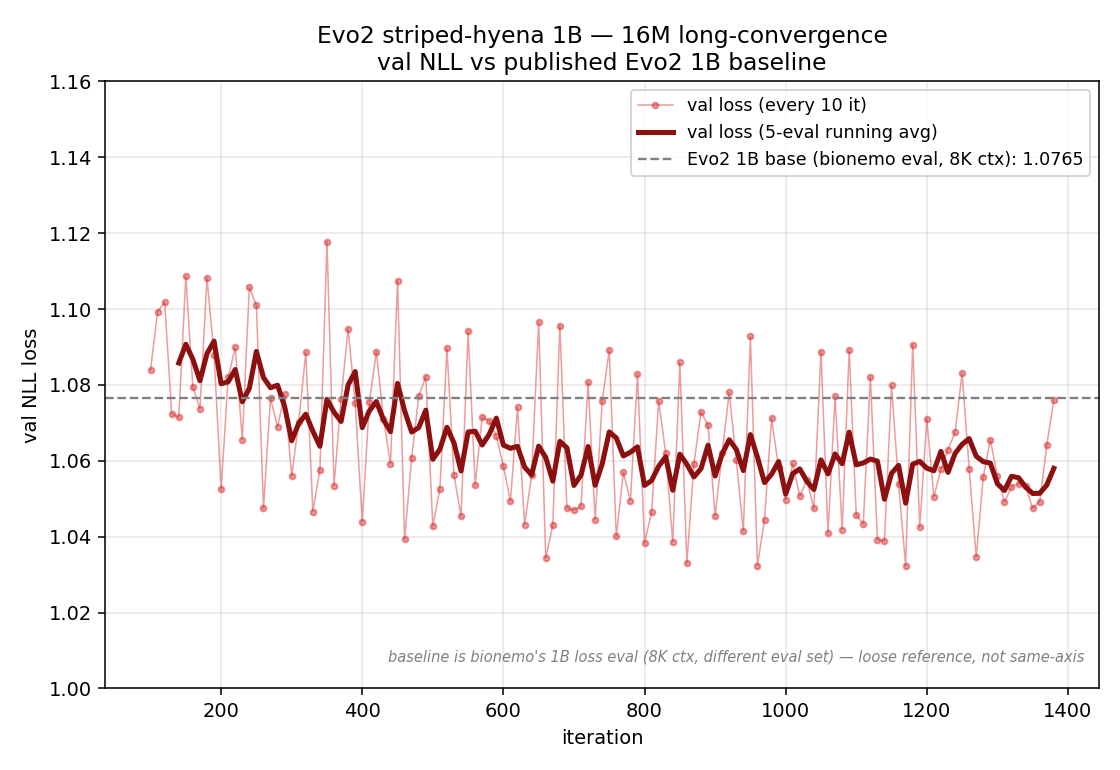}
\caption{16M long-convergence validation NLL (5-eval running average): training
converges, descending below the just-extended plateau and flattening.
}
\label{fig:genomics_convergence}
\end{figure}

\begin{figure}[h]
\centering
\includegraphics[width=0.72\linewidth]{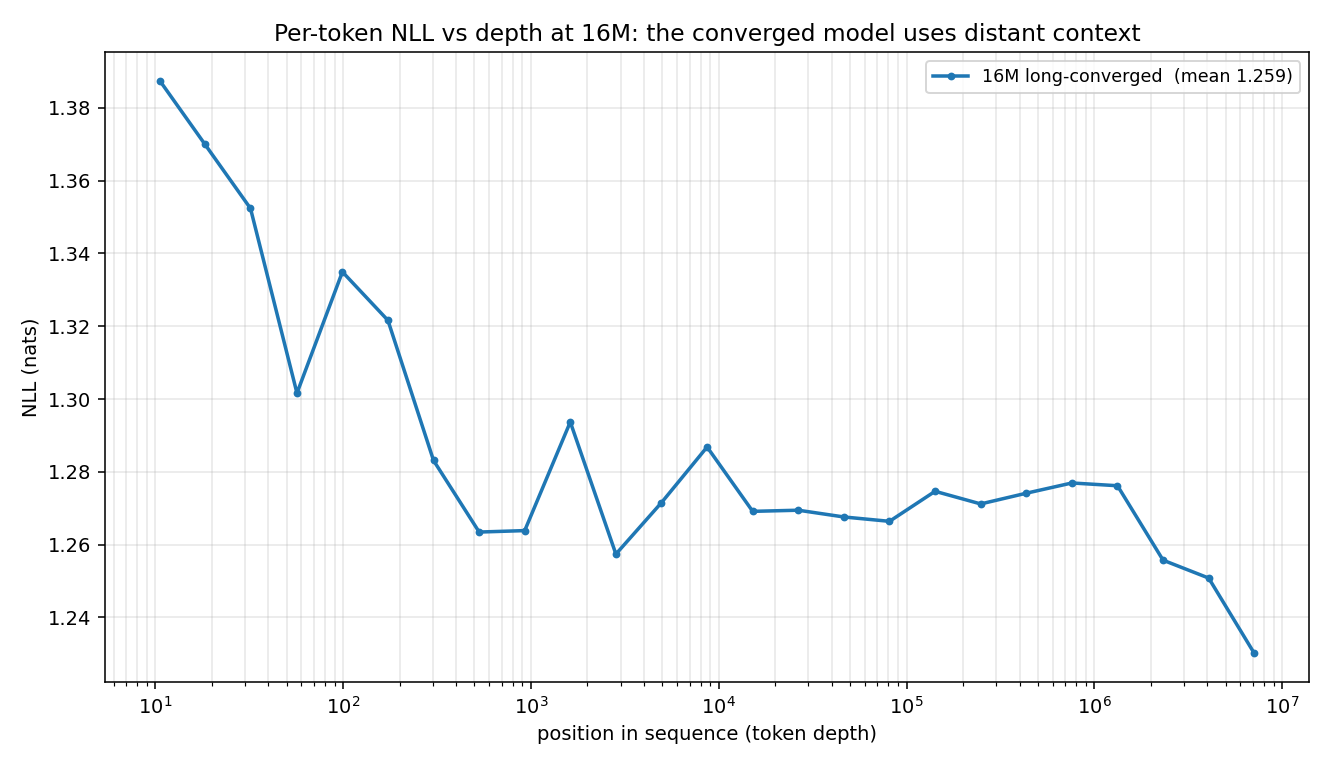}
\caption{Position-binned per-token NLL of the converged 16M model on held-out
stitched-genome test sequences. NLL continues to decrease with position depth
beyond the short-range ramp, indicating the model \emph{uses} distant context
rather than merely tolerating the sequence length. Deep bins ($>$4M tokens) rest
on few genomes; the trend is monotonic but statistically thin at the extreme.}
\label{fig:genomics_posbin}
\end{figure}

\subsection{ImageNet (Computer Vision, 2D)}
\label{app:imagenet}

\subsubsection{Training Setup}
\label{app:imagenet-setup}

\paragraph{Backbone.} All HyenaND models use the ViT-5-Small backbone~\cite{dosovitskiy2020image}: 12 blocks, hidden dimension 384, patch size $16{\times}16$, $224{\times}224$ input, 22M parameters.

\paragraph{Training recipe.} All models are trained for 800 epochs with LAMB (lr\,=\,4e-3, wd\,=\,0.05), cosine learning-rate schedule, 3-Augment~\cite{wightman2021resnet}, Mixup/CutMix, LayerScale, and EMA on 8$\times$H100 SXM (batch size 2048, bf16-mixed precision).

\paragraph{Model variants.} The attention baseline ($^\dagger$, same backbone and recipe as HyenaND) uses 6-head self-attention with RoPE and L2 QK-normalization. Pure HyenaND replaces every attention block with a 2D HyenaND layer: a SIREN-parameterized~\cite{sitzmann2020siren} long convolution evaluated via 2D FFT, operating directly on the $14{\times}14$ patch grid without rasterization. The hybrid variants interleave HyenaND and attention at $1{:}1$ (H\,A)$^{\times6}$ and $3{:}1$ (H\,H\,H\,A)$^{\times3}$ ratios within the same backbone.

\subsubsection{Patch-Size Ablation}
\label{app:imagenet_patch}
Patchification is the dominant image tokenization strategy for Vision Transformers: by grouping pixels into non-overlapping $p{\times}p$ patches before feeding them to the model, the token count is reduced by a factor of $p^2$, keeping the $\mathcal{O}(L^2)$ cost of self-attention tractable.
This compression, however, discards spatial detail.
Wang et al.~\cite{wang2025patchification} recently demonstrated a consistent patchification scaling law: classification accuracy improves monotonically as patch size decreases toward $1{\times}1$ (pixel tokenization), across architectures, tasks, and input resolutions.
For standard ViT, this law hits a hard wall at around $p{=}4$ (${\approx}3{,}000$ tokens): reducing $p$ further causes the $\mathcal{O}(L^2)$ attention cost to exceed GPU memory, making smaller patches practically infeasible, though FlashAttention~\cite{dao2022flashattention} kernels push this boundary lower by substantially reducing the memory footprint of attention, enabling $p{=}4$ and even $p{=}2$ for carefully tuned batch sizes.
Wang et al.\ circumvent this with a Mamba-based  $\mathcal{O}(L)$ architecture, reaching $p{=}1$ (50,176 tokens, $84.6\%$ top-1 with a base-sized model) at the cost of rasterizing the spatial grid into a 1D scan.

This experiment tests whether HyenaND can exploit the same patchification scaling law without rasterization. Because HyenaND operates via 2D FFT in $\mathcal{O}(L\log L)$, it should degrade far more gracefully than attention as $p$ decreases, while preserving the native 2D spatial structure of the image at every scale.
We ablate $p \in \{16, 8, 4, 2\}$ on the ViT-5-Small backbone (ImageNet-1K, $224{\times}224$) for the following mixer configurations: pure Attention, pure HyenaND, and the two hybrid ratios (H\,A)$^{\times6}$ and (H\,H\,H\,A)$^{\times3}$. At $p{=}16$ the token count is 196; at $p{=}2$ it reaches 12,544. All settings follow the training recipe in \S\ref{app:imagenet-setup}. Table~\ref{tab:imagenet_patch} additionally reports matched recurrence-based Mamba-attention hybrids at the same $1{:}1$ (M\,A)$^{\times6}$ and $3{:}1$ (M\,M\,M\,A)$^{\times3}$ ratios, isolating the token mixer. The HyenaND hybrids outperform their Mamba counterparts at every measured patch size, despite the bidirectional Mamba blocks costing both more parameters ($30$M/$34$M vs.\ $22$M) and more FLOPs (e.g.\ $12.5$/$14.0$ vs.\ $9.7$/$9.8$ GFLOPs at patch $16{\times}16$). 
Table~\ref{tab:imagenet_patch} reports top-1 accuracy and GFLOPs at each patch size, notably demonstrating that at the finest resolution ($p{=}2$), pure Attention requires around $6\times$ more FLOPs compared to pure HyenaND.

\begin{table*}[h]
\centering
\small
\setlength{\tabcolsep}{0.45em}
\caption{Patch-size ablation on ImageNet-1K. Top-1 accuracy (\%) and GFLOPs on a single H100 SXM. Mamba hybrids (MA, MMMA) share the ViT-5-Small backbone with attention layers replaced by bidirectional Mamba blocks; their GFLOPs are analytic estimates (standard selective-scan accounting) and patch $2$ was not run (--) for compute reasons. At equal hybrid ratios the Mamba variants use more parameters than the HyenaND hybrids ($30$M/$34$M vs.\ $22$M).}
\label{tab:imagenet_patch}
\resizebox{\textwidth}{!}{%
\begin{tabular}{@{}c r rr rr rr rr rr rr@{}}
\toprule
& & \multicolumn{2}{c}{\textbf{Attention}} & \multicolumn{2}{c}{\textbf{HyenaND} (pure)} & \multicolumn{2}{c}{\textbf{Hybrid} (HA)$^{\times6}$} & \multicolumn{2}{c}{\textbf{Hybrid} (HHHA)$^{\times3}$} & \multicolumn{2}{c}{\textbf{Mamba} (MA)$^{\times6}$} & \multicolumn{2}{c}{\textbf{Mamba} (MMMA)$^{\times3}$} \\
\cmidrule(lr){3-4}\cmidrule(lr){5-6}\cmidrule(lr){7-8}\cmidrule(lr){9-10}\cmidrule(lr){11-12}\cmidrule(lr){13-14}
\textbf{Patch} & \textbf{Tokens} & \textbf{Top-1} & \textbf{GFLOPs} & \textbf{Top-1} & \textbf{GFLOPs} & \textbf{Top-1} & \textbf{GFLOPs} & \textbf{Top-1} & \textbf{GFLOPs} & \textbf{Top-1} & \textbf{GFLOPs} & \textbf{Top-1} & \textbf{GFLOPs} \\
\midrule
16 & 196      & 81.8 & 9.41   & 81.5 & 9.97   & 82.1 & 9.69   & 82.0 & 9.83   & 81.8 & 12.46 & 81.0 & 13.99 \\
8  & 784      & 84.3 & 45.51  & 83.7 & 39.04  & 84.2 & 42.27  & 84.0 & 40.66  & 83.9 & 52.85 & 82.9 & 56.70 \\
4  & 3{,}136  & 85.1 & 317.34 & 84.0 & 155.32 & 85.0 & 236.33 & 84.4 & 195.83 & 84.4 & 278.14 & 83.9 & 259.40 \\
2  & 12{,}544 & 85.3 & 3{,}443.93    & 84.0 & 623.02     & 85.2 & 2{,}033.48    & 84.6 & 1{,}328.25     & -- & -- & -- & -- \\
\bottomrule
\end{tabular}
}
\end{table*}

\subsection{The Well: Results and Experimental Details}
\label{app:well-setup}
We compare three architectures on five selected datasets from \thewell{}. The CNextU-net baseline reproduces the strongest published model from \thewell~\cite{ohana2024well}: a $4$-stage U-Net encoder--decoder with ConvNext blocks ($42$ initial features, $2$ blocks per stage, ${\sim}18.6$M parameters) that operates at full pixel resolution. Against this, we evaluate two patch-based models that share a common ResidualNetwork backbone ($12$ blocks, hidden dimension $384$, RMSNorm, GLU-activated MLPs) and differ only in their sequence mixer: Attention, using $6$-head self-attention with QKV projections and axial RoPE (non-causal); and Hyena, a HyenaND operator whose implicit kernel is a $3$-layer SIREN MLP modulated by a learned Gaussian spatial envelope. Input fields are tokenized into non-overlapping $p^d$ patches with $p \in \{2, 4, 8\}$. All models are trained for 24 hours or up to $110{,}000$ iterations with AdamW (cosine schedule, $5\%$ warmup, gradient clipping at $1.0$, \texttt{bf16}-mixed precision) on a single H100 GPU. This budget exceeds the $12$-hour time-box of the published \thewell{} baselines; all three architectures are trained under the same extended protocol. The Attention and HyenaND models have parameter counts ($12$--$19$M) comparable to the ConvNext baseline.

Table~\ref{tab:well} reports validation VRMSE and training throughput for all models across the five \thewell{} datasets and three patch sizes. HyenaND attains the lowest VRMSE on every dataset; the gap over Attention widens as the patch size decreases and sequence length grows (see also Figure~\ref{fig:well_vrmse_scatter} in the main text).

\begin{table}[h]
\caption{\textbf{Patch-size ablation on \thewell.} Validation VRMSE
(lower is better) and training throughput (samples/sec) on a single
H100 GPU. CNextU-net operates at full pixel/voxel resolution (no patch
size). All runs target 24 hours or $110{,}000$ iterations. Patch-based
models use batch size $64$ for $p \in \{4, 8\}$ and batch size $16$
for $p{=}2$ due to memory constraints at long sequences; samples/sec
is reported to normalize across batch sizes.}
\label{tab:well}
\centering
\resizebox{\linewidth}{!}{%
\tiny
\setlength{\tabcolsep}{0.40em}
\begin{tabular}{@{}l c r rr rr rr@{}}
\toprule
& & & \multicolumn{2}{c}{\textbf{CNextU-net}} & \multicolumn{2}{c}{\textbf{Attention}} & \multicolumn{2}{c}{\textbf{Hyena}} \\
\cmidrule(lr){4-5} \cmidrule(lr){6-7} \cmidrule(lr){8-9}
\textbf{Dataset} & \textbf{Patch} & \textbf{Tokens}
& \textbf{VRMSE} & \textbf{samples/s}
& \textbf{VRMSE} & \textbf{samples/s}
& \textbf{VRMSE} & \textbf{samples/s} \\
\midrule
\texttt{gray\_scott\_reaction}   & 8   & 256       & ---    & ---   & 0.0520 & 89.6  & 0.0092          & 87.0  \\
\texttt{\_diffusion}   & 4   & 1{,}024   & ---    & ---   & 0.0538 & 92.8  & \textbf{0.0090} & 87.0  \\
$[128 \times 128]$  & 2   & 4{,}096   & ---    & ---   & 0.0974 & 6.9   & 0.0091          & 32.0  \\
     & --- & 16{,}384  & 0.2319 & 384.0 & ---    & ---   & ---             & ---   \\
\midrule
\texttt{active\_matter}        & 8   & 1{,}024   & ---    & ---   & 0.0586 & 91.5  & 0.0073          & 89.0  \\
$[256 \times 256]$     & 4   & 4{,}096   & ---    & ---   & 0.0616 & 85.1  & 0.0080          & 87.0  \\
     & 2   & 16{,}384  & ---    & ---   & 0.0914 & 6.9   & \textbf{0.0070} & 42.1  \\
                       & --- & 65{,}536  & 0.0347 & 89.0  & ---    & ---   & ---             & ---   \\
\midrule
\texttt{acoustic\_scattering}      & 8   & 1{,}024   & ---    & ---   & 0.0456 & 369.3 & 0.0086          & 271.4 \\
$[256 \times 256]$  & 4   & 4{,}096   & ---    & ---   & 0.0569 & 151.0 & 0.0068          & 105.6 \\
     & 2   & 16{,}384  & ---    & ---   & 0.1057 & 13.9  & \textbf{0.0062} & 25.3  \\
                       & --- & 65{,}536  & 0.0082 & 344.3 & ---    & ---   & ---             & ---   \\
\midrule
\texttt{shear\_flow}     & 8   & 2{,}048    & ---                  & ---   & 0.0354 & 130.4 & 0.0268 & 101.8 \\
                         & 4   & 8{,}192    & ---                  & ---   & 0.1049 & 32.2  & 0.0823 & 41.6  \\
$[256 \times 512]$       & 2   & 32{,}768   & ---                  & ---   & 0.4418 & 2.5   & 0.1185 & 11.2  \\
                         & --- & 131{,}072  & \textbf{0.0262}      & 124.2 & ---    & ---   & ---    & ---   \\
\midrule
\MHD                   & 8   & 512       & ---    & ---   & 0.3044 & 48.0  & 0.2810          & 47.4  \\
$[64 \times 64 \times 64]$ & 4 & 4{,}096 & ---    & ---   & 0.2164 & 23.7  & 0.1088          & 33.9  \\
                       & 2   & 32{,}768  & ---    & ---   & 0.3037 & 1.3   & \textbf{0.0543} & 13.0  \\
                       & --- & 262{,}144 & 0.2108 & 7.0   & ---    & ---   & ---             & ---   \\
\midrule
\texttt{supernova\_explosion}     & 8   & 512       & ---    & ---   & 0.6117 & 217.0 & 0.6151          & 271.4 \\
$[64 \times 64 \times 64]$   & 4   & 4{,}096   & ---    & ---   & 0.3879 & 217.0 & 0.3578          & 144.6 \\
 & 2 & 32{,}768 & ---   & ---   & 0.3000 & 2.6   & \textbf{0.1943} & 5.0   \\
                       & --- & 262{,}144 & 0.7400 & 21.1  & ---    & ---   & ---             & ---   \\
\midrule
\texttt{euler\_multi}    & 8   & 4{,}096   & ---    & ---   & 0.1293 & 58.0  & \textbf{0.0311} & 55.9  \\
\texttt{\_quadrants}     & 4   & 16{,}384  & ---    & ---   & OOM & --- & 0.0378 & 24.8  \\
$[512 \times 512]$       & 2   & 65{,}536  & ---    & ---   & OOM & --- & 0.1088 & 5.7   \\
                         & --- & 262{,}144 & 0.0332 & 63.8  & ---    & ---   & ---    & ---   \\
\midrule
\texttt{helmholtz\_}     & 8   & 4{,}096   & ---    & ---   & 0.0050 & 40.1  & \textbf{0.0042} & 30.2  \\
\texttt{staircase}       & 4   & 16{,}384  & ---    & ---   & 0.0147 & 7.6   & 0.0480 & 6.3  \\
$[1024 \times 256]$      & 2   & 65{,}536  & ---    & ---   & OOM & --- & 0.0443 & 1.9   \\
                         & --- & 262{,}144 & 0.0045    & 35.1   & ---    & ---   & ---    & ---   \\
\bottomrule
\end{tabular}
}
\end{table}

\subsection{$3\rm D$ Medical Imaging: Experimental Details}
\label{app:medical-setup}

SwinUNETR~\cite{hatamizadeh2022swin} encodes volumetric inputs through four hierarchical stages of windowed self-attention, with a patch-merging factor of $2$ applied after each stage. With ROI$=96^{3}$ inputs, the per-stage token counts span nearly three orders of magnitude: $48^{3}\!\approx\!1.1\!\times\!10^{5}$, $24^{3}\!\approx\!1.4\!\times\!10^{4}$, $12^{3}\!=\!\num{1728}$, $6^{3}\!=\!216$. We replace the per-stage mixer with our multi-dimensional Hyena operator lifted to $D{=}3$ according to a per-stage pattern $s \in \{A,H\}^{4}$, yielding four variants: $\mathbf{AAAA}$ (SwinUNETR baseline), $\mathbf{HHHH}$ (all-Hyena), $\mathbf{HAHA}$ (striped hybrid), and $\mathbf{HHAA}$ (hierarchical hybrid). All variants share the decoder, feature widths ($C_0{=}48$), and optimization recipe, and are trained from scratch for $100$ epochs on $4{\times}$NVIDIA GB200 GPUs. The training corpus is a pre-extracted patch collection derived from PanTS~\cite{li2025pants} ($\num{252376}$ patches of size $96^{3}$); evaluation uses $901$ held-out volumes against the $28$-class abdominal segmentation target with mean Dice as the primary metric.

\textbf{Stage placement aligns with the resolution hierarchy.}
The ordering $\mathbf{HHAA}\!>\!\mathbf{HAHA}\!>\!\mathbf{HHHH}\!>\!\mathbf{AAAA}$ has a simple explanation in the per-stage token counts. Hyena's advantage over windowed attention is greatest at the high-resolution early stages, where its global FFT receptive field captures anatomical extents that windowed attention can only approximate by stacking shifted windows; at the deep stages the contracted grid fits inside a single attention window, so Hyena's global advantage collapses while attention becomes essentially free. The hierarchical placement $\mathbf{HHAA}$ aligns this gradient with the encoder's resolution hierarchy. The striped placement $\mathbf{HAHA}$ uses the same Hyena budget but misallocates it: a Hyena block at stage~$2$ (where the contracted grid no longer rewards a global filter) and an attention block at stage~$1$ (where windowed attention is most constrained). This is a $3\rm D$ analogue of the striped-Hyena ordering Evo2-style language models exploit along depth (Section~\ref{sec:genomicsmodeling}): the best configuration is a hybrid in which Hyena is deployed where its global-context bias is most useful, namely the early, high-resolution end of the encoder.

\end{document}